%% file: main.tex
\documentclass[a4paper,fleqn,floatsintext]{cas-sc}

\usepackage[numbers]{natbib}

\usepackage{changes}
\usepackage{algorithm}
\usepackage{algorithmic}
\newproof{proof}{Proof}

\usepackage{subfigure}
\graphicspath{{./figures/}{./figures/FPE/}{./figures/Interaction/}{./figures/MEP/}}
\usepackage{physics}

\usepackage{hyperref}
\hypersetup{
     unicode=false,          
    pdftoolbar=true,        
    pdfmenubar=true,        
    pdffitwindow=false,     
    pdfstartview={FitH},    
    pdftitle={My title},    
    pdfauthor={Author},     
    pdfsubject={Subject},   
    pdfcreator={Creator},   
    pdfproducer={Producer}, 
    pdfkeywords={keyword1} {key2} {key3}, 
    pdfnewwindow=true,      
    colorlinks=true,       
    linkcolor=blue,          
    citecolor=blue,        
    filecolor=magenta,      
    urlcolor=cyan           
}

\newtheorem{remark}{Remark}
\newtheorem{assumption}{Assumption}
\newtheorem{theorem}{Theorem}
\newtheorem{lemma}{Lemma}

\renewcommand{\d}{\mathop{}\!\ensuremath{\mathrm{d}}}

\newcommand{\argmin}{ \operatornamewithlimits{argmin} }

\def\FW {Freidlin-Wentzell }
\def\DW {Dawson-G\"{a}rtner }

\usepackage{todonotes}

 \newcommand{\KL}{\operatorname{KL}}

 \newcommand{\pll}{\kern 0.56em/\kern -0.8em /\kern 0.56em} 

\def\tsc#1{\csdef{#1}{\textsc{\lowercase{#1}}\xspace}}
\tsc{WGM}
\tsc{QE}
\tsc{EP}
\tsc{PMS}
\tsc{BEC}
\tsc{DE}

\begin{document}
\let\WriteBookmarks\relax
\def\floatpagepagefraction{1}
\def\textpagefraction{.001}
\shorttitle{Generative Wasserstein Gradient Path Method}
\shortauthors{C. Liu and X. Zhou}

\title [mode = title]{Generative Path-Finding Method for Wasserstein Gradient Flow}                      

\tnotetext[1]{This work was funded by the Hong Kong General Research Funds  (11318522,11308323,11304525).}

\credit{Conceptualization of this study, Methodology, Software}

\affiliation[1]{organization={Department of Data Science, City University of Hong Kong},
                addressline={Kowloon}, 
                city={Hong Kong SAR},
                citysep={}}

\affiliation[2]{organization={Department of Mathematics, City University of Hong Kong},
                addressline={Kowloon}, 
                city={Hong Kong SAR},
              citysep={}
                }


\author[1]{Chengyu Liu}[%
  style=chinese,
  orcid=0009-0002-4513-1752]
\ead{cliu687-c@my.cityu.edu.hk}
\author[2]{Xiang Zhou}[%
  style=chinese,
  orcid=0000-0002-3835-3894]
\cormark[1]
\ead{xizhou@cityu.edu.hk}

\cortext[cor1]{Corresponding author}
\begin{abstract}
Wasserstein gradient flows (WGFs) describe the evolution of probability distributions in Wasserstein space as steepest-descent dynamics that minimize  a free-energy functional. 
To compute the entire path from an arbitrary initial distribution to an equilibrium distribution requires a long physical time, posing significant challenges for both forward‑ and implicit‑Euler (JKO) time‑marching schemes.  Grid-based Eulerian approaches for spatial variables are severely limited by the curse of dimensionality, whether solving dynamic or stationary equations for the density function. On the other hand, existing Lagrangian approaches, which leverage particle movements or generative maps, struggle to adaptively improve efficiency through time‑step tuning.
To address these limitations, we propose a  \emph{generative path-finding}  framework for    the   Wasserstein gradient path (GenWGP). This approach constructs a generative flow model that geometrically transports mass from an initial density to the unknown equilibrium distribution, guided by a path loss function that encodes the full trajectory, including the terminal endpoint for the equilibrium distribution. The path loss is derived from the geometric action functional based on  Dawson-G\"{a}rtner large-deviation theory, which characterizes the most probable  evolution   of empirical distributions in interacting diffusion systems. 
We first construct the path loss over an arbitrary finite time horizon using physical time parametrization and then derive the reparameterization-invariant geometric action functional, based on the Wasserstein arc-length parametrization. GenWGP employs normalizing flows to compute the geometric  curve converging to the equilibrium distribution.
A key feature of GenWGP is its enforcement of intrinsic constant-speed movement between adjacent layers of the generative neural network, encouraging approximately that discretized distributions remain equidistant with respect to  the Wasserstein metric along the entire path, even under complex free energy landscapes.
Consequently, our approach circumvents the need for intricate time-stepping schemes that are typically constrained by step-size limitations. Instead, it facilitates stable training that is robust and independent of the specific temporal or geometric discretization of the underlying continuous descent path. Furthermore, the learned generative map serves as a reusable sampler, enabling efficient computation of statistical quantities along the gradient flow.
We evaluate GenWGP on a variety of benchmark problems, including Fokker–Planck equations with both convex and non‑convex potentials, as well as interacting particle systems governing aggregation, aggregation‑drift, and aggregation‑diffusion dynamics. Numerical results demonstrate that GenWGP matches or exceeds the accuracy of high‑fidelity reference solutions, using as few as a dozen of discretization points for the entire gradient flow path, while effectively capturing complex dynamical behaviors.

\end{abstract}





\maketitle
\input{1-Introduction}

\input{2-Background}
\input{3-Physicaltime}

\input{4-Geometric}

\input{5-Numerical}

\input{6-Conclusion}

\bibliographystyle{apalike}
\bibliography{GenerativeWGF,gad}

\appendix
\section{Proofs in Section \ref{sec:GenWGP_finitetime}}

\subsection{Proof of Theorem \ref{theo:KL_bound}}\label{pf:KL_bound}
\renewcommand{\theequation}{A.1.\arabic{equation}}
\setcounter{equation}{0} 
\input{proof_KL_bound}

\subsection{Proof of Theorem \ref{theo:disc_error_bound}}\label{pf:disc_error_bound}
\renewcommand{\theequation}{A.2.\arabic{equation}}
\setcounter{equation}{0} 
\input{proof_disc_error_bound}

\subsection{Proof of Theorem \ref{thm:gMAM_equivalence}} \label{pf:gMAM_equivalence}
\renewcommand{\theequation}{A.3.\arabic{equation}}
\setcounter{equation}{0} 
\input{proof_gMAM}

\subsection{Proof of Theorem \ref{theo:EL_compare}}\label{pf:EL_compare}
\renewcommand{\theequation}{A.4.\arabic{equation}}
\setcounter{equation}{0} 
\input{proof_EL_compare}

\subsection{Proof of Theorem \ref{theo:disc_conv}}\label{pf:disc_conv}
\renewcommand{\theequation}{A.5.\arabic{equation}}
\setcounter{equation}{0} 
\input{proof_disc_conv}

\end{document}

%% file: 1-Introduction.tex
\section{Introduction}
Free energy on probability measure spaces provides a unified framework for analyzing both equilibrium states and non-equilibrium dynamics of complex systems in physics and applied mathematics \cite{lafferty1988density, rousset2010free}.
It is a fundamental concept to modeling phenomena ranging from diffusion and chemotaxis to pattern formation. 
In this work, we consider 
the following free energy $\mathcal{F}: \mathcal{P}_2(\Omega) \to \mathbb{R} \cup \{+\infty\}$
\begin{equation} \label{eqn:energy}
\mathcal{F} (\rho) := \int_{\Omega} \left[\beta^{-1} U_m(\rho(x)) + V(x) \rho(x) \right] \d x + \frac{1}{2} \int_{\Omega \times \Omega} W(x-y) \rho(x) \rho(y) \d x \d y,
\end{equation}
where $\rho$ denotes a probability density in $\mathcal{P}_{2}(\Omega)$, the space of probability measures that are absolutely continuous with respect to the Lebesgue measure and possess finite second moments.
Throughout this paper, we identify the probability measure with its Radon-Nikodym derivative (density) $\rho$ with respect to the Lebesgue measure.

Equation \eqref{eqn:energy} comprises three   contributions. The first term, $\int_{\Omega} \beta^{-1} U_m(\rho(x)) \d x$, corresponds to  the internal energy, where $\beta = (k_B T)^{-1}$ denotes the inverse temperature, $k_B$ is the Boltzmann constant, and $T$ is the absolute temperature. The choice of the function $U_m$ corresponds  to   distinct diffusive behavior: for instance, $U_1(\rho)=\rho\log\rho$ yields the classical Boltzmann entropy \cite{boltzmann1872weitere} and recovers linear diffusion \cite{villani2021topics}. Nonlinear diffusion is captured by the power-law form $U_m(\rho) = \frac{\rho^m}{m-1}$ (for $m>0, m \neq 1$), encompassing the porous medium regime ($m>1$) and the fast diffusion regime ($0<m<1$). Key properties of the solutions—such as mass conservation, finite-time extinction, finite propagation, and self-similarity—depend critically on the interplay between the power 
  $m$ and the dimension $d$ \cite{otto2001geometry, carrillo2003kinetic, vazquez2007porous}.
The second term involves an external potential $V(x)$, representing spatial confinement or environmental drift.
The third term accounts for pairwise interactions, modeling attraction or repulsion between particles.
A notable example of \eqref{eqn:energy} is the relative entropy (Kullback–Leibler divergence) $D_{\mathrm{KL}}(\rho\| p_\ast)=\int \rho\log(\rho/p_\ast)\d x$, by setting $\beta=1$, $m=1$, $W\equiv 0$, and $V=-\log p_\ast$.

A central goal in the applications is to identify the minimizers of the free energy $\mathcal{F}$ in \eqref{eqn:energy}, which correspond to stable equilibrium distributions of the underlying particle system.
The relaxation dynamics from an initial state $\rho_0$ toward an equilibrium is well described by the \textbf{Wasserstein Gradient Flow (WGF)} - the steepest descent  of $\mathcal{F}$ in the Wasserstein space equipped with the 2-Wasserstein metric $\d_\mathcal{W}$:  
\begin{equation} \label{eqn:GF_abstract}
\begin{cases}
\partial_t p_t= -\nabla_{\d_\mathcal{W}}\mathcal{F}(p_t),
\\ p_t|_{t=0} = \rho_0.
\end{cases}
\end{equation}
Using the Benamou-Brenier differential structure of the Wasserstein space \cite{ambrosio2008gradient}, the abstract gradient flow \eqref{eqn:GF_abstract} takes the form of a continuity equation:
\begin{equation} \label{eqn:GF_concrete}
\partial_t p_t= \nabla \cdot \left( p_t \nabla \frac{\delta \mathcal{F}(p_t)}{\delta p_t} \right),
\end{equation}
where the first variation (Fr\'echet derivative) of \eqref{eqn:energy} is given by
$$ \frac{\delta \mathcal{F}(\rho)}{\delta \rho}(x) = \beta^{-1} U'_m(\rho(x)) + V(x) + \int_{\Omega} W(x-y)\rho(y)\d y. $$
Under suitable conditions, $p_t$ converges as $t \to \infty$ to a steady state (equilibrium distribution or invariant measure) $p_\infty$hat corresponds to a minimizer of $\mathcal{F}$ \cite{ambrosio2008gradient, peletier2014variational}.
The gradient flow \eqref{eqn:GF_concrete} encompasses a broad class of fundamental evolution equations: 
selecting  $U(\rho)=U_1(\rho)=\rho\log\rho$ and $W\equiv 0$ recovers the Fokker-Planck equation; incorporating  the nonlocal term $W\ast\rho$ leads to  the McKean-Vlasov equation; and the setting  $V=W=0$ and $U_m$ with $m>1$ yields the porous medium equation.

Solving \eqref{eqn:GF_concrete} numerically poses significant challenges. Existing methods generally fall into two categories.
\textbf{Eulerian schemes}, such as Finite Difference or Finite Volume methods, solve  the PDE \eqref{eqn:GF_concrete} by discretizing the spatial domain on a fixed grid or mesh  \cite{carrillo2015finite}. While accurate in low dimensions, they suffer from the \emph{curse of dimensionality} as grid complexity scales exponentially with $d$.
Recent deep learning-based solvers addresses  this issue by parameterizing $p_t$ with neural networks \cite{raissi2019physics, yu2018deep} and defining the loss function as the mean squared error of the residual for \eqref{eqn:GF_concrete}. While these methods offer flexibility and generalization for high-dimensional problems, they often fail to strictly preserve essential physical invariants, such as probability mass conservation and non-negativity, and rely on sophisticated adaptive training techniques.

In contrast, \textbf{Lagrangian approaches}  employ particle dynamics  to track individual particles or trajectories, 
avoiding the need for fixed spatial grids. Numerically, these methods approximate the underlying particle flow rather than directly solving for the probability density function.   It preserves essential  probability properties, adapts naturally to sparse or localized solutions, and avoids the curse of dimensionality when equipped with the modern generative methods.
Prominent examples include score-based methods, which typically employ forward Euler time discretization \cite{boffi2023probability, lu2024score, huang2024vy}, and the Jordan–Kinderlehrer–Otto (JKO) scheme, based on implicit Euler discretization \cite{jordan1998variational}, as well as its neural extensions \cite{lee2024deep, xu2023normalizing, hu2024energetic}.

However, regardless of the specific spatial discretization approaches, the above existing methods fundamentally operate as time-marching schemes. Thesse methods rely on sequential updates of the density or particle movements over short time steps, requiring the solution of a subproblem at each step.This creates a critical dependency on the time‑step size: smaller step sizes - while necessary for accuracy and stability - significantly increase computational costs, making long‑time simulations particularly prohibitive.
In principle, adaptive tuning of the step size, as done in traditional Euclidean spaces, is feasible. However, in practice, it is   constrained by the intrinsic complexity of the Wasserstein metric. Moreover,   the gradient flow  requires infinite time to reach the exact equilibrium. Truncating the flow to a finite time interval can result in bias, causing the final state to deviate from the true equilibrium. There is no definitive way to select the truncation time interval {\it a prior}, as the accuracy of the truncation varies on a case-by-case basis.


To address these limitations, we propose a  \textbf{Generative Wasserstein Gradient Path (GenWGP)} method, which re-frames the problem from sequential time-stepping to global path optimization, emphasizing the geometric path viewpoint of the underlying Wasserstein gradient flow. 
The \emph{Wasserstein Gradient Path} (WGP) refers to a curve representing the Wasserstein-2 gradient flow from an arbitrary initial density toward a nearby equilibrium density, parameterized either by the original physical time or by an alternative geometric parameter, such as the Wasserstein arc-length.
To achieve such a path optimization, we connect the gradient flow with the large deviation principle \cite{DG1987,DW1989, Feng2006, Zimmer2013}, and construct the loss functional in path space based on the action functional of the \DW\ large deviation principle for interacting particles \cite{DW1989}.  
The gradient flow corresponds to the \emph{zero-action} path \cite{FW2012, weinan-MAM2004, aMAM2008, Zimmer2013}. 
We shall  see later that,  our path loss function is equivalent to the mean-squared  residual of the time-dependent PDE \eqref{eqn:GF_concrete} measured in the $H^{-1}_\rho$ metric associated with the Wasserstein structure. 
Thus our  variational principle is to  learn the entire trajectory up to a prescribed terminal time $T$, rather than updating the solution one short time interval after another.
With a sufficiently large $T$, our method identifies the learned  $\rho_T$ as an approximation to the equilibrium distribution. 
If the extra penalty of $\mathcal{F}(\rho_T)$ is added  to the path-loss functional,   the optimization method can simultaneously determines the final terminal state an  equilibrium and resolve the whole gradient path quite well. 

Compared with the explicit time-marching schemes, 
our resulting time-discretized path loss avoids the stability restrictions and permits to take large time step size. But there still remains two accuracy issues: the choice of the truncation horizon $T$ and the non-uniform accuracy induced by a non-optimal temporal grid.  In long-time relaxation problem, we face the typical situation that  the early stage of the evolution may contain rapid transients, while the later stage approaches equilibrium only in a very slow pace. 
A physical-time parametrization is therefore not the most efficient one for describing the whole trajectory.
The classical action in the \DW\ large deviation principle is expressed as an integral over physical time. 
Motivated by the Maupertuis principle, as in the geometric minimum action method for the \FW\ action functional \cite{heymann2008geometric}, we reformulate the original \DW\ action functional into a geometric form that is \emph{invariant under time reparameterization} and \emph{free of the explicit dependence on the horizon truncation $T$}. 
In this formulation, the free energy at the terminal state is always added to the path-loss functional to guide the convergence to the equilibrium state.

The optimization of the curve allows us to represent the same gradient path by its geometry alone, independently of how fast or slow the path is traversed in physical time. 
In particular, it enables more suitable parametrizations for long-time relaxation, such as the Wasserstein-2 arc-length parametrization.  
Our GenWGP first computes the desired curve geometrically, and then recovers the corresponding physical-time parameter through a simple post-processing step based on the derived nonlinear relation between the geometric and physical parametrizations. 
In this way, the geometric formulation retains the essential dynamic information while avoiding the inefficiency of resolving the slow tail directly on a physical-time grid.

Numerically, to parameterize a Wasserstein path with $K$ discrete points, our GenWGP method employs a single \textbf{Normalizing Flow (NF)} neural network \cite{kobyzev2020normalizing,papamakarios2021normalizing} composed of $K$ stacked layers, each of which represents a diffeomorphic map transporting mass between two neighboring distributions along the path. 
This yields a mesh-free, generative solver that computes the whole Wasserstein gradient path globally from the initial distribution to the final equilibrium distribution.
Enforcing a meaningful arc-length parametrization is a central challenge in path-based methods. 
Existing approaches typically rely either on global path reparametrization, which is mainly designed for finite-dimensional state spaces \cite{String2002}, or on penalizing the Riemannian norm of the tangent vector to enforce constant speed \cite{StringNET2026}. 
In contrast, GenWGP exploits the structure of Normalizing Flows together with the Monge representation of the Wasserstein distance. 
For two neighboring path images $p_{k-1}$ and $p_k$, the Wasserstein distance can be characterized by minimizing the $L_2$ mean-square distance between admissible transport maps sharing the same reference measure $\rho_0$. 
This makes it natural to lift the geometric action from the density level to the map level and to train the NF directly through the discrete transport increments between neighboring layers. 
When the corresponding Monge minimizer is realized, the segment cost coincides exactly with the Wasserstein distance. 
Therefore, imposing a constant-speed constraint on these lifted segment lengths provides a practical and geometrically consistent mechanism for distributing the $K$ discretized path images approximately evenly along the Wasserstein curve.

\medskip
\noindent\textbf{Main contributions.}
The main contributions of this work can be summarized as follows.

\begin{enumerate}
    \item \textbf{A global path-optimization formulation.}
    Instead of computing the WGF \eqref{eqn:GF_concrete} by sequential time marching, we reformulate the problem as the optimization of an entire probability path. This shifts the numerical viewpoint from local time stepping to a global least-action principle in path space, and provides a variational framework for learning the full relaxation trajectory from an initial distribution toward equilibrium.

    \item \textbf{A generative Lagrangian parameterization.}
    We represent the discrete probability path by one normalizing flow whose layerwise composition plays the role of transport along the path. This yields a mesh-free Lagrangian solver that avoids Eulerian spatial grids and directly parameterizes the evolving distributions and particle trajectories in a unified manner.

    \item \textbf{Physical-time and geometric action formulations.}
    Starting from the \DW action functional \cite{DW1989}, we derive both a physical-time path loss on a prescribed horizon and a geometric, reparameterization-invariant formulation for long-time relaxation problems. The geometric formulation removes the explicit dependence on the physical-time horizon and is therefore better suited to computing paths that approach equilibrium.

    \item \textbf{A practical discrete optimization framework.}
    We construct trainable discrete objectives for both formulations using Crank-Nicolson-type temporal discretization, Monte Carlo particle approximation, and a lifted map-level geometric loss. In the geometric setting, we further introduce an arc-length regularization and a terminal free-energy penalty to stabilize training and to guide the endpoint toward a low-energy equilibrium state.

    \item \textbf{Supporting analytical results and numerical validation.}
    Under the stated regularity assumptions, we establish an \emph{a priori} KL-divergence estimate for the physical-time formulation, a trajectory-error decomposition for the discrete physical-time scheme, and a consistency result for the discrete geometric objective. Numerical experiments on Fokker-Planck and interacting-particle systems (aggregation, aggregation-drift, and aggregation-diffusion dynamics) demonstrate that the proposed framework can accurately approximate both relaxation paths and terminal states.

\end{enumerate}

{\bf Related work}
We review related works in the recent literature to highlight the distinct features of our approach.

\textbf{Equilibrium Solvers.}
Several recent works focus solely on identifying the equilibrium $\rho_\infty$ by training NFs via strong/weak/variational formulations of the \emph{static} PDE $\partial_t p_t = 0$ in \eqref{eqn:GF_concrete} \cite{tang2022adaptive, cai2024weak, xie2023deep, cai2025weak}. While these methods efficiently find the steady state using adaptive sampling, they do not capture the path of evolution.

{\bf Score-Based and Flow-Based Time-Marching Methods.}  
This new class of methods on discrete physical time grid learns the time-dependent velocity field $\mathcal{V}[p_t](x)$ or Stein score ($\nabla \log p_t$) 
\cite{boffi2023probability, shen2024entropy, li2023selfconsistent, huang2024vy}. 
These approaches typically use forward Euler integration to advance the solution. Consequently, they inherit the standard limitations of ODE solvers, including the need for small time steps and lack of global error control. Furthermore, all score-based methods are essentially restricted to linear diffusion (entropy-based) models.

{\bf JKO-based Variational Schemes.}
The JKO scheme has inspired methods like JKO-iFlow \cite{xu2023normalizing}, EVNN \cite{hu2024energetic}, and Deep JKO \cite{lee2024deep}, which use NFs to solve the variational subproblem at each time step. While these methods correctly leverage the Lagrangian framework, they solve a sequence of optimization problems at $t_1, t_2, \dots$, which is computationally expensive and prone to error drift. In contrast,our GenWGP optimizes the global action in a single training loop.

{\bf Parametric WGF.} 
The Parameterized WGF scheme \cite{liu2022neural, nurbekyan2023efficient} projects the gradient flow onto a neural network's parameter space. This transforms the PDE into an ODE on parameters but requires to compute and invert the Fisher Information Matrix. This inversion is computationally prohibitive for a large number of network parameters, a bottleneck that GenWGP avoids entirely.

{\bf Minimum Action Methods and \DW Large Deviation Principle.}
Our approach is based on the equivalence of the zero-action path 
and the gradient flow. The background of large deviation on 
  \FW and \DW theories can be found in \cite{FW2012,DW1989}
as well as \cite{Zimmer2013}. 
An nonexhaustive list for the  development of minimum action methods,
developed either in the finite dimensional configuration   space  $\mathbb{R}^d$
or in the function space $L^2(\mathbb{R}^d)$  include \cite{weinan-MAM2004,aMAM2008,WantMAM2015, Heymann2006, Heyman2008, CiCP2018-SUNZHOU}. However, none of these approaches have been extended to the Wasserstein manifold.
More recent path‑finding frameworks that combine variational principles with deep neural networks are presented in \cite{StringNET2026,SIMONNET2023112349}.

\medskip
The remainder of this paper is organized as follows. Section~\ref{sec:Background} introduces the mathematical background on Wasserstein gradient flows, the Dawson-G\"artner large deviation theory, and normalizing flows. Section~\ref{sec:GenWGP_finitetime} develops the physical-time GenWGP formulation for approximating Wasserstein gradient flows on a finite time horizon. Section~\ref{sec:GenWGP_geo} then extends this framework to a geometric, reparameterization-invariant formulation for paths converging to equilibrium. Section~\ref{sec:Numer} presents numerical experiments to demonstrate the accuracy and effectiveness of the proposed framework. Finally, Section~\ref{sec:Conclusion} concludes the paper.

%% file: 2-Background.tex
\section{ Background and Preliminaries} 
\label{sec:Background}

This section lays out the foundational concepts underpinning our method. We begin by reviewing the Riemannian geometry of the Wasserstein space and its associated gradient flow structure. We then present two complementary variational characterizations of the dynamics—one local and one global. Finally, we introduce normalizing flows as a flexible framework for parameterizing transport maps.

\subsection{Wasserstein Geometry and Gradient Flows}
Let $\mathcal{P}_{2}(\Omega)$ denote the space of absolutely continuous probability measures on $\Omega \subset \mathbb{R}^d$ with finite second moments. We identify each measure with its density function $\rho$.
The Wasserstein-2 distance between $\rho_0, \rho_1 \in \mathcal{P}_{2}(\Omega)$ is defined by the optimal transport problem \cite{villani2021topics}:
\begin{equation} \label{eqn:Monge}
\d_\mathcal{W}^2(\rho_0, \rho_1) = \inf_{T_\# \rho_0 = \rho_1} \int_\Omega |x - T(x)|^2 \rho_0(x)  \d x,
\end{equation}
where $T: \Omega \to \Omega$ is a transport map and $T_\# \rho_0$ denotes the pushforward measure.


\textit{Tangent Space and Metric.}
The space $\mathcal{P}_{2}(\Omega)$ possesses a formal Riemannian structure \cite{otto2001geometry}. 
Given $\rho \in \mathcal{P}_2(\Omega)$, the tangent space at $\rho$ can then be identified with density perturbations $f = \partial_t \rho_t|_{t=0}$, namely
\begin{equation}\label{eqn:tangent-space}
  \mathcal{T}_\rho \mathcal{P}_2(\Omega)
  =
  \biggl\{
    f \in C^\infty(\Omega)
    \;\bigg|\;
    \int_\Omega f(x)\mathrm{d}x = 0,\ 
    \exists\psi \in C^\infty(\Omega)\ \text{such that}\
    f = - \nabla\cdot\bigl(\rho \nabla \psi\bigr)
  \biggr\},
\end{equation}
where $\psi$ is uniquely determined up to an additive constant under suitable boundary conditions.
The Riemannian metric at $\rho$ is given by the weighted negative Sobolev norm $\|\cdot\|_{-1,\rho}$. For two tangent vectors $f_i = -\nabla \cdot (\rho \nabla \psi_i)$ ($i=1,2$), the corresponding inner product is defined by
$
\langle f_1, f_2 \rangle_{-1, \rho} = \int_\Omega \nabla \psi_1(x) \cdot \nabla \psi_2(x)\rho(x)\d x.
$
Equivalently, the induced norm admits the dynamic characterization
\begin{equation}\label{eqn:H-1}
\|f\|_{-1,\rho}^{2}
=
\inf_{\mathbf u: f=-\nabla\cdot(\rho \mathbf u)}
\int_{\Omega} |\mathbf u(x)|^{2}\rho(x)\d x,
\end{equation}
where the infimum is attained at $\mathbf u=\nabla\psi$ whenever $f=-\nabla\cdot(\rho\nabla\psi)$.

\textit{Wasserstein Gradient Flow.}
The Wasserstein gradient of a functional $\mathcal{F}(\rho)$ \eqref{eqn:energy} is given by $\nabla_{\d_\mathcal{W}} \mathcal{F}(\rho) = -\nabla \cdot \left( \rho \nabla \frac{\delta \mathcal{F}(\rho)}{\delta \rho} \right).$
The gradient flow equation $\partial_t p_t = -\nabla_{\d_\mathcal{W}} \mathcal{F}(p_t)$, takes the form of a continuity equation:
\begin{equation} \label{eqn:WGF}
\partial_t p_t = -\nabla \cdot \left( p_t \mathcal{V}[p_t] \right),
\end{equation}
where $\mathcal{V}[p_t]$ is the specific Lagrangian velocity field:
\begin{equation} \label{eqn:velocity}
    \mathcal{V}[p_t](x) = - \nabla \left( \frac{\delta \mathcal{F}(p_t)}{\delta \rho} (x) \right) = 
     - \nabla \left (
    \beta^{-1} U'_m(p_t(x)) + V(x) + (W\ast p_t)(x) \right ).
\end{equation}
Here $(W\ast p_t)(x)=\int_\Omega W(x-y)p_t(y)\mathrm{d}y$ and $U_m'(\cdot)$ denotes the derivative of the internal-energy density.

\subsection{Variational and Dynamic Formulations}
We interpret the WGF \eqref{eqn:WGF} as  two complementary variational characterizations. The local-in-time formulation motivates time-marching algorithms, while the global-in-time formulation motivates our path-finding approach.

\textit{ Local: Onsager's Principle (Least Dissipation).}
At any given  time $t$, the system selects an instantaneous velocity field $\mathbf{v}$  that minimizes the sum of dissipation potential and energy change rates \cite{onsager1931reciprocal1}:
\begin{equation}\label{eqn:variational-wgf}
\mathbf{v}_t = \underset{\mathbf{v}}{\operatorname{argmin}} \left\{ \frac{1}{2} \int_\Omega |\mathbf{v}(x)|^2 p_t(x)  \mathrm{d}x + \frac{d}{dt} \mathcal{F}(p_t) \right\}=\underset{\mathbf{v}}{\operatorname{argmin}} \left\{ \frac{1}{2} \int_\Omega |\mathbf{v}(x)|^2 p_t(x)  \mathrm{d}x +  \int \partial_t p_t \frac{\delta \mathcal{F}}{\delta \rho} (p_t)  \mathrm{d}x  \right\},
\end{equation}
subject to the continuity constraint $\partial_t p_t + \nabla \cdot (p_t \mathbf{v}) = 0$.  By integral by part, the unique minimizer of \eqref{eqn:variational-wgf} is  exactly  the   velocity field  \eqref{eqn:velocity}, i.e., $\mathbf{v}^*_t = -\nabla \left(\frac{\delta \mathcal{F}}{\delta \rho}  \right) =  \mathcal{V}[p_t]$.
At this optimal velocity,  $\frac{d}{dt} \mathcal{F}(p_t) = - \int_\Omega \| \mathcal{V}[p_t] \|^2 p_t \d x \leq 0$, ensuring energy dissipation.

\textit{Global: Least Action Principle.}
Integrating the gradient flow \eqref{eqn:GF_abstract} over the time interval $[0, T]$ formally yields the \textit{Dawson-G\"artner action functional} \cite{DG1987}.
For a path $p = (p_t)_{t\in[0,T]}$, the action is defined as:
\begin{equation} \label{eqn:rate-functional}
S_T[p]
:=
\begin{cases}
\displaystyle
\frac12\int_0^T
\Bigl\|
\partial_t p_t+\nabla_{\d_{\mathcal W}}\mathcal F(p_t)
\Bigr\|_{-1,p_t}^2\d t,
& \text{if }  p_t
 \text{is absolutely continuous and the integral converges},\\[1.2em]
+\infty,
& \text{otherwise.}
\end{cases}
\end{equation}
where the norm  $\left\| \cdot \right\|_{-1, p_t}^2$ is defined in  \eqref{eqn:H-1}.
For entropy-driven diffusion, namely the case $m=1$, this $S_T$ is known as the \textit{rate function} in the large deviation theory of interacting diffusion particle systems \cite{Dawson1983, DG1987, DW1989}, and it quantifies the likelihood of observing a prescribed trajectory of the empirical measure $(p_t^N)_{0\le t\le T}$ of $N$ diffusion particles following the McKean-Vlasov system. More precisely, consider the system of $N$ interacting It\^o processes
\begin{equation} \label{eqn:NP-SDE}
\d X_i(t) = -\nabla V(X_i(t))\d t - \frac{1}{N} \sum_{j=1}^N \nabla W(X_i(t) - X_j(t))\d t + \sqrt{2\beta^{-1}}\d B^i_t, \quad i = 1,\dots,N,
\end{equation}
where $\{B_t^i\}_{i=1}^N$ are independent standard Brownian motions. The associated empirical measure
$
p_t^N:= \frac{1}{N} \sum_{i=1}^N \delta_{X_i(t)}
$
converges, under standard assumptions on $V$ and $W$, to a deterministic measure $p(t)$ solving the McKean-Vlasov SDE
\begin{equation} \label{eqn:mf-sde}
\d X_t = -\nabla V(X_t)\d t - (\nabla W * p(t))(X_t)\d t + \sqrt{2\beta^{-1}}\d B_t, \qquad p(t):=\mathrm{Law}(X_t),
\end{equation}
whose time-marginal density satisfies the WGF \eqref{eqn:WGF}. The fluctuations of the empirical-measure path $\{p_t^N\}_{t\in[0,T]}$ around the macroscopic limit are, for entropy-driven diffusion (namely the case $m=1$), formally described by a path-space large deviation principle of Dawson-G\"artner type \cite{DG1987}. For a fixed $T$, this suggests that the path law of $p^N=(p_t^N)_{t\in[0,T]}$ admits a path-space large-deviation description with speed $N$ and rate functional $S_T$. More precisely, for a prescribed absolutely continuous path $q=(q_t)_{t\in[0,T]}$, one formally writes
\[
\Pr\bigl(p^N \approx q\bigr) \asymp \exp\bigl(-N S_T[q]\bigr).
\]
Here the large deviation principle is understood in the usual setwise sense; heuristically, for sufficiently small Wasserstein neighborhoods of $q$, one expects probabilities of order $\exp(-N S_T[q])$. In particular, for admissible paths $q$ for which the action is well-defined, $S_T[q]=0$ if and only if $q_t$ satisfies the WGF equation \eqref{eqn:WGF} almost everywhere in $t\in[0,T]$. Hence the deterministic gradient flow trajectory is a zero-action path, reflecting the fact that it arises as the law-of-large-numbers limit of the underlying interacting particle system. Motivated by this large-deviation structure, for more general free energies of Wasserstein gradient-flow type, we still use the same action functional as the natural global variational principle for path computation.
For a prescribed terminal distribution $\rho_1$, it is therefore natural to consider the associated \textit{minimum action problem} \cite{weinan-MAM2004, Heyman2008,aMAM2008}:
\begin{equation} \label{eqn:path-finding}
\begin{aligned}
 \inf_{p} \quad & S_T(p), \\
\text{subject to} \quad & p_0(x) = \rho_0(x), \;\; p_T(x) = \rho_1(x).
\end{aligned}
\end{equation}
This variational problem determines the least-action path connecting $\rho_0$ to $\rho_1$. To find the zero-action path, we simply drop the terminal state and minimize $S_T(p)$ only with the initial $p_0=\rho_0$, since the solution of the WGF \eqref{eqn:WGF} indeed gives zero action under this initial constraint.

\subsection{Normalizing Flows} \label{ss:NF}
To numerically approximate the probability path, we employ Normalizing Flows (NFs)\cite{rezende2015variational,dinh2016density,chen2018neural,durkan2019neural,kobyzev2020normalizing,papamakarios2021normalizing}.
An NF is a deep generative model that represents a complex probability distribution as the pushforward of a simple reference distribution $\rho_{\text{ref}}$ (e.g., the standard Gaussian) through an invertible map $\Phi$.
This map is typically parameterized as a composition of $K$ invertible neural network blocks $\Psi_k$ (commonly referred to as layers):
\[
\Phi = \Psi_K \circ \cdots \circ \Psi_2 \circ \Psi_1.
\]
For a sample $z_0 \sim \rho_{\text{ref}}$, the target density at $z_K = \Phi(z_0)$ is computable via the change-of-variables formula:
\begin{equation} \label{eqn:change-of-variable}
\log p(z_K) = \log \rho_{\text{ref}}(z_0) - \sum_{k=1}^K \log \left| \det \frac{\partial \Psi_k}{\partial z_{k-1}} \right|.
\end{equation}
Various architectures facilitate efficient computation. For instance, RealNVP coupling flows~\cite{dinh2016density} use triangular Jacobians to ensure linear-cost determinant evaluation ($O(d)$). More expressive variants include spline-based flows~\cite{durkan2019neural} which implement flexible monotone transforms, and diffeomorphic non-uniform B-spline flows~\cite{hong2023neural} which offer $C^2$-smooth, bi-Lipschitz maps with controlled regularity.

In our framework, we interpret the layer-wise composition of the flow as a \textit{temporal discretization} of the Lagrangian trajectory.
By associating each layer $\Psi_k$ with the transport over a time step $\Delta t$, the intermediate activations $z_k = (\Psi_k \circ \cdots \circ \Psi_1)(z_0)$ represent the positions of particles at time $t_k$.
Consequently, the full network $\Phi$ parameterizes the entire discrete trajectory $(p_{t_k})_{k=0}^K$.
This design  differs from Neural ODEs \cite{chen2018neural} as it retains the exact tractability of the density via Eq.~\eqref{eqn:change-of-variable} at every layer,circumventing the need for ODE solvers, numerical integration errors, and post‑hoc density estimation. As a result, the action functional can be evaluated directly and efficiently, making the method particularly well‑suited for path‑based variational problems in the Wasserstein space.

%% file: 3-Physicaltime.tex
\section{The Generative Wasserstein Gradient Path (GenWGP) Method for Wasserstein Gradient Flow}
\label{sec:GenWGP_finitetime}
This section develops our  Lagrangian, generative framework for computing Wasserstein gradient path. 
We  first present  a \textit{physical-time parameterized} formulation, which learns the WGF dynamics on a fixed time horizon $[0,T]$.

\subsection{Physical Time Parameterized Lagrangian Representation of the Action Functional}
\label{ssec:physical_time}
Building on the variational principles established in Section~\ref{sec:Background}, we derive a computational framework to minimize the action functional $S_T[p]$
 \eqref{eqn:rate-functional} by employing  a \textit{Lagrangian particle approximation} of the density $p_t$. The continuum weighted $H^{-1}_{p_t}$ norm appearing in the rate functional \eqref{eqn:rate-functional}  is then replaced by a discrete $L^2$ norm over a finite ensemble of particle trajectories, thereby avoiding the need to solve an elliptic PDE at every time step.

We first take the continuous time perspective and characterize the trajectory of the probability density using a time-dependent velocity field $\mathbf{f}_t: \Omega \to \mathbb{R}^d$.
This field determines the motion of particles via the characteristic ODE:
\begin{equation} \label{eqn:flow-ODE}
    \partial_t \Phi(t, z) = \mathbf{f}_t(\Phi(t, z)), \quad \Phi(0, z) = z,
\end{equation}
where $z \sim \rho_0$ represents the initial Lagrangian coordinate.
The time-dependent density $p_t$ is defined as the pushforward $p_t = \Phi(t, \cdot)_\# \rho_0$.
By the transport theorem, $p_t$ satisfies the continuity equation driven by $\mathbf{f}_t$:
\begin{equation} \label{eqn:continuity_f}
    \partial_t p_t + \nabla \cdot (p_t \mathbf{f}_t) = 0.
\end{equation}

This path $(p_t)$ is used to match the target Wasserstein gradient flow \eqref{eqn:WGF} governed by the thermodynamic driving force $\mathcal{V}[p_t]$ in \eqref{eqn:velocity}, via the least action principle of minimizing the action functional \eqref{eqn:rate-functional}. To do this,  we 
substitute  $\partial_t p_t$ in  $S_T$ by  the continuity equation \eqref{eqn:continuity_f} and  see the ``residual'' term becomes the divergence of the velocity mismatch:
$$
    \partial_t p_t - (-\nabla_{\d_\mathcal{W}} \mathcal{F}(p_t)) = -\nabla \cdot \left( p_t (\mathbf{f}_t - \mathcal{V}[p_t]) \right)
    =-\nabla \cdot \left( p_t (\partial_t \Phi(t,z) - \mathcal{V}[p_t]) \right).
$$
Applying the definition of $H^{-1}_\rho$ norm \eqref{eqn:H-1}, we arrive at the following minimization problem for the loss function:
\begin{equation}
\begin{aligned} \label{eqn:J-continuous}
\inf_p S_T[p]
&=
\inf_p\inf_{\mathbf f:\partial_t p_t+\nabla\cdot(p_t\mathbf f_t)=0}
\frac{1}{2} \int_0^T \int_{\Omega}
\left\| \mathbf{f}_t(x) - \mathcal{V}[p_t](x) \right\|^2 p_t(x)\d x\d t \\
&=
\inf_p\inf_{\Phi: (\Phi_t)_\# \rho_0 = p_t}
\frac{1}{2} \int_0^T \mathbb{E}_{z \sim \rho_0}
\left[
\left\| \partial_t \Phi(t, z) - \mathcal{V}[p_t](\Phi(t, z)) \right\|^2
\right] \d t =: \inf_{\Phi} J[\Phi].
\end{aligned}
\end{equation}
This formulation seeks a flow map  $\Phi$ whose instantaneous kinematic velocity matches the  thermodynamic driving force $\mathcal{V}[p_t]$  given by \eqref{eqn:velocity}.
This formulation can be interpreted as a \emph{Physics-Informed Neural Network (PINN)} in the Wasserstein space, but unlike standard Eulerian version of the PINNs that minimizes PDE residuals  for $p_t$ on spatial grids, our method minimizes the ``residual''  of the flow map 
$\Phi$ as governed by  the neural  ODE \eqref{eqn:flow-ODE}. This Lagrangian perspective for the PINN  in the Wasserstein space naturally builds the connections of many generative models and the evolution PDE for the density.   


We rigorously justify our path loss function in \eqref{eqn:J-continuous} below by providing an {\it a priori} bound on the  error measured in Kullback-Leibler divergence.

\begin{theorem}[KL Divergence Bound] \label{theo:KL_bound}
Let $p_t$ be the density induced by the flow $\Phi$ via \eqref{eqn:flow-ODE}\eqref{eqn:continuity_f} and $\widehat{p}_t$ be the exact solution to the WGF \eqref{eqn:WGF}. Under the regularity assumptions \ref{ass:densitytheta}, there exist constants $\alpha,\gamma > 0$ such that
\begin{equation} \label{ineqn:KLmix2}
    \sup_{t\in[0,T]} D_{\mathrm{KL}}(p_t \| \widehat{p}_t) \leq  \exp(\gamma T) \alpha J[\Phi].
\end{equation}
\end{theorem}
\begin{proof}
See Appendix \ref{pf:KL_bound}.
\end{proof}

\subsection{Discretization via Normalizing Flows.}
In the numerical implementation, we parameterize the flow map $\Phi(t, \cdot)$ at discrete time points using a Normalizing Flow.
We partition the time horizon $[0, T]$ into $K$ intervals - for instance, with a uniform step size  $\Delta t = T/K$.
The particle positions at time $t_k = k \Delta t$ are modelled by the generative map 
$
    \Phi_k(z) = (\Psi_{\theta_k} \circ \dots \circ \Psi_{\theta_1})(z),
$
and the velocity field $\mathbf{f}_t$ is approximated by the finite difference scheme:
\begin{equation} \label{eqn:dis-v}
    \mathbf{f}_{t_k}(\Phi_k(z)) \approx \frac{\Phi_k(z) - \Phi_{k-1}(z)}{\Delta t}.
\end{equation}
To achieve second‑order temporal accuracy for the path loss \eqref{eqn:J-continuous}, we employ the Crank–Nicolson scheme, which yields the following discrete empirical loss:
\begin{equation} \label{eqn:particle-loss}
J^K_N[\Phi] = \frac{\Delta t}{N} \sum_{k=1}^K \sum_{i=1}^N
\left\| \frac{\Phi_k(z_i) - \Phi_{k-1}(z_i)}{\Delta t}
- \frac{\mathcal{V}_N[p_k](\Phi_k(z_i)) + \mathcal{V}_N[p_{k-1}](\Phi_{k-1}(z_i))}{2} \right\|^2,
\end{equation}
where $\{z_i\}_{i=1}^N$ are i.i.d. samples drawn from $\rho_0$.
Here, $\mathcal{V}_N[p_k]$ denotes the empirical approximation of the velocity field  \eqref{eqn:velocity} computed using the particle batch at step $k$:
\begin{equation} \label{eqn:est-v}
    \mathcal{V}_N[p_k](x) := - \nabla \left( \beta^{-1} U_m^{\prime}(p_k(x)) + V(x) + \frac{1}{N} \sum_{j=1}^N W(x - x^{(j)}_k) \right),\quad x^{(j)}_k := \Phi_k(z_j).
\end{equation}
The density  $p_k(x)$ is computed via the formula \eqref{eqn:change-of-variable}.
The training procedure is  summarized in Algorithm~\ref{alg:MAM}.

\begin{algorithm}[H]
\caption{GenWGP (Physical-Time)}
\label{alg:MAM}
\begin{algorithmic}[1]
\REQUIRE NF parameters $\{\theta_k\}_{k=1}^K$, initial distribution $\rho_0$, horizon $T$, steps $K$.
\FOR{each training iteration}
    \STATE Sample a batch of $N$ particles $\{z_i\}_{i=1}^N \sim \rho_0$.
    \STATE Compute particle trajectories $x_k^{(i)} := \Phi_k(z_i)$ for $k=0, \dots, K$.
    \STATE For each step $k$, compute densities $p_k(x_k^{(i)})$ via \eqref{eqn:change-of-variable} and empirical velocities $\mathcal{V}_N[p_k](x_k^{(i)})$ via \eqref{eqn:est-v}.
    \STATE Evaluate the discrete loss $J^K_N[\Phi]$ using \eqref{eqn:particle-loss}.
    \STATE Update parameters $\{\theta_k\}$ via gradient descent (e.g., Adam).
\ENDFOR
\RETURN Optimized NF parameters $\{\theta_k\}_{k=1}^K$.
\end{algorithmic}
\end{algorithm}

As $\Delta t\to 0$ and $N\to \infty$, the discrete loss \eqref{eqn:particle-loss} consistently approximates the continuous action \eqref{eqn:J-continuous}. 
The following result gives a trajectory-error estimate in terms of a consistency residual measured against the exact Crank-Nicolson driving force.
\begin{theorem}[Residual-based trajectory-error bound] \label{theo:disc_error_bound}

Let $X^*(t, z)$ be the exact characteristic trajectory  associated with the velocity field \eqref{eqn:velocity} starting from $z \sim \rho_0$, and let $X^N_k(z) := \Phi_k(z)$ denote the discrete numerical trajectory generated by Algorithm~\ref{alg:MAM}.
We define the \textit{consistency residual} $\varepsilon$ as the maximum mismatch between the kinematic velocity and the Crank-Nicolson driving force:
\begin{equation} \label{eqn:epsilon_def}
    \varepsilon := \sup_{k, z} \left\| \frac{\Phi_k(z) - \Phi_{k-1}(z)}{\Delta t} -  \frac{\mathcal{V}_N[\widehat{p}_{t_k}](\Phi_k(z))+\mathcal{V}_N[\widehat{p}_{t_{k-1}}](\Phi_{k-1}(z))}{2} \right\|.
\end{equation}
Under the assumptions in Appendix~\ref{ass:densitytheta}, the expected trajectory error at any step $t_k$ satisfies:
\begin{equation} \label{eqn:disc_error}
    \mathbb{E}_{z \sim \rho_0} \left\| X^*(t_k, z) - X^N_k(z) \right\|
    \leq \left[ \underbrace{\mathcal{O}\left(N^{-1/2}\right)}_{\text{Sampling Error}} + \underbrace{\mathcal{O}(\varepsilon)}_{\text{Consistency Residual}} + \underbrace{\mathcal{O}(\Delta t^2)}_{\text{Discretization Error}} \right] t_k.
\end{equation}
\end{theorem}
\begin{proof}
See Appendix \ref{pf:disc_error_bound}. 
In particular, Eq.~\eqref{eqn:disc_error} shows that the trajectory error is controlled by three contributions: the sampling error $O(N^{-1/2})$, the consistency residual $O(\varepsilon)$, and the Crank-Nicolson discretization error $O(\Delta t^2)$.
\end{proof}

\begin{remark}
Our path formulation allows flexible  time discretization with no essential implementation barriers. 
Here we employ the Crank-Nicolson scheme, which uses the average of the velocity $\mathcal{V}_N$ evaluated at $t_k$ and $t_{k+1}$. 
Under suitable regularity, this gives a second-order accurate discretization in time. 
By incorporating information from both ends of each time interval, the scheme typically provides a more faithful approximation of the continuous trajectory than the  first‑order explicit or implicit methods.

\end{remark}

\begin{remark}
\label{rem:ada}
While the above results extend straightforwardly to non‑uniform time grids, the adaptive time‑meshing strategy employed in the adaptive minimum action method \cite{aMAM2008,CiCP2018-SUNZHOU} is not an easy task in the infinite‑dimensional space $\mathcal{P}_2(\Omega)$. The reason lies in the architecture of the normalizing flow: each layer corresponds exactly to a specified time point. Consequently, adjusting to a new time mesh requires an expensive {\it refitting} of the entire network \cite{xu2023normalizing}, which is fundamentally different from the simple component-wise interpolation in finite‑dimensional settings.
 \end{remark}

 \begin{remark}
Since we are interested in the zero-action path for the WGP,
we can introduce a weight function $\omega_t $ for the path loss \eqref{eqn:J-continuous} or the discrete $\omega_k$ for the discrete loss \eqref{eqn:particle-loss} as in the  score-training approach \cite{song2021scorebased}. For example, to enhance the path accuracy near the initial state, we may use a larger weight $\omega_1$ than  $\omega_2,\cdots,\omega_K$.
This weighted training   effectively changes the $L^2$ norm in \eqref{eqn:particle-loss} and may be interpreted  as a new  large deviation rate function   for   
\eqref{eqn:mf-sde} associated with a time-dependent $
\beta_t$ for the noise amplitude. For simplicity, we use the constant weight in our algorithms and examples.
 \end{remark}

\subsection{Connections to Classical Results}

Our Lagrangian action‑minimization framework \eqref{eqn:J-continuous} provides a valuable len for reinterpreting existing numerical topics, such as numerical time‑marching schemes, geometric optimal transport, and the probabilistic theory of interacting particle systems.

\begin{enumerate}
 \item \textit{Time-discretization schemes:}
    The path loss functional $J(\Phi)$ in~\eqref{eqn:J-continuous}, if restricted in a single time interval from $t_k$  to $t_{k+1}$ sequentially,  is closely related to several well-known time-discretization schemes for WGFs:
    \begin{itemize}
        \item \textit{Forward Euler (explicit):}  choosing
    $
    \frac{\Phi_{k+1}(z) - \Phi_k(z)}{\Delta t} - \mathcal{V}_N[p_k](\Phi_k(z))
    $ in \eqref{eqn:particle-loss} 
    results in an explicit marching scheme.  This is  related to  score‑based transport modeling methods \cite{boffi2023probability, lu2024score,huang2024vy},   where  a score function,  instead of  the entire velocity field $\mathcal{V}$  independently trained by a neural network based on data.

    \item \textit{Backward Euler (implicit) and JKO:}
    Replacing the midpoint velocity in \eqref{eqn:particle-loss} by the backward endpoint velocity  produces a backward-Euler-type residual.
The  stationary condition is formally consistent with the Euler-Lagrange equation associated with the JKO minimization \cite{jordan1998variational}
\[
p_{k+1} = \argmin_{\rho \in \mathcal P_2(\mathbb R^d)}
\left\{
\mathcal F[\rho] + \frac{1}{2\Delta t} \d_\mathcal W^2(\rho,p_k)
\right\}.
\]
This establishes the connection between our action-based formulation and implicit variational time-stepping methods such as  Deep JKO \cite{lee2024deep},  JKO-iFlow \cite{xu2023normalizing} and EVNN \cite{hu2024energetic}.

    \end{itemize}

\item \textit{Relation to dynamic optimal transport:}
In the special case $\mathcal F\equiv 0$, the driving force vanishes and \eqref{eqn:J-continuous} reduces to a kinetic-energy minimization over transport maps. In this sense, the proposed formulation is consistent with the Benamou-Brenier dynamic characterization of the Wasserstein-2 distance \cite{Villani2009}.

 

\item \textit{Entropy-driven flows, Fokker-Planck equation, and  score-based diffusion models:}
For $U = U_1(\rho) = \rho \log \rho$, the WGF becomes the Fokker-Planck equation. Minimizing the action \eqref{eqn:J-continuous} is equivalent to:
\begin{equation*}
    \frac{1}{2} \int_0^T \mathbb{E}_{x \sim p_t} \left[ \left\| \mathbf{f}_t(x) - \left( -\nabla \log p_t(x) - \nabla V(x) \right) \right\|^2 \right]  \d t.
\end{equation*}
This is a continuous-time analogue of score matching \cite{hyvarinen2005estimation, song2021scorebased}.

\end{enumerate}

We emphasize that while most of these results are limited to gradient flows corresponding to the diffusion or entropy case of the free energy \eqref{eqn:GF_concrete} with 
$m=1$ only, our action‑minimization framework applies more broadly to gradient flows in the Wasserstein space. This formulation can even characterize transition paths between two distinct local minima \cite{Dawson1983}, although this application  is not further pursued in the present work.

%% file: 4-Geometric.tex
\section{The GenWGP Method for Wasserstein Gradient Flow Converging to Equilibrium Distribution}
\label{sec:GenWGP_geo}

The physical-time path formulation in Section~\ref{sec:GenWGP_finitetime} provides a Lagrangian path loss to learn the gradient-flow dynamics over a specified finite horizon $[0,T]$. 
However, when the purpose is to capture the full relaxation from an initial state $\rho_0$ to an {\it equilibrium} $\rho_\infty$ - namely, a stationary point of the free energy - we suffer from the finite time interval truncation. Without any prior knowledge about the truncation error between $\rho_T$ and the true $\rho_\infty$ due to a finite $T$,  a safe play  is to use a very large $T$. Even though this practically works for classical adaptive minimum action method \cite{aMAM2008}, 
 our  Remark \ref{rem:ada} pointed out the fundamental difficulty of adaptive  time-stepping strategy of repeated   global  redistribution of the temporal mesh and model refitting of normalizing flow. 
In fact, parameterizing the entire Wasserstein gradient flow by physical time  is suboptimal for describing its convergence to an equilibrium. 

The main purpose is to characterize the path by its geometry rather than by the speed at which it is traversed. By reparameterizing the trajectory with an intrinsic variable — such as arc‑length — one removes the explicit dependence on the physical time horizon and transforms the long‑horizon relaxation problem into a finite‑length path optimization problem on the Wasserstein manifold.
In this section, we develop the geometric reformulation in the Wasserstein space, adapting the core principle of the geometric Minimum Action Method \cite{heymann2008geometric} originally formulated in Euclidean space—to this infinite‑dimensional Wasserstein space.

\subsection{Reparameterization-Invariant   Path Formulation of Geometric Action}
\label{ssec:geometric_formulation}

The basic idea is in the spirit of Maupertuis's  principle, which allows  for possible variation in the final time $T$ while keeping the   beginning and end points fixed, in contrast to Hamilton mechanics's principle in Section \ref{sec:GenWGP_finitetime}  with fixed initial state and final time. 
We start with the establishment of a geometric action functional below. 
The optimal relation between the arc-length and the physical time as the result of  the variation of the time  interval $T$ is used  in   Section \ref{ssec:recover_time} to recover the physical time.

\begin{theorem}[Geometric reformulation of the \DW  action function] \label{thm:gMAM_equivalence}

Assume $\rho_a,\rho_b\in \mathcal P_2(\Omega)\cap \mathrm{Dom}(\mathcal F)$.
Let $AC_{\rho_a,\rho_b,T}$ denote the set of absolutely continuous paths (with respect to Wasserstein metric) connecting two distributions from $\rho_a$ to $\rho_b$ over $[0,T]$.  Under the regularity assumptions \ref{ass:gMAM} specified in Appendix \ref{pf:gMAM_equivalence}, the variational problem associated with the time-dependent \DW  action $S_T[p]$ in \eqref{eqn:rate-functional} admits the following geometric variational form: 
\begin{equation} \label{eqn:gMAM_equivalence}
    \inf_{T > 0} \inf_{p \in AC_{\rho_a,\rho_b,T}} S_T[p] = \inf_{p \in AC_{\rho_a,\rho_b,1}} \widehat{S}[p],
\end{equation}
where the \textit{Geometric Action} $\widehat{S}[p]$ is defined on the  interval $\tau \in [0,1]$ by 
\begin{equation}\label{eqn:geo-action}
\widehat S[p]
:=
\begin{cases}
\displaystyle
\int_0^1
\left(
\|\partial_\tau p_\tau\|_{-1,p_\tau}
\|\nabla_{\d_{\mathcal W}}\mathcal F(p_\tau)\|_{-1,p_\tau}
+
\Bigl\langle
\nabla_{\d_{\mathcal W}}\mathcal F(p_\tau),\partial_\tau p_\tau
\Bigr\rangle_{-1,p_\tau}
\right)\d\tau,
& \text{if } p_\tau \text{ is absolutely continuous} \\ & \text{and the integral converges},\\[1.1em]
+\infty,
& \text{otherwise.}
\end{cases}
\end{equation}
\end{theorem}
\begin{proof}
    See Appendix \ref{pf:gMAM_equivalence}.
\end{proof}
Here we slightly abuse the notation: $(p_\tau)_{0\le \tau \le 1} $ and  $(p_t)_{0\le t \le T}$ ($T$ could be infinity) denote    the same  curve $p$, parametized by arc‑length $\tau$ and by physical time $t$, respectively.
We highlight here that 
$\widehat{S}$ is invariant under any reparametrization. So $\tau$ here may  refer  to any curve parameter, not restricted to the arc-length parameter.

The integrand in \eqref{eqn:geo-action} is equivalent to   
$\|\partial_\tau p_\tau\|_{-1,p_\tau}
\|\nabla_{\d_{\mathcal W}}\mathcal F(p_\tau)\|_{-1,p_\tau} (1-\cos \alpha)
$ where $\alpha$ is the angle between the tangent and the negative Wasserstein gradient. When the action $\hat{S}[p]$ is zero, this angle is exactly zero everywhere,
indicating the path $(p_\tau)_{0\le \tau \le 1}$ is indeed the Wasserstein gradient flow.
The second term in \eqref{eqn:geo-action} is the endpoint contribution of the free energy along the path. 
When the terminal state is prescribed, namely $p_0=\rho_a$ and $p_1=\rho_b$, the chain rule in Wasserstein space gives
\begin{equation} \label{eqn:geo-boundary-term}
\int_0^1 \langle \nabla_{\d_\mathcal{W}} \mathcal{F}(p_\tau), \partial_\tau p_\tau \rangle_{-1, p_\tau}  \d\tau
=\int_0^1 \frac{\d}{\d \tau} \mathcal{F}(p_\tau)   \d \tau=
\mathcal{F}(\rho_b)-\mathcal{F}(\rho_a).
\end{equation}
Hence, for fixed $p_0$ and $p_1$, minimizing the geometric action $\widehat{S}[p]$ is equivalent to minimizing only its first term, namely the \textit{Eulerian Geometric Action}:
\begin{equation} \label{eqn:geo-eulerian}
\mathcal{J}_{\text{Euler}}[p] :=\int_0^1
    \left\| \nabla_{\d_\mathcal{W}} \mathcal{F}(p_\tau) \right\|_{-1, p_\tau}
    \left\| \partial_\tau p_\tau \right\|_{-1, p_\tau}  \d\tau,
\end{equation}
which is certainly invariant under any reparameterization of the curve $\{p_\tau\}$.
To implement this geometric principle by particle-based generative models, we utilize the same isometry discussed in Section~\ref{ssec:physical_time} to minimize the equivalent \textit{Lagrangian Geometric Action}:
\begin{equation} \label{eqn:geo-lagrangian}
    \mathcal{J}_{\text{Lagrangian}}[\Phi] := \int_0^1 \left( \mathbb{E}_{z \sim \rho_0} \left\| \mathcal{V}[p_\tau](\Phi(\tau,z)) \right\|^2 \right)^{1/2} \left( \mathbb{E}_{z \sim \rho_0} \left\| \partial_\tau \Phi(\tau,z) \right\|^2 \right)^{1/2}  \d\tau,
\end{equation}
which expresses the same geometric cost 
as $\mathcal{J}_{\text{Euler}}[p]$, but at the level  of particles' path and the velocity fields evaluated along them.

In the present work of searching the geometric path connecting $p_0=\rho_a$ to an {\it unknown} equilibrium $\rho_b$,  the terminal state  $p_{\tau=1}$ is  optimized jointly with the  path losses $\mathcal{J}$ by keeping the additional  terminal free-energy penalty $\mathcal{F}(p_{\tau=1})$.

\subsection{Discrete Geometric Optimization}
We approximate the continuous-integral  action functional by discretizing  $\tau \in [0,1]$ into $K$ equal intervals. Let $p = \{p_k\}_{k=0}^K$ be the sequence of densities (which are referred to as discrete ``images'' in the minimum action method\cite{weinan-MAM2004}), and $\Phi = \{\Phi_k\}_{k=0}^K$ be the corresponding sequence of transport maps, with $\Phi_0 = \text{Id}$ and $p_k = (\Phi_k)_\# \rho_0$.
A natural discretization of the Eulerian geometric action \eqref{eqn:geo-eulerian} approximates $ \left\| \partial_\tau p_\tau \right\|_{-1, p_\tau} \d\tau$ with the Wasserstein distance 
between two neighbors and leads to the following  sum
\begin{equation} \label{eqn:prob_euler}
    \min_{p} \mathcal{J}^K_{\text{Euler}}[p] := \sum_{k=1}^{K} \d_\mathcal{W}(p_k, p_{k-1}) \cdot \frac{\left\| \nabla_{\d_\mathcal{W}} \mathcal{F} (p_{k-1}) \right\|_{-1, p_{k-1}}+\left\| \nabla_{\d_\mathcal{W}} \mathcal{F} (p_k) \right\|_{-1, p_k} }{2},
    \quad \text{s.t.} \quad p_0 = \rho_0, \ p_K = \rho_1.
\end{equation}
Likewise, we obtain the following equivalent discrete Lagrangian problem for $\Phi$:
\begin{equation} \label{eqn:prob_lagrange}
    \min_{\Phi} \mathcal{J}^K_{\text{Lagrangian}}[\Phi] := \sum_{k=1}^K \left\| \Phi_k - \Phi_{k-1} \right\|_{L^2(\rho_0)} \cdot \frac{ \left\| \mathcal{V}[p_{k-1}] \circ \Phi_{k-1} \right\|_{L^2(\rho_0)} + \left\| \mathcal{V}[p_k] \circ \Phi_k \right\|_{L^2(\rho_0)} }{2},
    \quad \text{s.t.} \quad \Phi_0 = \text{Id}, \ (\Phi_K)_\# \rho_0 = \rho_1.
\end{equation}
where $p_k = (\Phi_k)_\# \rho_0$. The consistency between these two formulations is formally established below.

\begin{theorem}[Equivalence of Discrete Formulations] \label{theo:EL_compare}
Let $\rho_0$ and $\rho_1$ be two given probability distributions.
\begin{enumerate}
    \item[(a)] If $\Phi^*$ is a minimizer of the Lagrangian problem \eqref{eqn:prob_lagrange}, then the induced density sequence $p^*_k = (\Phi^*_k)_\# \rho_0$ is a minimizer of the Eulerian problem \eqref{eqn:prob_euler}.
    \item[(b)] Conversely, if $p^*$ is a minimizer of \eqref{eqn:prob_euler}, and $\Psi_k$ is the optimal transport map from $p^*_{k-1}$ to $p^*_k$, then the composite map $\Phi^*_k = \Psi_k \circ \cdots \circ \Psi_1$ is a minimizer of \eqref{eqn:prob_lagrange}.
\end{enumerate}
\end{theorem}
\begin{proof}
    See Appendix \ref{pf:EL_compare}.
\end{proof}

We parameterize $\Phi_k$ using Normalizing Flows and approximate the $L^2(\rho_0)$ expectations in Eq.~\eqref{eqn:prob_lagrange} via Monte Carlo sampling. 
We finally have the trainable \textit{Empirical Discrete Geometric Loss} for  a batch of data $\{z_i\}_{i=1}^N \sim \rho_0$ as follows:
\begin{equation} \label{eqn:discrete-geo-loss}
\hat{J}^K_N[\Phi] := \sum_{k=1}^K d_k(\Phi) \cdot \frac{v_{k-1}(\Phi) + v_k(\Phi)}{2},
\end{equation}
where $d_k(\Phi)$ and $v_k(\Phi)$ are the batch estimators:
\begin{equation} \label{eqn:dk_vk_def}
    d_k(\Phi) := \left( \frac{1}{N} \sum_{i=1}^N \left\| \Phi_k(z_i) - \Phi_{k-1}(z_i) \right\|^2 \right)^{1/2}, \quad
    v_k(\Phi) := \left( \frac{1}{N} \sum_{i=1}^N \left\| \mathcal{V}_N[p_k](\Phi_k(z_i)) \right\|^2 \right)^{1/2}.
\end{equation}

\begin{theorem}[Consistency of the discrete geomtric objective]\label{theo:disc_conv}
Let $\hat{J}[\Phi]$ be the continuous geometric action \eqref{eqn:geo-lagrangian}, and let $\hat{J}_N^K[\Phi]$ be the discrete empirical objective \eqref{eqn:discrete-geo-loss}. Under the regularity and moment assumptions stated in Appendix~\ref{pf:disc_conv}, one has
\[
\mathbb{E}\big|\hat{J}_N^K(\Phi)-\hat{J}(\Phi)\big|
=
\mathcal{O}(K^{-2})+\mathcal{O}(N^{-1/2}).
\]
In particular, the geometric discretization is second-order accurate in the number of 
path segments $K$ (equivalent to the number of layers in the neural networks), up to the Monte Carlo sampling error.
\end{theorem}

\begin{proof}
See Appendix \ref{pf:disc_conv}.
\end{proof}

\subsection{Arc-length Parametrization as Regularization in  Training Algorithm}

However, 
direct optimization of the geometric objective can lead to a numerical artifact: degenerate parametrization, where many of the 
$K$ discrete images cluster in a small portion of the path. This occurs because the action $\widehat{S}$
  is invariant under any parametrization, including numerically pathological ones.
To obtain a stable and informative discretization,  we impose     the Wasserstein arc-length parametrization, which is associated with  a \textit{constant-speed constraint}, $\|\partial_\tau p_\tau\|_{-1} \equiv \text{const}$.
In the discrete setting, this is enforced by a variance penalty on the segment lengths $\{d_k\}$ computed via \eqref{eqn:dk_vk_def}:
\begin{equation} \label{eqn:arc_penalty}
\mathcal{L}_{\text{arc}}[\Phi] = \frac{\text{Var}(d_1, d_2, \dots, d_K)}{\text{Mean}(d_1, d_2, \dots, d_K)},
\end{equation}
which is one of standard strategies in adaptive minimum action method \cite{aMAM2008,StringNET2026} to ensure the even distance  $\{d_k\}$ along the path .

Because the equilibrium 
$\rho_b$ is unknown, we employ a penalized  terminal cost for $p_K$, as if our objective were solely to locate this equilibrium rather than the entire Wasserstein gradient flow  path. Specifically, the terminal density 
$p_K$ is treated as an optimization variable, and the penalty 
$\mathcal{F}(p_K)$ is introduced 
\begin{equation}  \label{eqn:terminal_loss}
    \begin{aligned}
    \mathcal{F}(p_K)
    &= \mathbb{E}_{x \sim p_K} \left[ \beta^{-1} \frac{U(p_K(x))}{p_K(x)} + V(x) + \frac{1}{2} \mathbb{E}_{y \sim p_K} [ W(x - y) ] \right] \\
    &= \mathbb{E}_{z \sim \rho_0} \left[ \beta^{-1} \frac{U(p_K(\Phi_K(z)))}{p_K(\Phi_K(z))} + V(\Phi_K(z)) + \frac{1}{2} \mathbb{E}_{z' \sim \rho_0} [ W(\Phi_K(z) - \Phi_K(z')) ] \right]
\end{aligned}
\end{equation}
to drive  this endpoint toward a low-energy terminal state. 
The final  training objective   then consists of the following three contributions :
\begin{equation}  \label{eqn:total_geo_loss}
\mathcal{L}_{\text{total}} = \hat{J}^K_N[\Phi] + \alpha_{\text{term}} \mathcal{F}(p_K) + \alpha_{\text{arc}} \mathcal{L}_{\text{arc}}[\Phi],
\end{equation}
where $\alpha_{\text{term}}$ and $\alpha_{\text{arc}}$ are penalty parameters. The complete training procedure is summarized in Algorithm~\ref{alg:gMAM}.

\begin{algorithm}[H]
\caption{GenWGP: Geometric Path}
\label{alg:gMAM}
\begin{algorithmic}[1]
\REQUIRE NF parameters $\{\theta_k\}_{k=1}^K$, initial distribution $\rho_0$, steps $K$, weights $\alpha_{\text{term}}, \alpha_{\text{arc}}$.
\FOR{each training iteration}
    \STATE Sample a batch of $N$ particles $\{z_i\}_{i=1}^N \sim \rho_0$.
    \STATE Compute paths $x_k^{(i)} = \Phi_k(z_i)$ and densities $p_k(x_k^{(i)})$ via Eq.\eqref{eqn:change-of-variable}.
    \STATE Compute segment lengths $d_k$ and force magnitudes $v_k$ via Eq.~\eqref{eqn:dk_vk_def}.
    \STATE Evaluate Geometric Loss $\hat{J}^K_N[\Phi]$ via Eq.~\eqref{eqn:discrete-geo-loss}.
    \STATE Evaluate Regularizers: terminal energy $\mathcal{F}(p_K)$ via Eq.~\eqref{eqn:terminal_loss} and arc-length penalty $\mathcal{L}_{\text{arc}}[\Phi]$ via Eq.~\eqref{eqn:arc_penalty}.
    \STATE Update parameters $\{\theta_k\}$ via gradient descent on $\mathcal{L}_{\text{total}}$ \eqref{eqn:total_geo_loss}.
\ENDFOR
\RETURN Optimized NF parameters.
\end{algorithmic}
\end{algorithm}

\begin{remark}[Control of Lipschitz Regularity by  transport cost]
\label{rmk:spatial_smoothness}
\cite{huang2023bridging} indicates that  minimizing the transport cost effectively controls the Lipschitz constant of the learned flow, which benefits robustness and generalization.
 Invoking the triangle inequality, the divergence between   two particles $x, y$ at the final map $\Phi_K$ is bounded by the accumulated transport cost:
\begin{equation} \label{eqn:regularity_bound}
    \begin{aligned}
    \|\Phi_K(x) - \Phi_K(y)\| &= \left\| (x - y) + \sum_{k=1}^{K} (\Phi_{k}(x) - \Phi_{k-1}(x)) - \sum_{k=1}^{K} (\Phi_{k}(y) - \Phi_{k-1}(y)) \right\| \\
    &\leq \|x - y\| + \sum_{k=1}^{K} \|\Phi_{k}(x) - \Phi_{k-1}(x)\| + \sum_{k=1}^{K} \|\Phi_{k}(y) - \Phi_{k-1}(y)\|,
    \end{aligned}
\end{equation}
where the summation terms represent the discrete path lengths. 
Our geometric action \eqref{eqn:prob_lagrange} is a weighted total arc-length length, where the weights encode the contributions of movements against the gradient flow. 
So, the geometric action $\hat{J}^K_N$ 
could be understood as a  regularization, similar to 
 the right-hand side of \eqref{eqn:regularity_bound},  to    enhance the regularity of the trained 
diffeomorphic map.
\end{remark}

\subsection{Recovering Physical Time from the Geometric Path}
\label{ssec:recover_time}
The geometric formulation in Section~\ref{ssec:geometric_formulation} produces a WGP $p_\tau$ parameterized by an intrinsic geometric variable $\tau \in [0,1]$, implicitly performing the optimal  adaptive discretization of the physical time march while   
hiding  the temporal evolution.
But 
from any initial $p_0$ (which is not a stationary point of $\mathcal{F}$), we can indeed  recover the original physical-time dynamics $p_t$
from our geometric path $p_\tau$, at least up to the last second one, $p_{K-1}$.

\paragraph{Time-Rescaling Relation.}
The zero-action trajectory follows the Wasserstein gradient flow $\partial_t p_t = -\nabla_{\d_\mathcal{W}} \mathcal{F}(p_t)$ in physical time $t$. 
Let $t(\tau)$ be the strictly increasing map from arc-length parameter $\tau$ to physical time $t$,
then by the chain rule, we obtain the the geometric ``velocity''  as 
\begin{equation} \label{eqn:tangent_relation}
    \partial_\tau  {p}_\tau = \frac{\d t}{\d\tau} \partial_t p_{t} = - \frac{dt}{\d \tau} \nabla_{\d_\mathcal{W}} \mathcal{F}(p_\tau).
\end{equation}
Taking the Wasserstein tangent norm $\|\cdot\|_{-1, p_\tau}$ on both sides, we obtain the scalar differential equation governing the time mapping:
\begin{equation} \label{eqn:scalar_ode}
    \|\partial_\tau p_\tau\|_{-1, p_\tau} = \frac{\d t}{\d \tau} \|\nabla_{\d_\mathcal{W}} \mathcal{F}(p_\tau)\|_{-1, p_\tau}.
\end{equation}


A key feature of the geometric training (Algorithm \ref{alg:gMAM}) is the regularization of the arc-length speed, so the path  $p_\tau$ satisfies the condition of the arc-length parametrization $\|\partial_\tau p_\tau\|_{-1, p_\tau} \approx c$ for some constant (i.e., total length) $c > 0$. 
Substituting this into \eqref{eqn:scalar_ode} allows us to solve for the time scaling factor:
\begin{equation} \label{eqn:dtdtau_final}
    \frac{\d t}{\d \tau} = \frac{c}{\|\nabla_{\d_\mathcal{W}} \mathcal{F}(p_\tau)\|_{-1, p_\tau}} = \frac{c}{\|\mathcal{V}[p_\tau]\|_{p_\tau}}.
\end{equation}
Equation~\eqref{eqn:dtdtau_final} offers a clear physical interpretation: the physical time lapse $\d t$ required to traverse a fixed geometric distance $\d \tau$ is inversely proportional to the magnitude of the driving force. 
Crucially, near equilibrium or metastable states where $\|\mathcal{V}\| \ll 1$, the derivative $\d t/\d \tau$ naturally becomes large. This allows the method to capture the ``slow tail'' of the relaxation process accurately without the computational burden of infinitesimal time-stepping required by Eulerian solvers.

\paragraph{Determination of the Time Constant.}
The constant $c$ represents the total path length in the Wasserstein metric and fixes the global time scale. It is determined from the free-energy dissipation identity along the geometrically parameterized path.
By \eqref{eqn:tangent_relation}, the rate of free energy dissipation along the geometric path is:
\begin{equation}
    \frac{d\mathcal{F}}{\d \tau} = \langle \nabla_{\d_\mathcal{W}} \mathcal{F}[p_\tau], \partial_\tau p_\tau \rangle_{-1, p_\tau}
   = - \frac{\d t}{\d \tau} \|\nabla_{\d_\mathcal{W} }\mathcal{F}[p_\tau]\|_{-1,p_\tau}^2 
    = -c \|\mathcal{V}[p_\tau]\|_{p_\tau}.
\end{equation}
Integrating both sides over $\tau \in [0,1]$ yields the formula for $c$:
\begin{equation} \label{eqn:c_final}
    c = \frac{\mathcal{F}(p_{\tau=0}) - \mathcal{F}(p_{\tau=1})}{\int_0^1 \|\mathcal{V}[p_\tau]\|_{p_\tau}  \d\tau}.
\end{equation}
This ensures that the reconstructed time evolution exactly matches the total free energy difference specified by the boundary conditions.

\paragraph{Numerical Reconstruction.}
Given the discrete sequence of transport maps $\{\Phi_k\}_{k=0}^K$ produced by the Normalizing Flow, we estimate the velocity magnitudes $v_k \approx \|\mathcal{V}[p_k]\|$ using the batch estimator defined in Eq.~\eqref{eqn:dk_vk_def}.
We approximate the integrals using the trapezoidal rule. First, the constant $c$ is estimated by \eqref{eqn:c_final}
\begin{equation}
\label{eqn:discrete_c}
    c \approx \frac{\mathcal{F}(p_0) - \mathcal{F}(p_K)}{\sum_{k=1}^K \frac{v_{k-1} + v_k}{2} \Delta\tau},
\end{equation}
where $\Delta\tau = 1/K$. Subsequently, the physical time increments $\Delta t_k = t_k - t_{k-1}$ are recovered by the mid-point scheme:
\begin{equation} \label{eqn:discrete_dt}
    \Delta t_k \approx \int_{(k-1)\Delta\tau}^{k\Delta\tau} \frac{c}{\|\mathcal{V}[p_\tau]\|_{p_\tau}}  d\tau \approx \frac{c  \Delta\tau}{2} \left( \frac{1}{v_{k-1}} + \frac{1}{v_k} \right).
\end{equation}
This procedure is summarized in Algorithm~\ref{alg:recover_wasserstein_flow}.
This reconstruction of the physical time is quite accurate up the last second   distribution $p_{K-1}$  on the path; the terminal distribution is the equilibrium state, taking infinitely long time to reach in theory.

\begin{algorithm}[H]
\caption{Recover Physical Time from Geometric Path}
\label{alg:recover_wasserstein_flow}
\begin{algorithmic}[1]
\REQUIRE Trained NF parameters $\{\theta_k\}_{k=0}^{K}$, initial distribution $\rho_0$.
\STATE Compute boundary energies $\mathcal{F}_0 = \mathcal{F}((\Phi_0)_\#\rho_0)$ and $\mathcal{F}_K = \mathcal{F}((\Phi_K)_\#\rho_0)$.
\STATE Estimate velocity norms $v_k$ for $k=0,\dots,K$ using batch samples via Eq.~\eqref{eqn:dk_vk_def}.
\STATE Compute path length constant $c$ via discrete approximation of Eq.~\eqref{eqn:discrete_c}.
\STATE Initialize $t_0 \leftarrow 0$.
\FOR{$k=1$ to $K$}
    \STATE Compute time step $\Delta t_k \leftarrow \frac{c \Delta\tau}{2}(v_{k-1}^{-1} + v_k^{-1})$.
    \STATE Update physical time $t_k \leftarrow t_{k-1} + \Delta t_k$.
\ENDFOR
\RETURN Physical timestamps $\{t_k\}_{k=0}^K$.
\end{algorithmic}
\end{algorithm}

\begin{remark} \label{rm4}
Our numerical recover of the physical time using Algorithm \ref{alg:recover_wasserstein_flow} also also yields an approximate terminal time   $t_K$,  even though the theoretical time required to reach equilibrium is infinite. 
Nevertheless, this numerical $t_K$ is  practically meaningful:  it indicates that the  time interval $[0,t_K]$   is sufficiently long  for  the gradient flow to approach equilibrium and for the free energy to converge close to its minimal value. Consequently, the setup of the  terminal $T\approx t_K$ (or between  $t_{K-1}$ and $ t_K$) - together with the entire recovered time mesh $(t_k)_{0\le k\le K-1})$ - can  be directly used  in the physical-time path optimization Algorithm \ref{alg:MAM}  as an optimal adaptive time mesh.
This allows refinement of the path within an practically optimal interval without requiring any change to the network architecture.
\end{remark}

%% file: 5-Numerical.tex
\section{\label{sec:Numer} Numerical examples}

Our numerical examples focus on the validation and application of the geometric GenWGP approach (Algorithm~\ref{alg:gMAM}), together with its time-recovery postprocessing (Algorithm~\ref{alg:recover_wasserstein_flow}). Unlike time-marching methods that operate on a prescribed finite horizon $T$, our goal is to approximate the full Wasserstein gradient flow toward equilibrium and to assess the learned path not only at the terminal state but also along the  evolution in physical time.

Section \ref{sec:Numer} is organized to validate one central numerical claim: the geometric GenWGP formulation provides a more effective representation of long-time relaxation than uniform physical-time discretization, while retaining accurate recovered dynamics on the transient regime.
We begin with analytically tractable Fokker–Planck examples, where exact solutions allow direct verification of both recovered trajectories and terminal states. We then perform matched comparisons with the physical-time formulation under identical architectures and training setups, so that the effect of geometric parametrization and time recovery can be isolated cleanly. Finally, for non-convex and interacting-particle systems where full transient references are unavailable or only partially reliable, we use partial-reference comparisons and structure-preserving diagnostics to test whether the learned path remains dynamically meaningful.

Unless otherwise specified, training samples are drawn from the standard Gaussian base distribution $\mathcal{N}(x;0,I_d)$ with $N$ particles. Models are trained using Adam with exponential learning-rate decay. Our  Python implementation is available at \href{https://github.com/cliu687/GenWGP}{GitHub}.

\subsection{Diffusion Process: The Fokker-Planck Equation}

We begin with entropy-driven dynamics associated with the free energy
$
\mathcal{F}(\rho) = \int \rho(x)\log\rho(x)\d x + \int V(x)\rho(x)\d x,
$
which corresponds to the Fokker-Planck equation at $\beta=1$. In this subsection, the availability of exact reference solutions allows for rigorous   direct quantitative validation of both the recovered physical-time dynamics and the terminal equilibrium state.
\subsubsection{Quadratic Potentials}
We consider convex quadratic potentials
$
V(x)=\tfrac12(x-\mu)^\top\Sigma^{-1}(x-\mu).
$
The corresponding WGF is the Ornstein-Uhlenbeck dynamics
$
\mathrm{d}X_t=-\Sigma^{-1}(X_t-\mu)\mathrm{d}t+\sqrt{2}\mathrm{d}B_t,
$
whose   law is $\mathcal{N}(\mu(t),\Sigma(t))$ given by 
\begin{equation*}
\mu(t)=\mu-(\mu-\mu(0))e^{-\Sigma^{-1}t},\qquad
\Sigma(t)
= e^{-\Sigma^{-1}t}\Sigma(0)e^{-\Sigma^{-1}t}
  + \Sigma\bigl(I-e^{-2\Sigma^{-1}t}\bigr).
\end{equation*}
and   $\mu,\Sigma$  are the 
equilibrium mean and covariance, respectively.

These examples provide the cleanest setting for quantitative validation, since both the  trajectory and the equilibrium state are explicitly known. We test the method on three cases: (i) a 2D isotropic potential, (ii) a 2D anisotropic potential, and (iii) a 10D block-structured potential.

We use $N=5000$ particles. The transport map is parameterized by a RealNVP normalizing flow with $K=9$ affine coupling ``layers''. The  coupling sub-network for each layer is a four-layer MLP of width $128$, with LeakyReLU activations and a final $\tanh$ scale head. Training uses $1000$ epochs, an initial learning rate $8\times10^{-4}$, and exponential decay factor $\gamma=0.9999$.

\paragraph{2D isotropic diffusion.}

\begin{figure}[pos=!htbp]
    \centering
\includegraphics[width=0.4\linewidth]{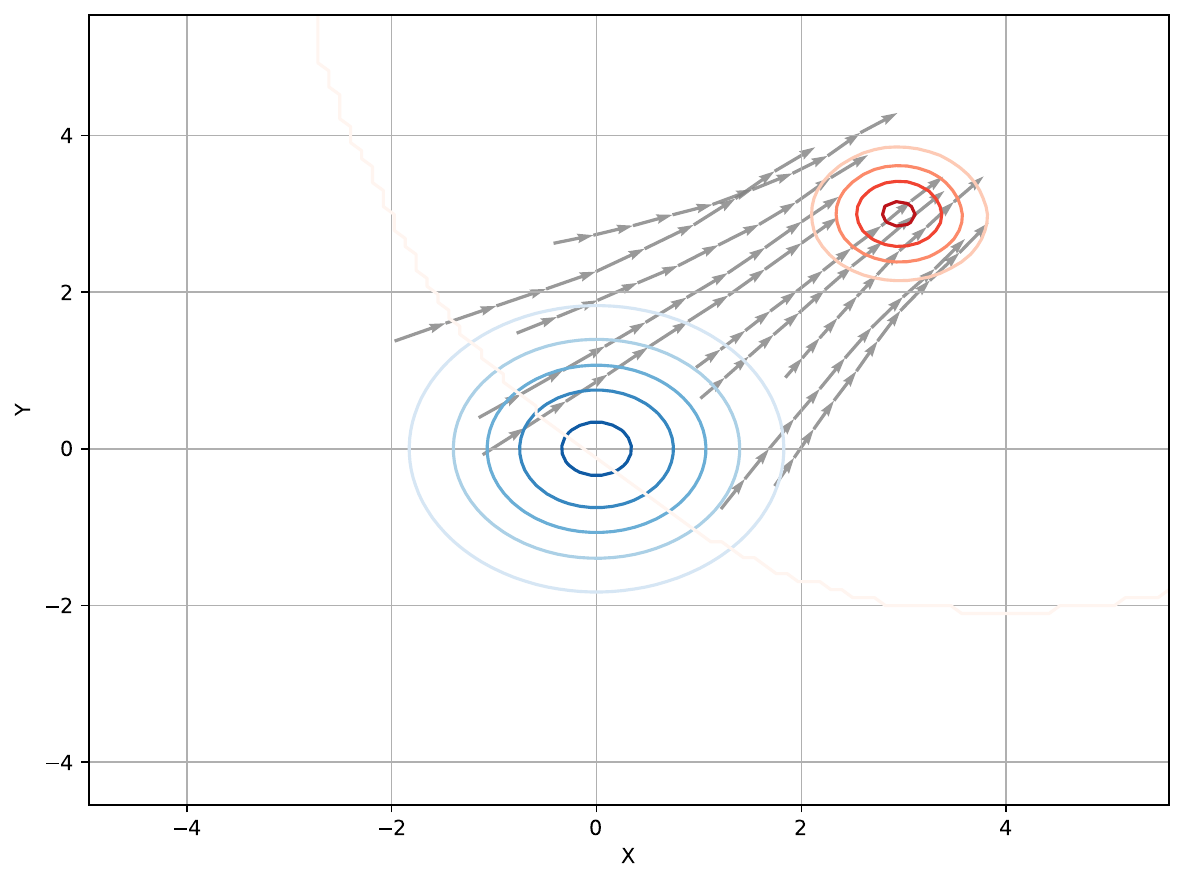}
    \caption{Transport map learned by Algorithm~\ref{alg:gMAM} with $K=9$ layers for the 2D isotropic Gaussian case. The  flow map (indicated by the arrows) transports the initial density toward equilibrium while preserving isotropic structure.}
    \label{fig:quad2d:transport}
\end{figure}

With the target $\rho_\infty=\mathcal{N}(x;[3,3]^\top,0.25I_2)$ and base $\rho_0=\mathcal{N}(0,I_2)$, the exact solution remains Gaussian with
$
\mu(t)=\mu(1-e^{-4t})
$
and
$
\Sigma(t)=0.25I_2+0.75e^{-8t}I_2.
$
We first show the results of GenWGP from Algorithm~\ref{alg:gMAM}. 
Fig.~\ref{fig:quad2d:transport} illustrates the   transport map learned by Algorithm~\ref{alg:gMAM} with $K=9$ stacked layers, which produces a smooth contraction flow toward equilibrium. 
Fig.~\ref{fig:quad2d:time}  validates the learned curve is indeed nearly an arc-length parametrized Wasserstein gradient flow: the panel (a) shows that the Wasserstein distance 
between each neighboring layers (segment length) is well maintained close to  constant, indicating a good quality of  arc-length parametrization; 
the panel (b) confirms that  the cosine alignment between $\partial_\tau p_\tau $ and $-\nabla_{\d_\mathcal{W}}\mathcal{F}(p_\tau)$ remains close to one,   verifying  the key parallel condition for the gradient flow: $ \partial_\tau p_\tau \propto -\nabla_{\d_\mathcal{W}}\mathcal{F}(p_\tau)$.
These two quantities are  numerically computed  as in Eqn.   \eqref{eqn:prob_lagrange} as follows  $\partial_\tau p_{\tau_k}\approx \frac{ \|\Phi_{k+1}-\Phi_{k}\|_{L^2(\rho_0)}  } {\tau_{k+1}-\tau_k}$ and 
$-\nabla_{\d_\mathcal{W}}\mathcal{F}(p_{\tau_k})\approx \frac{ \left\| \mathcal{V}[p_{k+1}] \circ \Phi_{k+1} \right\|_{L^2(\rho_0)} + \left\| \mathcal{V}[p_k] \circ \Phi_k \right\|_{L^2(\rho_0)} }{2}$.

\begin{figure}[pos=!htbp]
    \centering
     \subfigure[Arc-length (segment norm) between each pair of neighbouring layers]{%
        \includegraphics[width=0.41\linewidth]{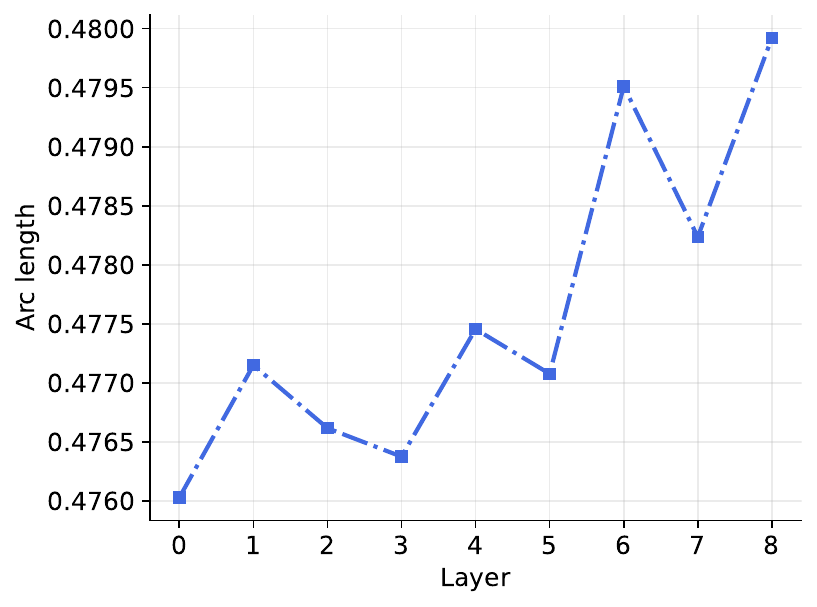}
    }
    \subfigure[Cosine alignment between tangent and negative gradient  at each  layer]{%
        \includegraphics[width=0.41\linewidth]{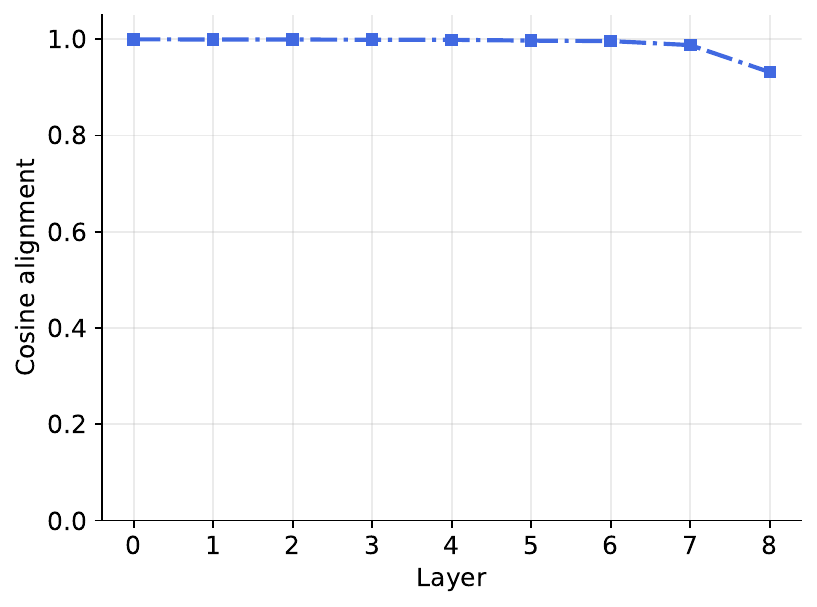}
    }
         \subfigure[Density snapshots (numerical  vs.\ exact pdf at recovered times)]{%
\includegraphics[width=0.9\linewidth]{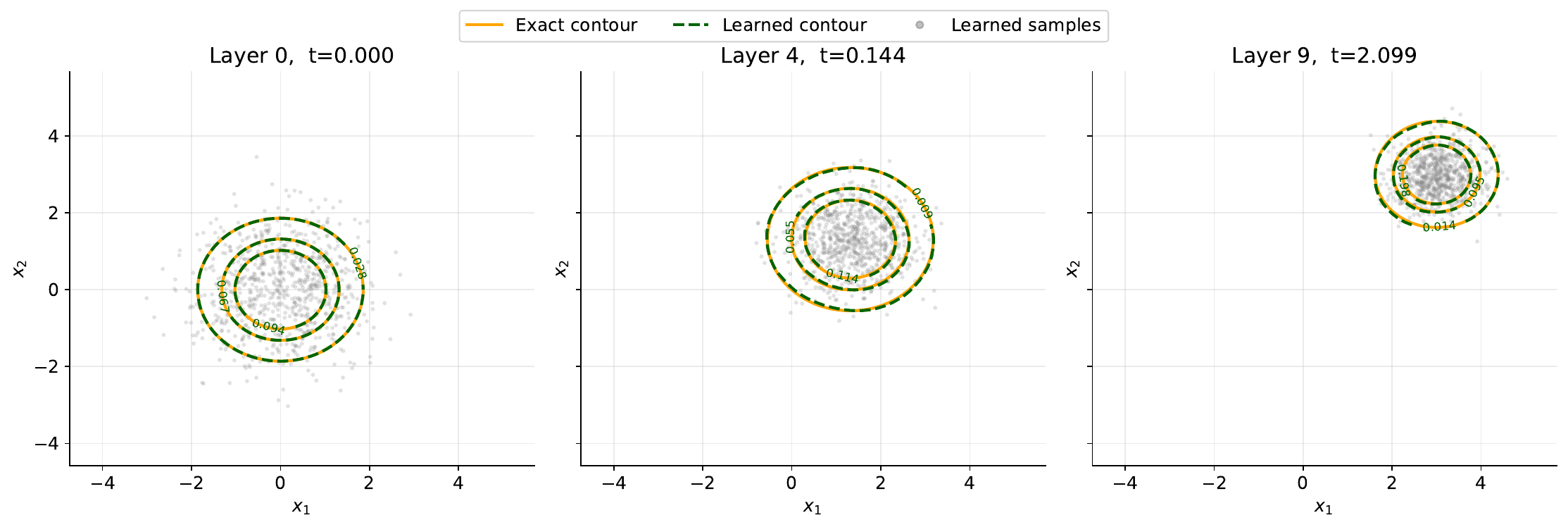}}

    \caption{Validation of the learned geometric path for the 2D isotropic Gaussian example. (a): nearly constant segment lengths indicate approximate arc-length parametrization; (b): cosine alignment close to one is consistent with the gradient-flow direction; and (c) the density snapshots at two selected times agree well  with the exact solution at the recovered physical times.}

    \label{fig:quad2d:time}
\end{figure}

To compare with  the true gradient flow, we recover the physical times $t_0, t_1, \ldots, t_{K-1}$ using   Algorithm~\ref{alg:recover_wasserstein_flow} from the  geometric path $p_{\tau_k}$ ($0\le k\le K$),
thereby aligning the numerical solution with the true solution as a function of time.  Fig.~\ref{fig:quad2d:time}(c) shows the contours of the density snapshots at two selected physical times, comparing the numerical and true solutions. Excellent agreement is observed, demonstrating that our geometric GenWGP accurately captures the dynamics even when evaluated in terms of physical time.

We also  compare the  geometric  formulation (Algorithm~\ref{alg:gMAM})  with its physical-time counterpart (Algorithm~\ref{alg:MAM}) (using $T=1$) under the same network architecture ($K=9$) and training setup, demonstrating the consistency of the two numerical gradient paths while highlighting their distinct characteristics.
Fig. \ref{fig:quad2d:index}a presents that the physical-time method uses a uniform discretization on $[0,T]$ with $T=1$ which is sufficiently large here to approach the equilibrium,  whereas the geometric method recovers  a non-uniform time mesh  but  adopts uniform in Wasserstein arc-length. 
The results in the panel (b)(c)   in Fig.~\ref{fig:quad2d:index} are  the decay of free energy $\mathcal{F}(p)$ in terms of $t$ and $\tau$ respectively, for the 
numerical paths from these two methods and the truth WGF. 
 In particular, more images are placed in the early stage where the free energy decays rapidly, leading to a more balanced distribution of resolution along the relaxation path.  The gap of free energy between two neighboring discrete layers is more even in the geometric approach than the uniform physical-time approach.
 The accuracy of the two methods  measured  in  $\d_\mathcal{W}$ error is validated by Fig.~\ref{fig:quad2d:index}c.

\begin{figure}[pos=htbp]
    \centering
    \subfigure[Recovered physical time vs.\ layer]{%
        \includegraphics[width=0.22\linewidth]{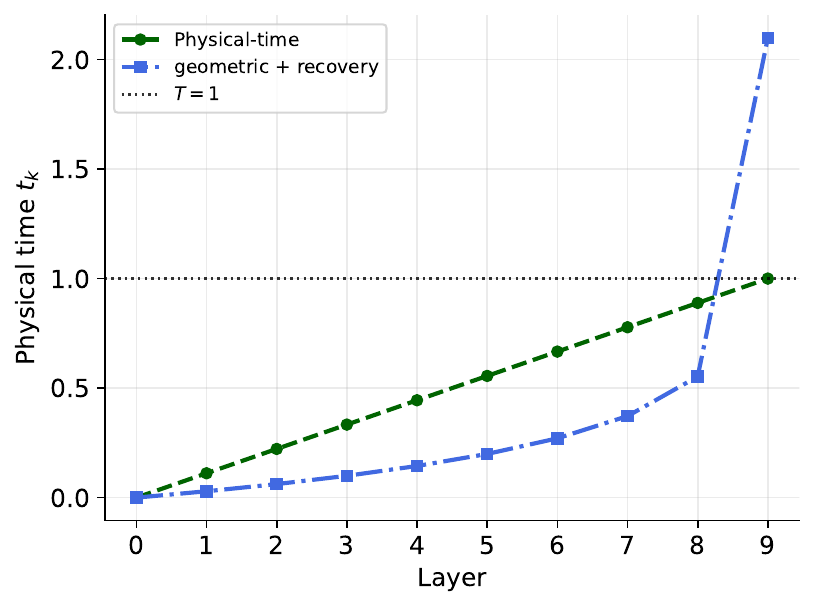}
    }
    \subfigure[Free energy vs.\ physical time]{%
        \includegraphics[width=0.22\linewidth]{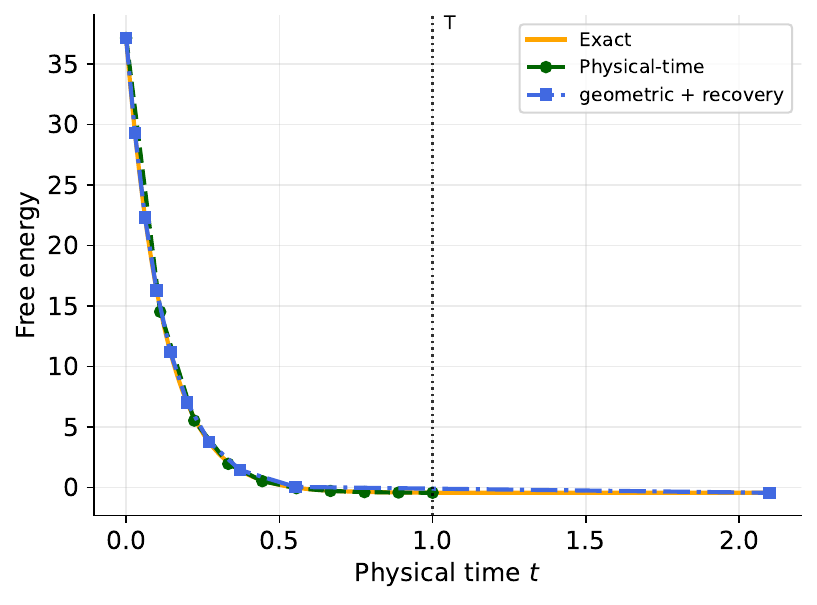}
    }
    \subfigure[Free energy vs.\ layer]{%
        \includegraphics[width=0.22\linewidth]{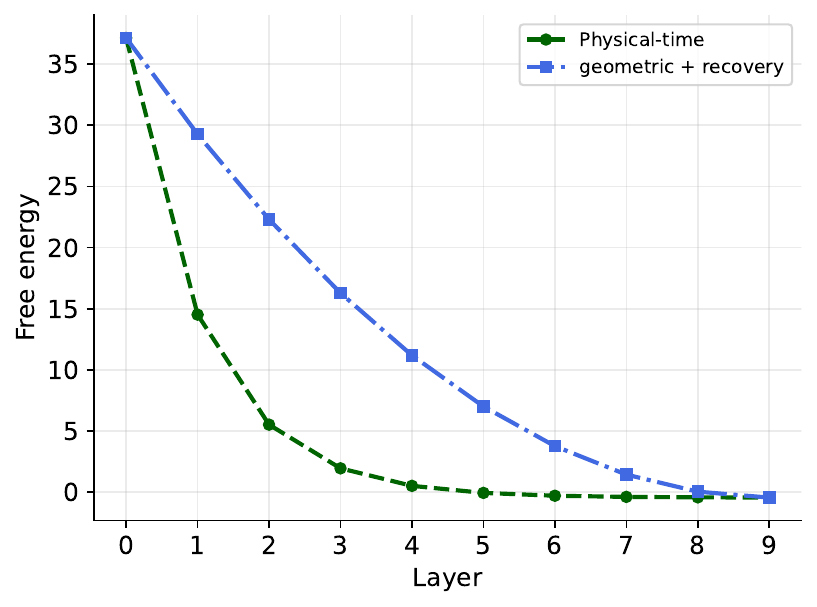}
    }
    \subfigure[$\d_\mathcal{W}$ error vs.\ physical time]{%
        \includegraphics[width=0.22\linewidth]{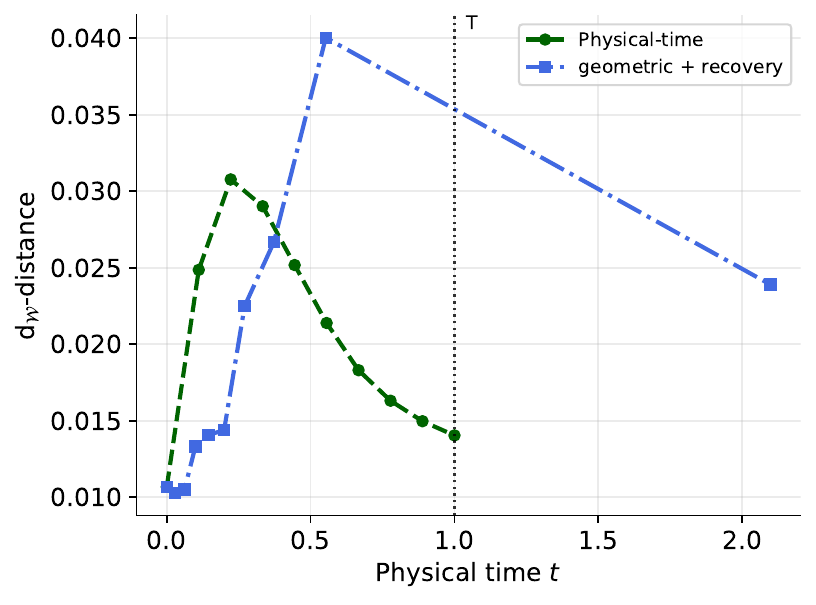}
    }
    \caption{
    Comparison between the physical‑time formulation   and the geometric formulation   for the 2D isotropic Gaussian example. (a):  the uniform time mesh for each layer in Algorithm~\ref{alg:MAM} and the recovered physical time mesh in Algorithm~\ref{alg:gMAM};
    (b): the decay of free energy plot in physical time;  (c) the decay of free energy plot in layers; (d) the $\d_\mathcal{W}$ errors. 
   }
    \label{fig:quad2d:index}
\end{figure}

\paragraph{2D anisotropic diffusion.}
For $\mu=[3,3]^\top$ and $\Sigma=\mathrm{diag}(1,0.25)$, the exact solution remains Gaussian with
\[
\mu(t)=[3(1-e^{-t}),3(1-e^{-4t})]^\top,
\qquad
\Sigma(t)=\mathrm{diag}(1,0.25+0.75e^{-8t}).
\]
Compared with the isotropic case, this example is a bit more challenging because the two coordinate variables evolve on distinct time scales to take longer time to reach equilibrium.  As in the isotropic case, we compare  in Fig.~\ref{fig:quad2d:aniso} the geometric formulation (blue curve) against the matched physical-time approach (red curve) with  $T=1$  under the same architecture and training setup. 
 The recovered physical time from the geometric path is now much longer than $T=1$.  Consequently,  the panel (c) shows that  the geometric path achieves a lower free energy value.  The comparison of  free energy decays up to $T=1$ confirms that  the geometric path resolves both the fast initial transient and the slower remaining relaxation. The $\d_\mathcal{W}$ error curves in the panel (d) show that, on the interval $[0,1]$, the recovered geometric method has slightly less accurate  than  the physical-time path method, owing to the fewer discrete points available on the geometric path within     $[0,1]$. 
  This accuracy can be straightforwardly improved, as discussed in Remark~\ref{rm4}.
Finally, the density snapshots  in the panel (e) further confirm the accuracy of the geometric WGF path when benchmarked against the true solution.

\begin{figure}[pos=htbp]
    \centering
    \subfigure[Recovered physical time vs.\ layer]{%
        \includegraphics[width=0.22\linewidth]{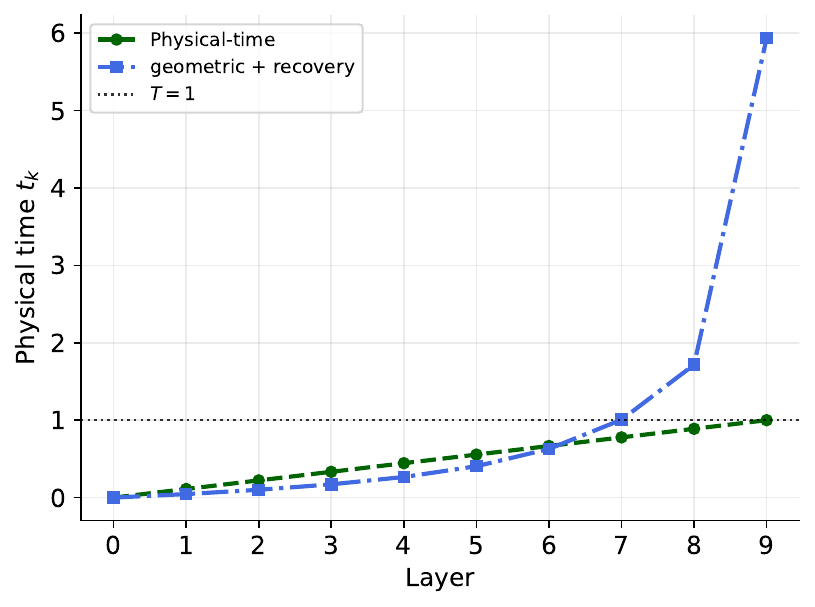}
    }
    \subfigure[Free energy vs.\ physical time]{%
        \includegraphics[width=0.22\linewidth]{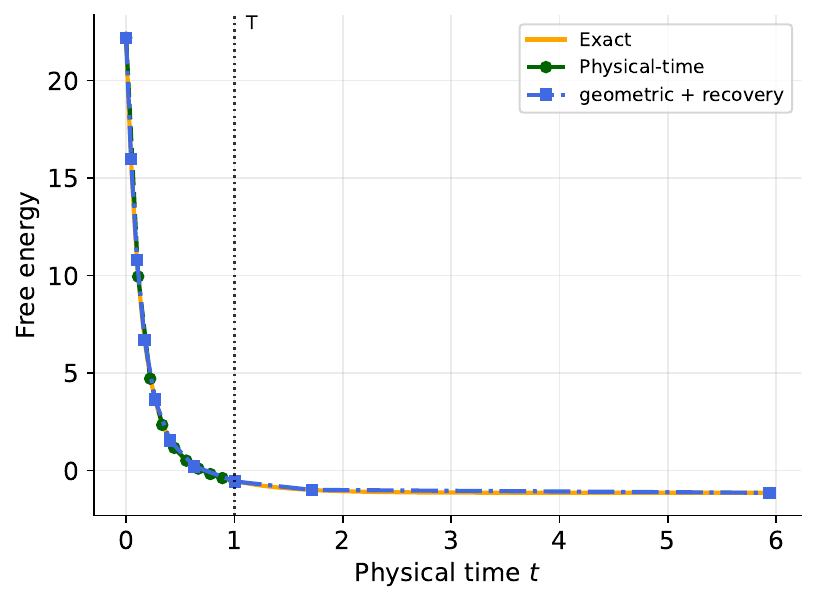}
    }
    \subfigure[Free energy vs.\ layer]{%
        \includegraphics[width=0.22\linewidth]{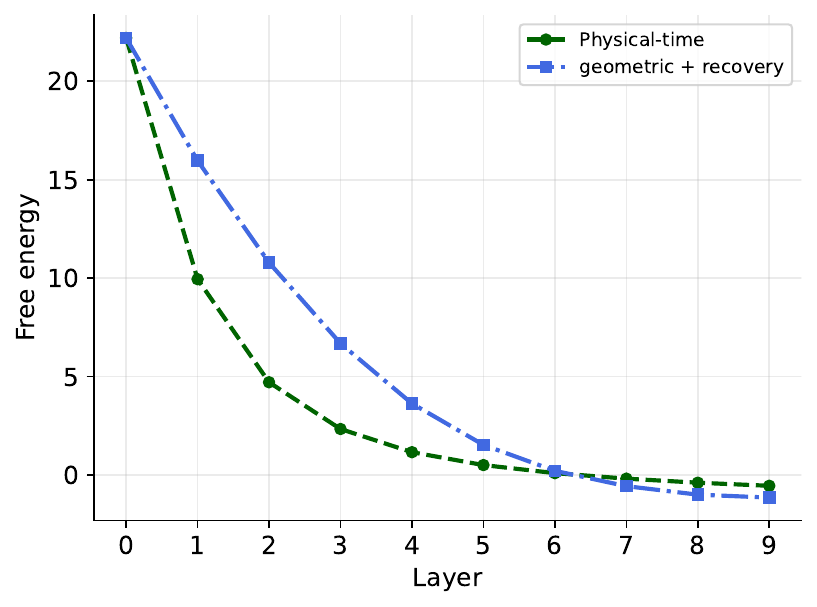}
    }
    \subfigure[$\d_\mathcal{W}$ error vs.\ physical time]{%
        \includegraphics[width=0.22\linewidth]{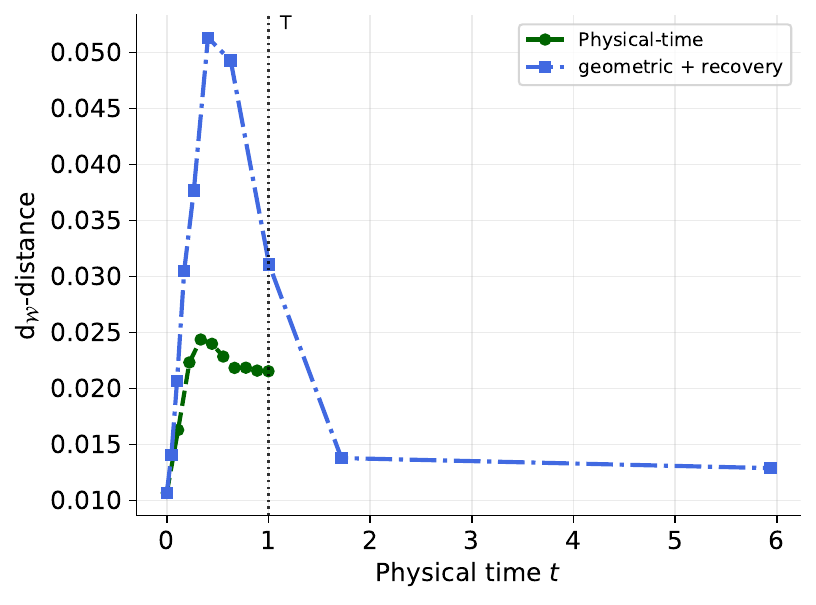}
    }

    \subfigure[Density snapshots]{%
\includegraphics[width=0.9\linewidth]{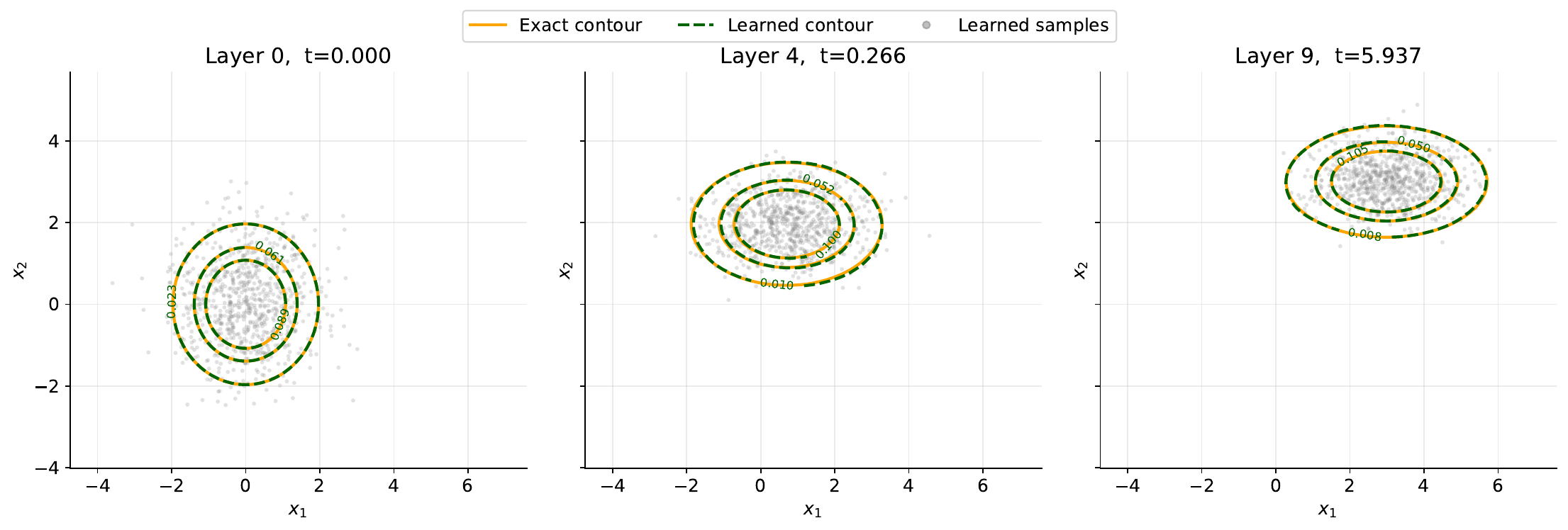}}
\caption{Comparison in the anisotropic Gaussian case. 
}
    \label{fig:quad2d:aniso}
\end{figure}
\paragraph{10D  diffusion.}
Let $\mu=(1,1,0,0,1,2,0,0,2,3)^\top$ and $\Sigma=\mathrm{diag}(\Sigma_A,I_2,\Sigma_B,I_2,\Sigma_C)$, where
\[
\Sigma_A=
\begin{bmatrix}
5/8 & -3/8\\
-3/8 & 5/8
\end{bmatrix},
\qquad
\Sigma_B=
\begin{bmatrix}
1 & 0\\
0 & 0.25
\end{bmatrix},
\qquad
\Sigma_C=0.25I_2.
\]
With the initial distribution $\rho_0=\mathcal{N}(0,I_{10})$, the exact solution is $\mathcal{N}(\mu(t),\Sigma(t))$, where
\begin{align*}
\mu(t)=&(1-e^{-t},1-e^{-t},0,0,1-e^{-t},2(1-e^{-4t}),0,0,2(1-e^{-4t}),3(1-e^{-4t}))^\top,\\
\Sigma(t)=&\mathrm{diag}(\Sigma_A(t),I_2,\Sigma_B(t),I_2,\Sigma_C(t)), \\
\text{with} \quad & \Sigma_A(t) = 
\begin{bmatrix}
\frac{5 + 3e^{-4t}}{8}  & -\frac{3+ 3e^{-4t}}{8}   \\
-\frac{3+ 3e^{-4t}}{8}  & \frac{5+ 3e^{-4t}}{8} 
\end{bmatrix}, \quad 
\Sigma_B(t) = 
\begin{bmatrix}
1 &  \\
 & \frac{1 + 3e^{-8t}}{4}
\end{bmatrix}, \quad 
 \Sigma_C(t) = 
\begin{bmatrix}
\frac{1 + 3e^{-8t}}{4} &  \\
& \frac{1 + 3e^{-8t}}{4}
\end{bmatrix}.
\end{align*}
This example examines whether the method remains accurate in a moderate-dimensional setting with coupled and anisotropic substructures. The learned flow
from Algorithm~\ref{alg:gMAM}  captures the expected rotated and anisotropic components in the selected two-dimensional projections; see Fig.~\ref{fig:quad10d:projection}. For accuracy, Fig.~\ref{fig:quad10d:error} reports the errors in the   mean and covariance against the exact Gaussian solution  over time. The errors remain small throughout the evolution, indicating that the geometric formulation retains good accuracy in this higher-dimensional but still exactly solvable setting.

\begin{figure}[pos=htbp]
    \centering
    \includegraphics[width=0.68\linewidth]{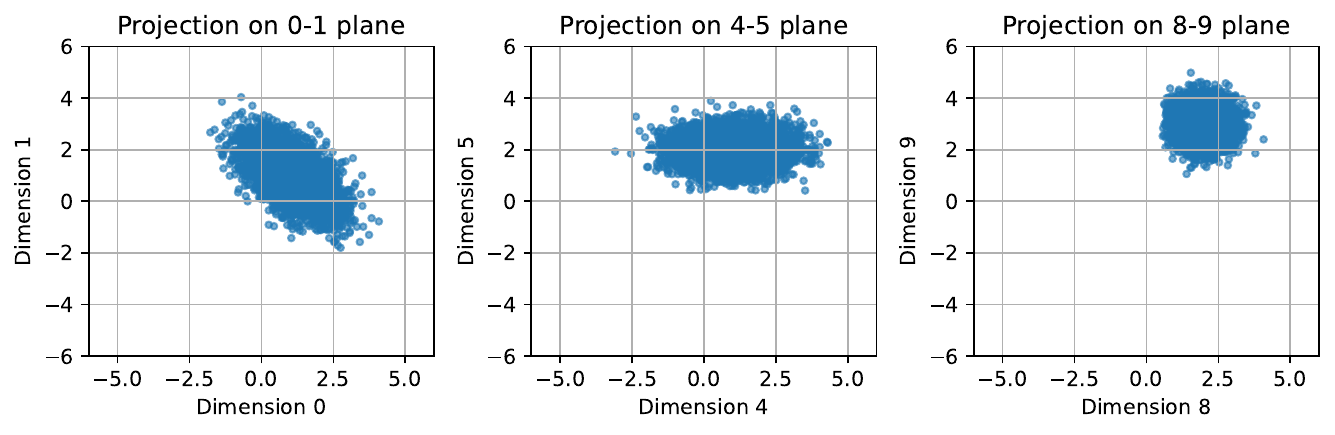}
    \caption{2D projections of the terminal distribution for the
    10D (dimension 0 to 9) block-structured Gaussian example.}
    \label{fig:quad10d:projection}
\end{figure}

\begin{figure}[pos=htbp]
    \centering
    \includegraphics[width=0.6\linewidth]{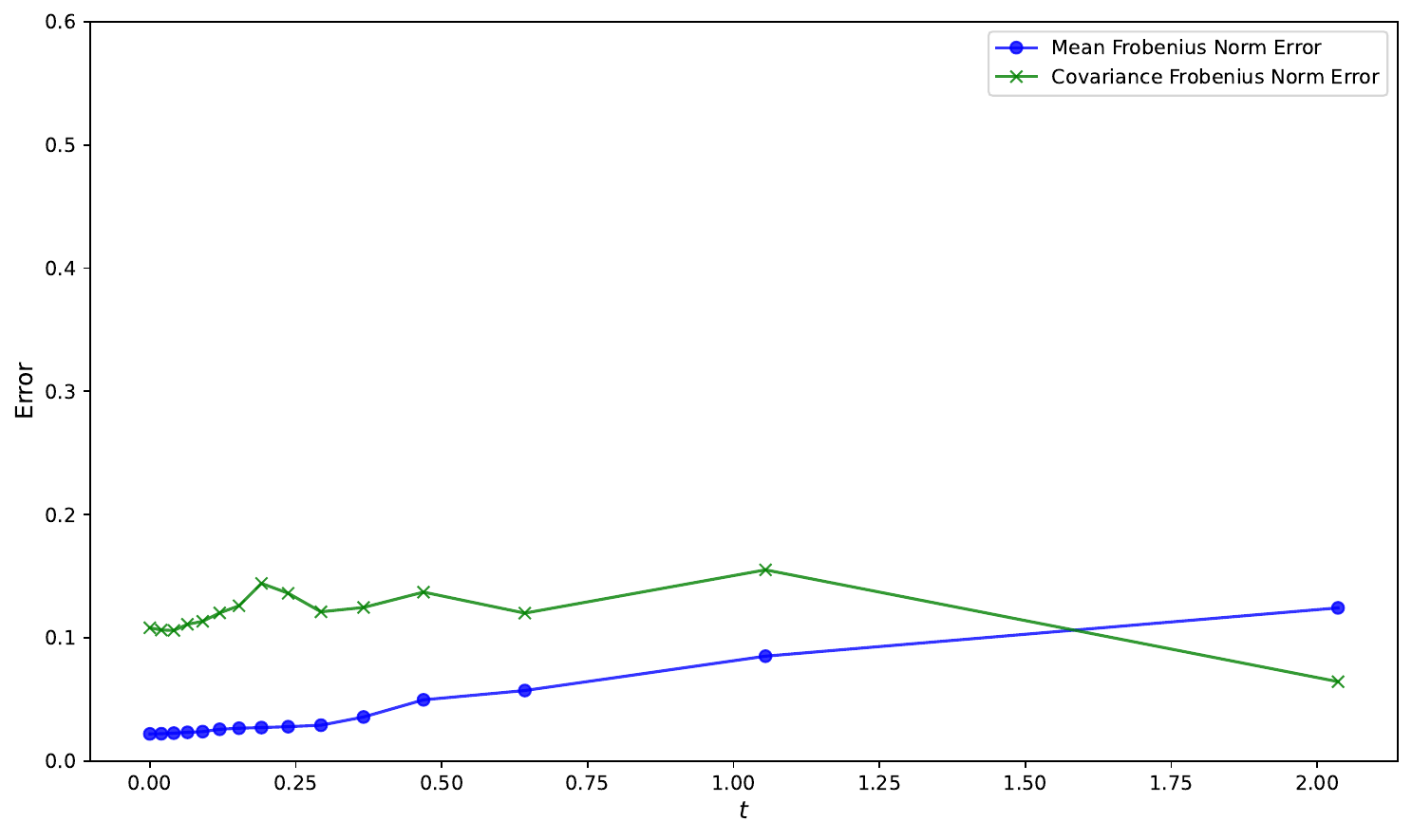}
    \caption{Absolute error of the mean and Frobenius norm of the covariance error vs.\ recovered time in the 10D block-structured Gaussian example.}
    \label{fig:quad10d:error}
\end{figure}

\subsubsection{Non-Convex Potential: 10D Styblinski-Tang Potential}
We next consider the 10D Styblinski-Tang potential, given by a sum of identical one-dimensional potentials over each coordinate:
\[
V(x)=\frac{3}{50}\left(\sum_{i=1}^d x_i^4-16x_i^2+5x_i\right),
\qquad
x=(x_1,\dots,x_{10})\in\mathbb{R}^{10}.
\]
The initial is the standard Gaussian measure. 
Because of permutation symmetry, all one-dimensional marginals $p_\tau^{(i)}(x_i)$ are statistically identical, thus $p_{\tau}(x)= \Pi_i p_\tau^{(1)}(x_i)$. 
 
\paragraph{Parameterization and visualization.}
We parameterize the path by a non-uniform B-spline Flow~\cite{hong2023neural} with two hidden layers (width $100$, SiLU activation), which provides smooth $C^2$-diffeomorphic transports with controlled regularity. Fig.~\ref{fig:styblinski:samples} shows a representative two-dimensional projection $(x_5,x_6)$ of the particle evolution along the path, illustrating the transition from a unimodal Gaussian to a complex multimodal distribution.

\begin{figure}[pos=htbp]
    \centering
    \includegraphics[width=1\linewidth]{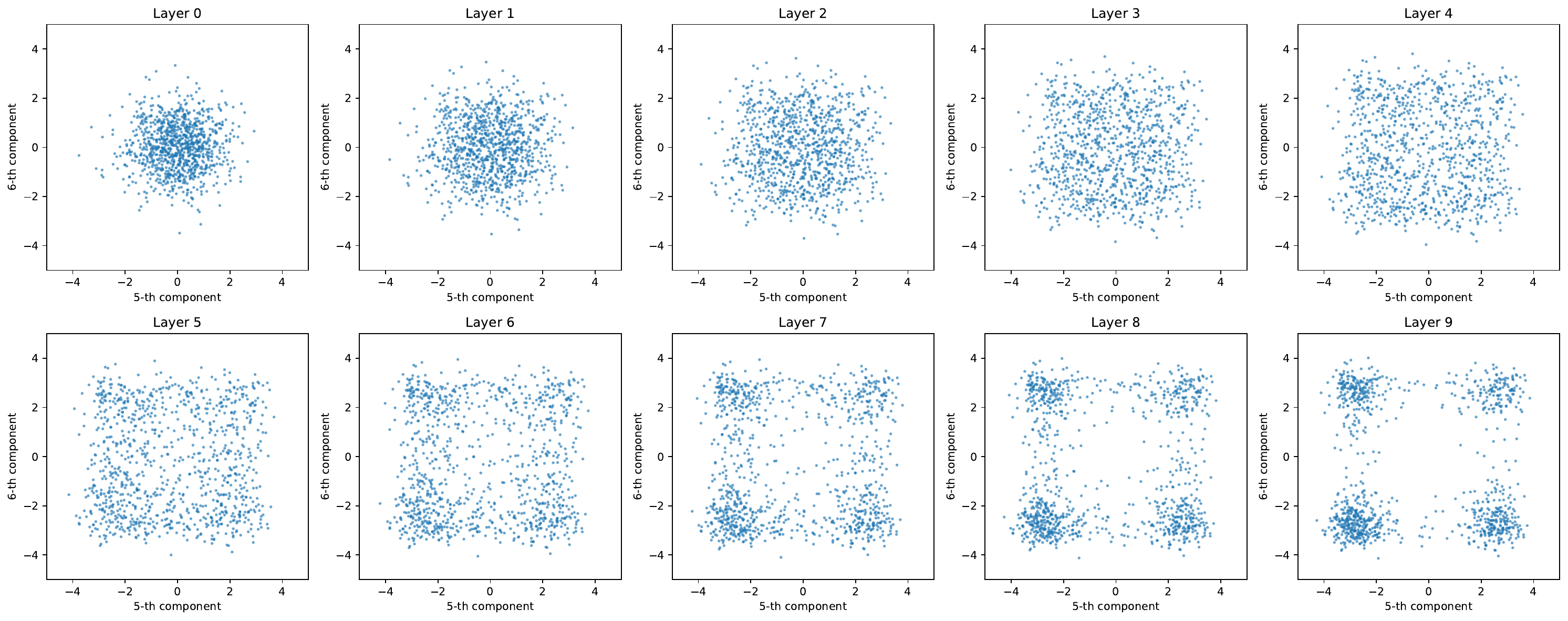}
    \caption{Sample points projected onto the $(x_5, x_6)$-plane at each layer along the learned geometric path for the 10D Styblinski-Tang potential.}
    \label{fig:styblinski:samples}
\end{figure}
\paragraph{Reference solution and comparison.}
 Each one-dimensional marginal evolves independently according to the one-dimensional Fokker-Planck equation, equivalently the over-damped Langevin SDE $$
\mathrm{d}X_t=-V_1'(X_t)\mathrm{d}t+\sqrt{2}\mathrm{d}B_t,
\qquad
V_1(X)=\frac{3}{50}\left(X^4-16X^2+5X\right),
\qquad
X_t\in\mathbb{R}.
$$
Unlike the Ornstein–Uhlenbeck process, this SDE has no analytical expression of the density evolution. We therefore   simulate  $5000$ Euler-Maruyama trajectories with time step $10^{-3}$, and compare the resulting empirical marginal density with the learned one-dimensional marginals at matched physical times. Fig.~\ref{fig:styblinski:marginals} shows close agreement, including pronounced non-Gaussian and multimodal features. This provides quantitative evidence that the learned transport geometry remains accurate in a high-dimensional nonconvex setting, at least at the marginal level made accessible by the separable structure.
\begin{figure}[pos=htbp]
    \centering
    \includegraphics[width=1\linewidth]{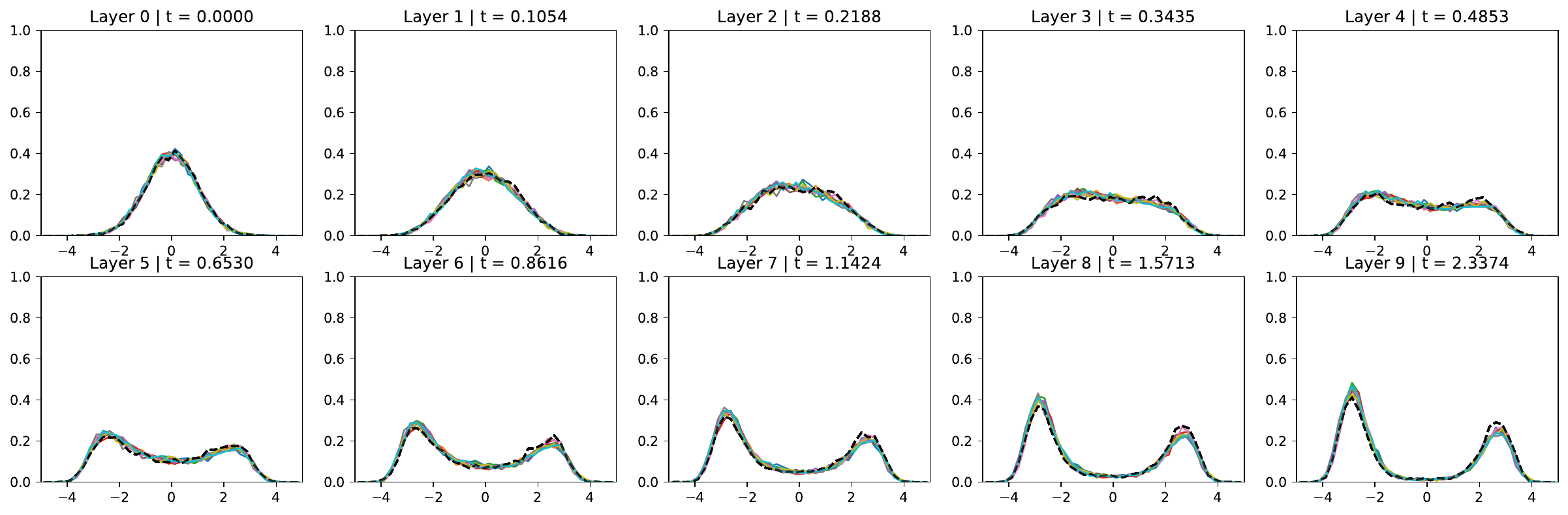}
    
    \caption{Comparison of marginal densities from the learned path (colored for each component) and from 1D SDE simulation (black) at each layer, with the recovered physical times indicated.}
    \label{fig:styblinski:marginals}
\end{figure}

\subsection{Interacting Particle Dynamics}

We next consider WGFs driven by nonlocal interaction energies with the pairwise term $
\frac{1}{2}\int_{\mathbb{R}^d\times\mathbb{R}^d}W(x-y)\rho(x)\rho(y)\mathrm{d}x\mathrm{d}y.$
Such models arise in aggregation, swarming, and mean-field dynamics. In contrast to the Fokker-Planck examples above, these systems typically do not admit explicit transient solutions and may even possess compactly supported equilibrium states. Our validation therefore focuses on problem-adapted quantities: exact steady-state information whenever available, recovered physical-time comparisons under matched training setups, and structural diagnostics that test whether the learned path remains consistent with Wasserstein gradient-flow behavior.

\subsubsection{Pure Aggregation}

\begin{figure}[pos=!htbp]
    \centering
    \includegraphics[width=0.88\linewidth]{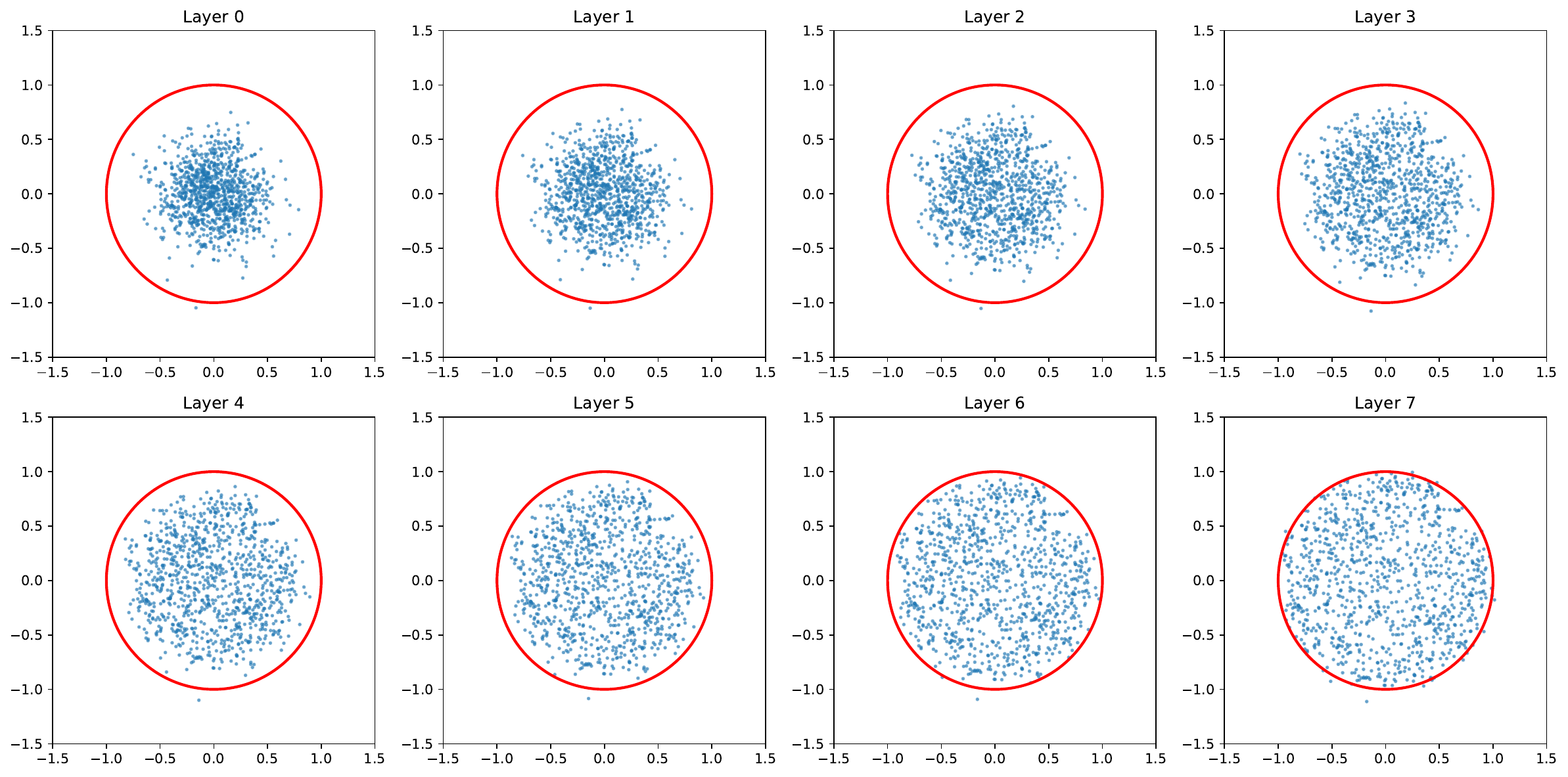}
    \includegraphics[width=0.88\linewidth]{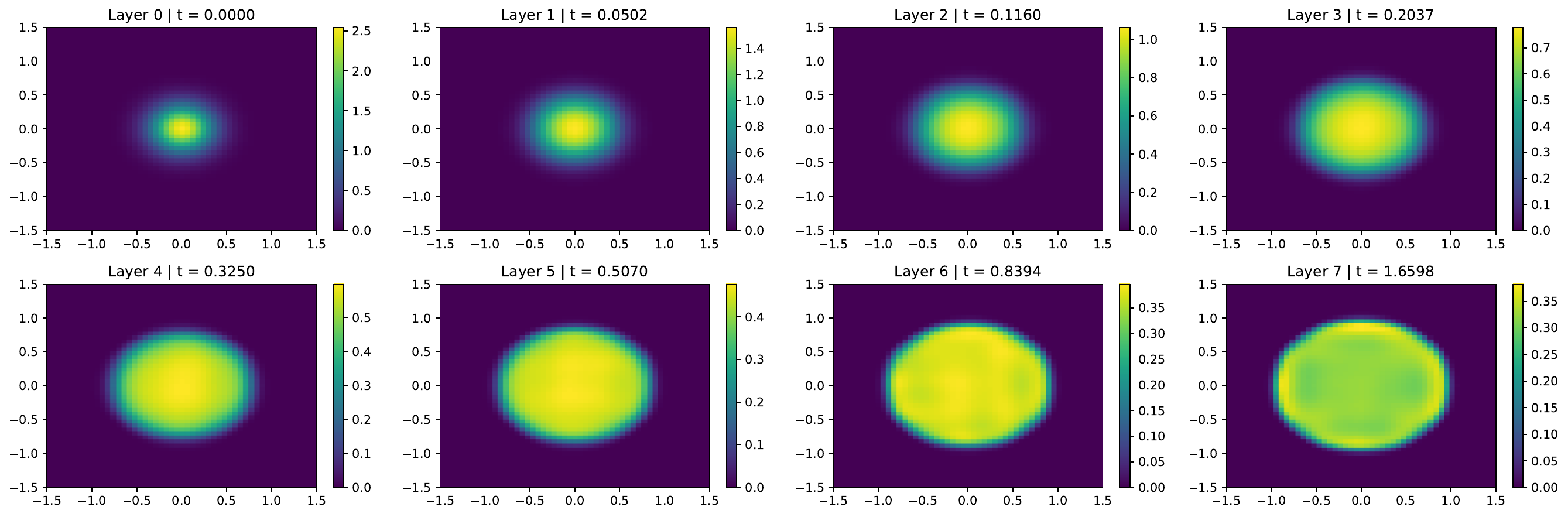}
    \caption{Pure aggregation: particle transport (top) and density evolution (bottom) along the learned geometric path. The terminal state approaches the uniform distribution supported on the unit disk.} 
    \label{fig:agg:transport}
\end{figure}

The first example is a two-dimensional pure aggregation model where
\[
\mathcal{F}(\rho)=\frac{1}{2}\int_{\mathbb{R}^2\times\mathbb{R}^2}W(x-y)\rho(x)\rho(y)\mathrm{d}x\mathrm{d}y,
\qquad
W(x)=\frac{1}{2}\|x\|^2-\log\|x\|.
\]
This kernel balances quadratic attraction with Newtonian repulsion, and its steady state is the uniform distribution on the unit disk. We initialize from $\mathcal{N}(0,0.25I_2)$ and use a seven-layer non-uniform B-spline Flow with $N=5000$ particles. Fig.~\ref{fig:agg:transport} displays the particle transport and density evolution along the learned path. The solution becomes radially symmetric, develops compact support, and approaches the analytical minimizer at terminal time.

\begin{figure}[pos=!htbp]
    \centering
\includegraphics[width=0.46\linewidth]{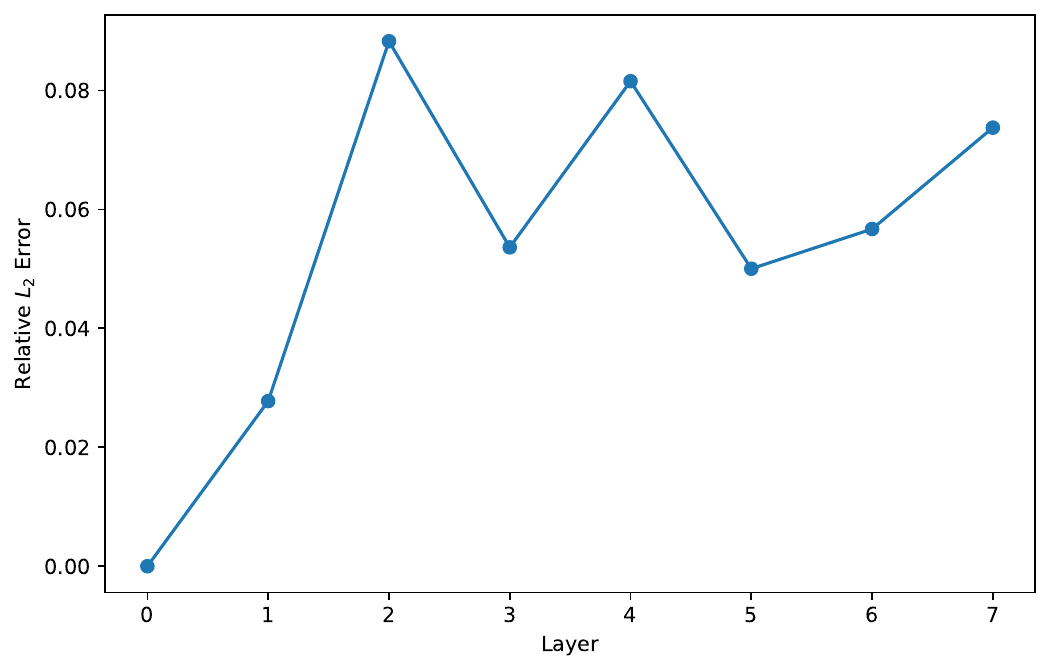}
\caption{Relative $L^2$ difference between the learned densities and the reference solution computed by the primal-dual method.}
    \label{fig:agg:comparison}
\end{figure}

There is no closed-form transient solution for this example. For quantitative validation, we therefore compute an independent numerical reference using the primal-dual scheme developed in ~\cite{carrillo2022primal} which is  grid-based and of time-marching type. After reparameterizing the learned geometric path to physical time via Algorithm~\ref{alg:recover_wasserstein_flow}, we compare the resulting densities with the primal-dual solution at the corresponding discrete times. Fig.~\ref{fig:agg:comparison} shows that the relative $L^2$ difference between the two solutions remains uniformly small throughout the evolution. This example therefore provides a direct quantitative comparison against an existing  numerical solver for a nontrivial interacting-particle dynamics.

\subsubsection{Aggregation-Drift Equation}
We next consider an aggregation model with an additional singular confining potential $V$, where the free energy is
\[
\mathcal{F}(\rho)=\int_{\mathbb{R}^2}V(x)\rho(x)\mathrm{d}x
+\frac{1}{2}\iint_{\mathbb{R}^2\times\mathbb{R}^2}W(x-y)\rho(x)\rho(y)\mathrm{d}x\mathrm{d}y,
\]
with
\[
W(x)=\frac{1}{2}\|x\|^2-\log\|x\|,
\qquad
V(x)=-\frac{\alpha_1}{\alpha_2}\log\|x\|.
\]
The explicit steady state \cite{byun2024planar} is the uniform distribution on the annulus with inner and outer radii
\(
R_i=\sqrt{\frac{\alpha_1}{\alpha_2}},
R_o=\sqrt{R_i^2+1}.
\)
Starting from a non-radially symmetric initial datum composed of five Gaussians, the learned geometric path restores radial symmetry and converges to the annular steady state; see Fig.~\ref{fig:aggdrift:transport}.

\begin{figure}[pos=htbp]
    \centering
    \includegraphics[width=0.88\linewidth]{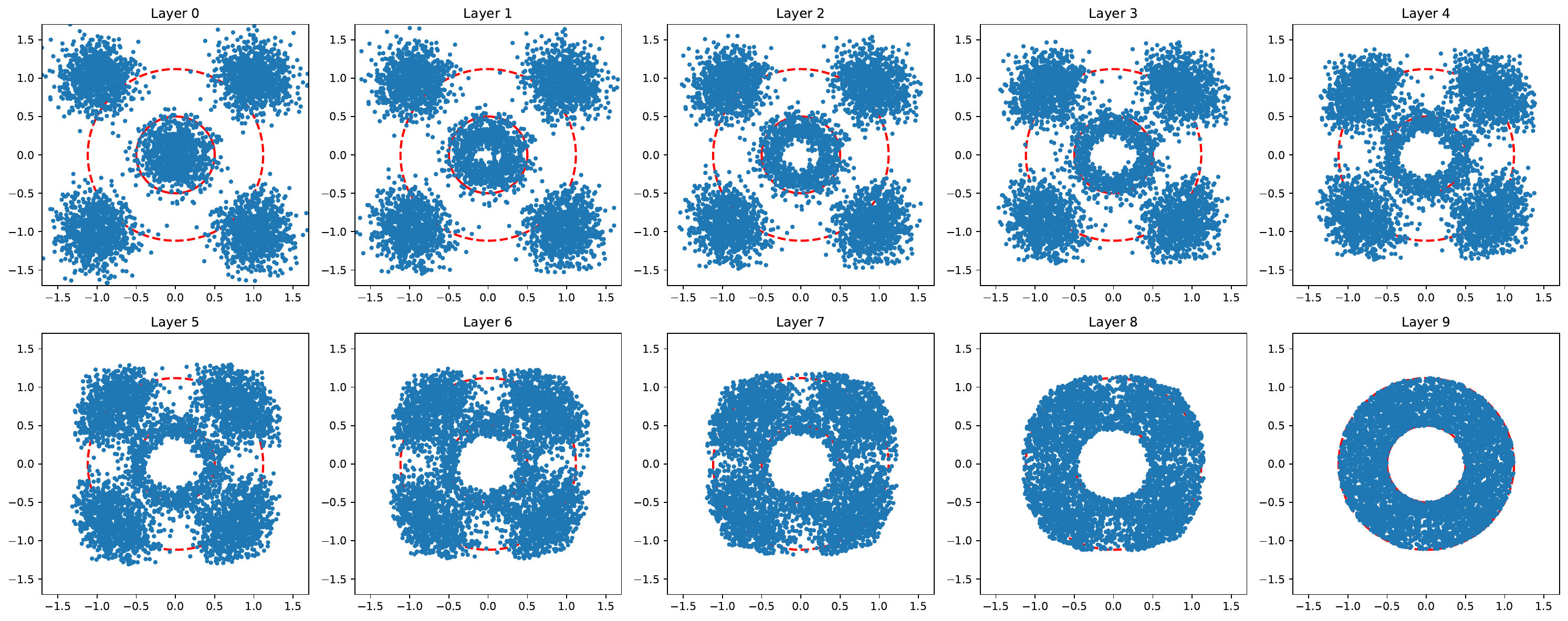}
    \caption{Aggregation-drift: particle evolution from an asymmetric initial distribution toward the annular steady state.}
    \label{fig:aggdrift:transport}
\end{figure}

We use this example  to  show how the geometric path obtained by Algorithm~\ref{alg:gMAM} can provide a better non-uniform time mesh as well as a good terminal time  for  the physical-time path minimization algorithm ~\ref{alg:MAM}, as discussed in Remark \ref{rm4}. More precisely,  we first compute a geometric path, then recover the corresponding physical times by Algorithm~\ref{alg:recover_wasserstein_flow}, and finally use this recovered non-uniform time mesh back in Algorithm~\ref{alg:MAM}. 
By comparing the numerical optimal paths under the uniform time mesh grid and this recovered non-uniform time mesh grid, we show the improvement of the accuracy Fig.~\ref{fig:aggdrift:timecompare}(b), where 
Fig.~\ref{fig:aggdrift:timecompare}(a) presents  the difference of these two time meshes in terms of each layer. Fig.~\ref{fig:aggdrift:timecompare}(b) presents the cumulative MAM loss along the discrete path. The recovered-time discretization yields a uniformly smaller cumulative loss, indicating that the recovered mesh provides a more effective physical-time representation of the same relaxation process.

\begin{figure}[pos=htbp]
    \centering
    \subfigure[Physical time vs.\ layer]{%
        \includegraphics[width=0.42\linewidth]{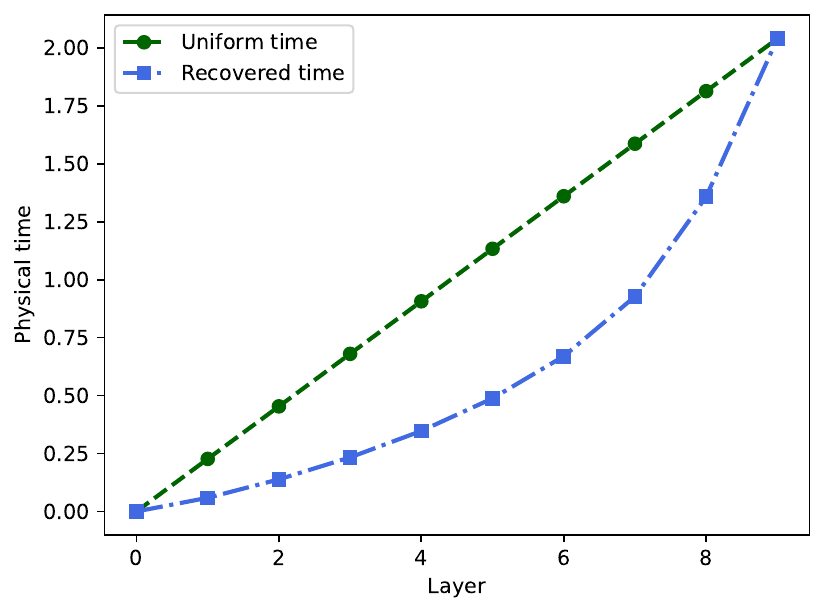}
    }
    \subfigure[Cumulative MAM loss vs.\ layer]{%
        \includegraphics[width=0.42\linewidth]{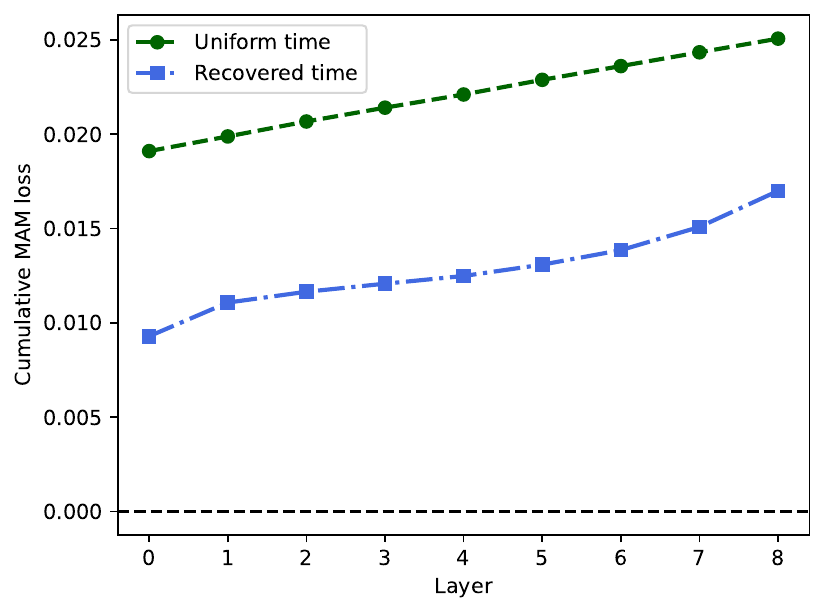}
    }
    \caption{Comparison between two physical-time discretizations for the aggregation-drift equation: the standard uniform-time mesh and the recovered time mesh obtained from the geometric path. The recovered time allocates more layers to the fast transient regime and produces a smaller cumulative MAM loss.}
    \label{fig:aggdrift:timecompare}
\end{figure}

To further examine the structure of the two learned paths, we report in Fig.~\ref{fig:aggdrift:diagnostics} three intrinsic diagnostics: the free-energy profile, the cosine alignment between the discrete path velocity and the negative Wasserstein gradient, and the Wasserstein distance between neighboring layers. The recovered-time path offers a more uniform drop of the free energy along the discrete path,   a cosine value closer to one, and a more balanced distribution of inter-layer distances. 
This example  illustrates the recommendation of Remark \ref{rm4} where the geometric method serves as an effective preprocessing step to further fine tune the physical-time discretization for more accurate path.

\begin{figure}[pos=htbp]
    \centering
    \subfigure[Free energy vs.\ layer]{%
        \includegraphics[width=0.3\linewidth]{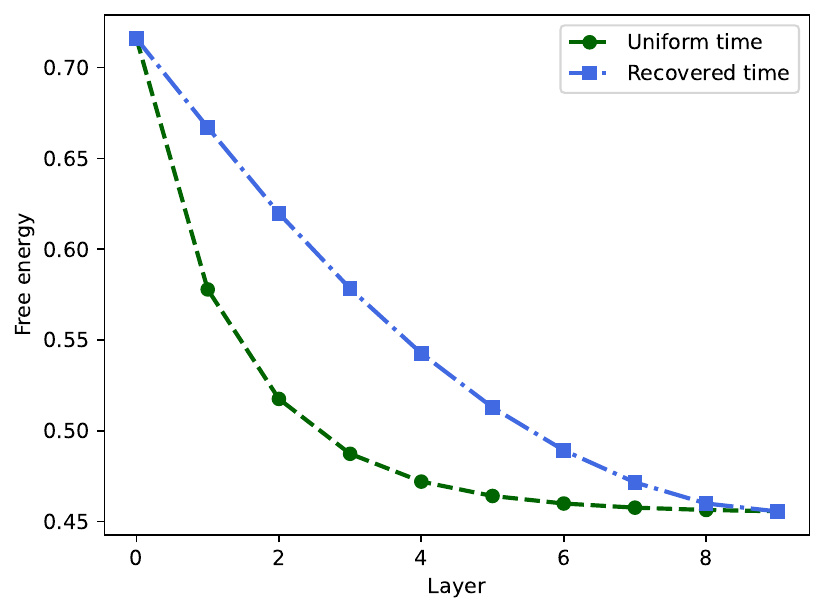}
    }
    \subfigure[Cosine alignment vs.\ layer]{%
        \includegraphics[width=0.3\linewidth]{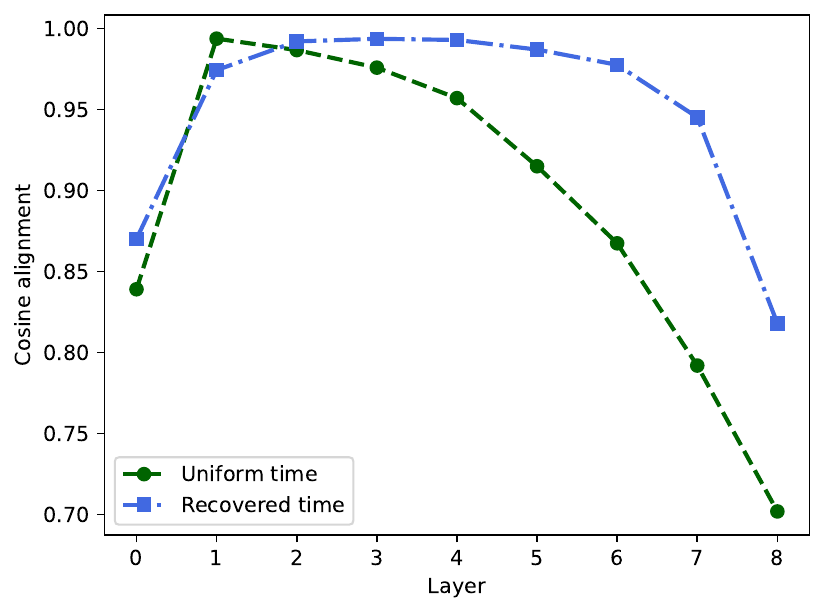}
    }
    \subfigure[Inter-layer distance vs.\ layer]{%
        \includegraphics[width=0.3\linewidth]{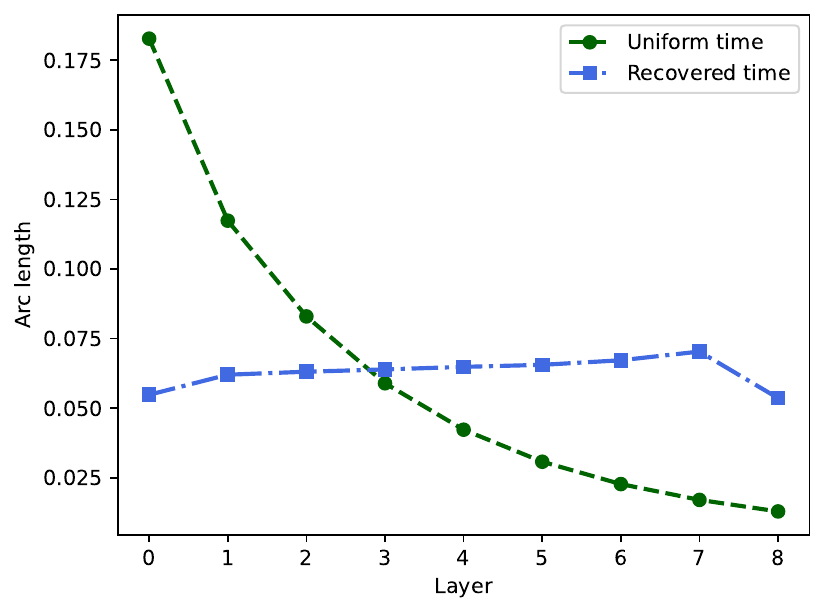}
    }
    \caption{Structural diagnostics for the aggregation-drift equation under uniform time and recovered time discretizations. From left to right: free energy, cosine alignment with $-\nabla_{\d_\mathcal{W}}\mathcal{F}$, and the Wasserstein distance between neighboring layers.}
    \label{fig:aggdrift:diagnostics}
\end{figure}

\subsubsection{Aggregation-Diffusion Equation}
We conclude this section with an aggregation-diffusion model that combines nonlinear diffusion and nonlocal attraction:
\[
\mathcal{F}(\rho)=\int_{\mathbb{R}^2}\frac{\nu}{m-1}\rho^m(x)\mathrm{d}x
+\frac{1}{2}\iint_{\mathbb{R}^2\times\mathbb{R}^2}W(x-y)\rho(x)\rho(y)\mathrm{d}x\mathrm{d}y,
\]
where the interaction kernel is
\[
W(x)=-\frac{1}{\pi}e^{-|x|^2}.
\]
This smooth radially symmetric kernel induces short-range attraction and is frequently used in models of biological aggregation and collective behavior. The corresponding evolution equation is
\[
\partial_t\rho=\nabla\cdot(\rho\nabla W * \rho)+\nu\Delta\rho^m.
\]
We choose $m=2$. To probe nontrivial intermediate dynamics toward the equilibirum, we take as initial datum the  characteristic function on $[-3,3]^2$  thus the constant  mass is  $9$. These  settings produce  a compactly supported constant-density profile that undergoes highly nontrivial transient dynamics before converging to equilibrium. 
We employ an equal \emph{arc-action} penalty (i.e., $\mathcal{F}(p_{k+1})-\mathcal{F}(p_{k})$ is constant) instead of the arc-length parametrization,  so that distributed discrete images (layers) along the path  
offer  better quality with only 11 layers. Fig.~\ref{fig:aggdiff:density} successfully discovers  the  complex  evolution of this system: the initial density first  splits into four localized clusters,  merging  into a single radially symmetric steady state, reflecting the balance between attraction and diffusion in the underlying dynamics.
\begin{figure}[pos=htbp]
    \centering
\includegraphics[width=1.1\linewidth]{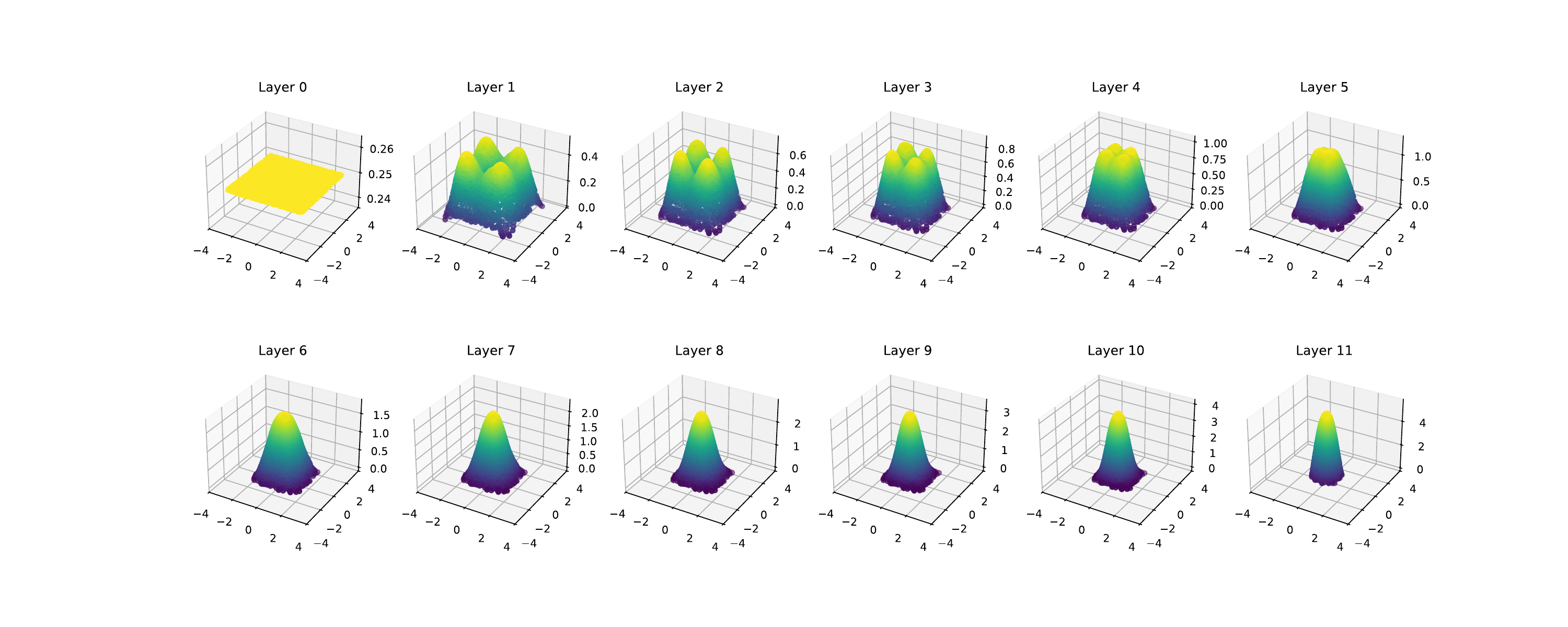}
    \caption{Aggregation-diffusion: density snapshots along the learned geometric path. Multi-bump transient states gradually merge into a smooth radial steady state.}
    \label{fig:aggdiff:density}
\end{figure}

\begin{figure}[pos=htbp]
    \centering
    \includegraphics[width=0.4\linewidth]{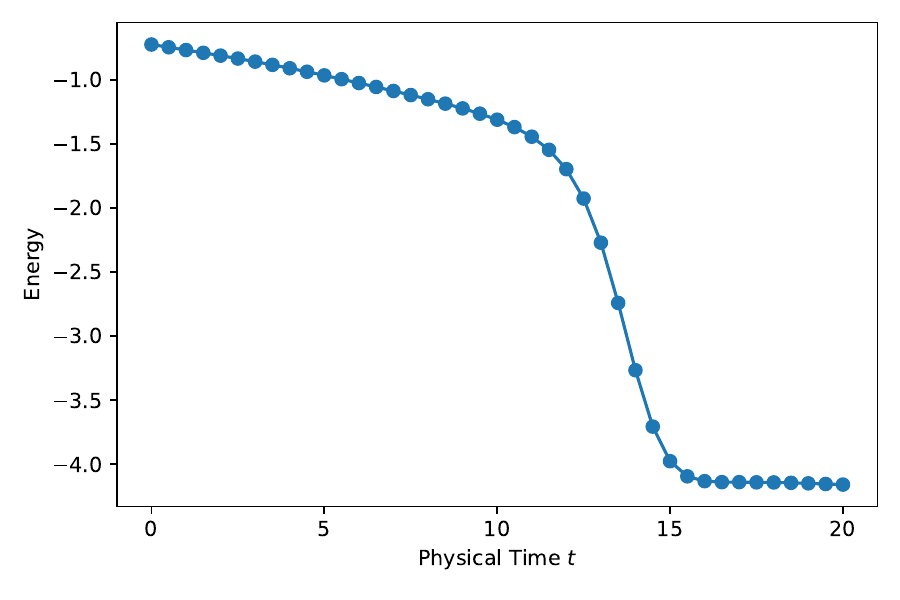}
    \includegraphics[width=0.4\linewidth]{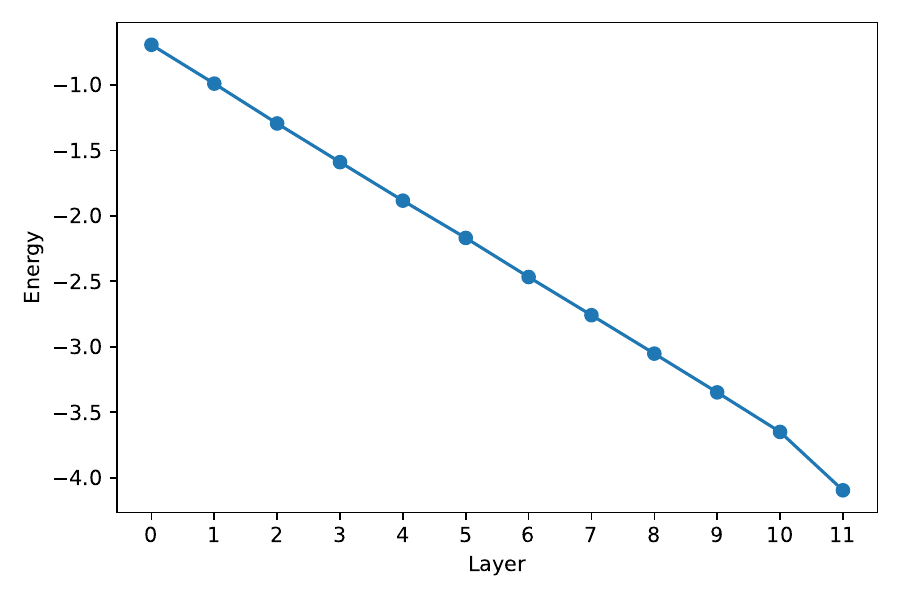}
    \caption{Free-energy profiles along the gradient flow    path. Left: the evolution computed by a conventional primal-dual scheme with a  constant   time step size $0.5$. Right: evolution along the learned geometric path with only 11 images.}
    \label{fig:aggdiff:energy}
\end{figure}
 
Fig.~\ref{fig:aggdiff:energy} compares the free-energy profiles computed from two methods under two different parameterizations. The left panel shows the energy evolution produced by a conventional primal-dual scheme~\cite{carrillo2022primal} with fixed physical time step $\Delta t=0.5$, which exhibits distinct dynamical phases: a relatively slow initial descent, a rapid drop during bump formation, and then slower relaxation. The right panel shows the energy profile along the learned geometric path with arc-action parametrization, where the energy decreases much more evenly along the gradient flow path. 
This example is included as a stress test of the central geometric claim of the paper: when the relaxation contains strongly nonuniform dynamical phases, an equal arc-action parametrization yields a more informative distribution of path images than a fixed physical-time mesh.

%% file: 6-Conclusion.tex
\section{\label{sec:Conclusion} Conclusion}

We have proposed an efficient and robust computational method for the entire Wasserstein gradient flows from a global path‑optimization perspective.
The introduced GenWGP is a generative framework that represents the entire probability trajectory through a Lagrangian normalizing‑flow parametrization of transport maps, enabling to efficiently  approximate the long‑time relaxation process toward equilibrium without relying on sequential time marching. 

Our proposed path‑finding approach is based on the transport-flow-based numerical scheme  for the 
Dawson-Gartner action functional in the large deviation theory, which is presented  in two  forms: one is for the path  in physical time within any specified  finite time horizon,
and the other is the parametrization‑free geometric formulation. The physical‑time formulation provides a horizon‑based variational description together with a trainable discrete objective built from Monte Carlo particle approximation and a Crank–Nicolson‑type discretization. The geometric formulation furthermore automatically determines the final time and is capable of capturing the entire gradient path up to equilibrium, yielding a practical arc‑length or free‑energy parametrized curve. Both methods achieve well‑approximated paths for which time‑stepping methods would require a significantly larger number of steps to attain comparable accuracy.

Our analysis shows that these variational formulations and algorithms exhibit mathematically controlled errors. In particular, we derived an a priori KL‑divergence estimate for the physical‑time formulation, established a trajectory‑error decomposition for its discrete scheme, and proved consistency of the discrete geometric objective. As demonstrated by numerical results on representative Fokker–Planck and interacting aggregation particle models, these findings indicate that GenWGP can approximate both relaxation trajectories and terminal equilibrium states in a stable and computationally efficient manner.

The least‑action principle underlying our method naturally extends beyond Wasserstein gradient flows to more general non‑equilibrium systems in probability measure space, even in the absence of a free energy or in the presence of additional non‑conservative forces. The corresponding action functionals retain a similar structure, and both the physical‑time and geometric formulations developed in this paper remain applicable—particularly in scenarios where the terminal state is prescribed and the noise‑induced optimal transition path is of primary interest. We leave this promising application for future investigation.

%% file: proof_KL_bound.tex
\begin{assumption}\label{ass:densitytheta}
\hfill
    \begin{enumerate}
    \item The domain $\Omega$ is a bounded domain of finite measure (in particular, $\mathbb T^d$), and the boundary conditions are periodic or no-flux.
    \item \label{subass:initiallaw}
    The initial distribution $\rho_0$ is absolutely continuous with respect to the Lebesgue measure (we still denote its density as $\rho_0$) and there exists a positive constant $C_0$ such that $(C_0)^{-1}\leq\rho_0(x)\leq C_0$ for all $x\in\Omega$, and $\rho_0\in C^2(\Omega)$.
    
    \item \label{subass:truedensity}
    For every $T\geq0$, the solution $\widehat{p}_t\in C^3(\Omega)$ and there is a positive constant $C^*_T$ such that $(C^*_T)^{-1}\leq \widehat{p}_t(x)\leq C^*_T$ for all $x\in\Omega, t\in[0,T]$.
    
    \item \label{subass:KT}
    The velocity field $\mathbf{f}_t(x)=\mathbf{f}(x,t)\in C^{2,1}(\Omega\times\mathbb{R},\mathbb{R}^d)$. For every $T \geq 0$, there is a positive constant $C_T$ such that $\sup_{(t,x)\in\Omega\times [0,T]}|\nabla \cdot \mathbf{f}(x ,t)| \leq C_T$.
    
    \item \label{subass:bounded}
    The kernel $W\in C^3(\mathbb{R}^d)$ and there exists a positive constant $C_W$ such that $|\nabla W(x)| \leq C_W$ for all $x\in\Omega$.

    \item \label{subass:approx_regularity}
    \textit{(Regularity of Generated Density)} 
    The density $p_t$ induced by the flow $\Phi$ satisfies the following regularity conditions for $t \in [0, T]$:
    \begin{itemize}
        \item There exists a constant $C^f_T$ such that $(C^f_T)^{-1} \leq p_t(x) \leq C^f_T$ for all $x \in \Omega$.
        \item The score function is bounded: there exists $C_s > 0$ such that $\sup_{x \in \Omega} \|\nabla \log p_t(x)\|^2 \leq C_s$.
    \end{itemize}
    \end{enumerate}
\end{assumption}

We first present a lemma that establishes  the upper bound on the squared $L^2$-norm of two probability density functions in terms of  their Kullback–Leibler divergence.

\begin{lemma}\label{L2KL}
    Suppose $p$ and $q$ are two probability densities on $\Omega$, and there exists a positive constant $C_{\text{MAX}}$ such that $0<p(x),q(x)<C_{\text{MAX}}$ for all $x\in\Omega$. Then we have
    {
    \begin{equation*}
        \int_{\Omega}|p(x)-q(x)|^2\d x\leq \frac{2C_{\text{MAX}}}{1-\log2}D_{\mathrm{KL}}(p\|q).
    \end{equation*}
    }
\end{lemma}
\begin{proof}
Define $\zeta(x):= \frac{q(x)-p(x)}{p(x)}$ for $x\in\Omega$. Then $D_{\mathrm{KL}}(p\|q)= \int_{\Omega}p(x)\log\frac{p(x)}{q(x)}\d x= -\int_{\Omega}p(x)\log(1+\zeta(x))\d x$. Define two Borel sets: $A:= \{x \mid \zeta(x)>1\}$ and $B:= \{x \mid \zeta(x)\leq 1\}$; then  for $x\in A$, $1+\zeta(x)\leq e^{\alpha\zeta(x)}$ where $\alpha=\log2>0$; for $x\in B$, $1+\zeta(x)\leq e^{\zeta(x)-\beta\zeta(x)^2}$ where $\beta=1-\log2>0$. Note that $\int_{\Omega}p(x)\zeta(x)\d x= \int_{\Omega}(q(x)-p(x))\d x=0$, which implies $\int_{A}p(x)\zeta(x)\d x= - \int_{B}p(x)\zeta(x)\d x$. Thus 
    {
    \begin{equation*}
        \begin{aligned}
            D_{\mathrm{KL}}(p\|q) =&\ -\int_{A}p(x)\log(1+\zeta(x))\d x  -\int_{B}p(x)\log(1+\zeta(x))\d x\\
            \geq&\ -\alpha\int_{A}p(x)\zeta(x)\d x -\int_{B}p(x)\zeta(x)\d x + \beta\int_{B}p(x)\zeta(x)^2\d x\\
            =&\ (1-\alpha)\int_{A}p(x)\zeta(x)\d x + \beta\int_{B}p(x)\zeta(x)^2\d x\\
            =&\ (1-\log2)\left( \int_{A}|q(x)-p(x)|\d x + \int_{B}p(x) \left(\frac{q(x)-p(x)}{p(x)}\right)^2\d x  \right).
        \end{aligned}
    \end{equation*}
    }
    For the first term, we have $\int_{A}|q(x)-p(x)|\d x\geq \frac{1}{2C_{\text{MAX}}}\int_{A}|q(x)-p(x)|^2\d x$. For the second term, we have $\int_{B}p(x) \left(\frac{q(x)-p(x)}{p(x)}\right)^2\d x \geq \frac{1}{2C_{\text{MAX}}}\int_{B} \left| q(x)-p(x) \right|^2\d x$. Finally, we have $D_{\mathrm{KL}}(p\|q) \geq \frac{1-\log2}{2C_{\text{MAX}}}\int_{\Omega}|q(x)-p(x)|^2\d x$.
\end{proof}

\begin{lemma}\label{lem:inequality}
Let $p(x)$ and $q(x)$ be probability densities defined on a domain $\Omega$ satisfying $0 < C_{\text{min}} \le p(x),q(x) \le C_{\text{MAX}} < \infty$. Assume the score function $\nabla\log p(x)$ is bounded with $C_s = \sup_{x\in \Omega} \|\nabla\log p(x)\|^2<\infty$. Then, for any real  $k\neq 0$, the following inequality holds:
{
\begin{multline}\label{eqn:main-ineq}
    \int_{\Omega} \Bigl[p^k(x) \nabla\log p(x)-q^k(x) \nabla\log q(x)\Bigr]\cdot \Bigl(\nabla\log \frac{p(x)}{q(x)}\Bigr)p(x)\d x
\ge K_1^\sigma \int_{\Omega} \left \|\nabla\log \frac{p(x)}{q(x)} \right \|^2p(x)\d x + K_2^\sigma D_{\mathrm{KL}}(p\|q),
\end{multline}
}
where $\sigma=\operatorname{sign}(k)\in\{+,-\}$, and for any $\lambda>0$,
\[
K_1^+ = C_{\min}^k - \frac{\lambda}{2}C_s,
\qquad
K_2^+ = -\frac{k^2 C_{\text{MAX}}^2 \max\{C_{\min}^{2k-2},C_{\text{MAX}}^{2k-2}\}}{(1-\log 2) \lambda}
<0,
\qquad (k>0),
\]
and
\[
K_1^- = C_{\text{MAX}}^k - \frac{\lambda}{2}C_s,
\qquad
K_2^- = -\frac{k^2 C_{\text{MAX}}^2C_{\min}^{2k-2}}
{(1-\log 2)\lambda}
<0,
\qquad (k<0).
\]
In the applications below we take $p=p_t$ and $q=\widehat{p}_t$. Then the existence of positive constants $C_{\text{min}}$ and $C_{\text{MAX}}$ is guaranteed by Assumptions~\ref{subass:truedensity} and~\ref{subass:approx_regularity}.
\end{lemma}

\begin{proof}
Decompose the integrand on the LHS as $I_1 + I_2$:
{
\begin{align*}
\text{LHS} &= \int_{\Omega}\Bigl[(p^k-q^k) \nabla\log p+q^k \nabla\log \frac{p}{q} \Bigr]\cdot\nabla\log \frac{p}{q} p\d x\\
&=\underbrace{ \int_{\Omega}\Bigl[(p^k-q^k) \nabla\log p \Bigr]\cdot \nabla\log \frac{p}{q} p\d x}_{I_1}+\underbrace{ \int_{\Omega} q^k \left \| \nabla\log \frac{p}{q} \right \|^2 p\d x}_{I_2}.
\end{align*}
}
To estimate \( I_1 \), we apply the mean value theorem to \( f (s) = s^k \). For each \( x \), there exists \( \xi(x) \in [\min\{p, q\}, \max\{p, q\}] \) such that $p^k - q^k = k \xi^{k-1} (p - q)$.
Applying Young's inequality with any constant  \( \lambda > 0 \):
{
\begin{align*}
I_1 &= \int_{\Omega} k \xi^{k-1} (p - q) \left[ \nabla \log p \cdot \nabla \log \frac{p}{q} \right] p \d x \\
&\ge - \frac{1}{2\lambda} \int_{\Omega} \left[ k \xi^{k-1} (p - q) \right]^2 p \d x - \frac{\lambda}{2} \int_{\Omega} \left[ \nabla \log p \cdot \nabla \log \frac{p}{q} \right]^2 p \d x.
\end{align*}
}
Using Cauchy–Schwarz and the bound $C_s$:
\begin{equation}
\left[ \nabla \log p \cdot \nabla \log \frac{p}{q} \right]^2 \le \| \nabla \log p \|^2 \left\| \nabla \log \frac{p}{q} \right\|^2 \le C_s \left\| \nabla \log \frac{p}{q} \right\|^2.
\end{equation}

Also, since $\xi(x)\in [C_{\min},C_{\text{MAX}}]$, we have
\[
|\xi(x)^{k-1}|
\le
\max\{C_{\min}^{k-1},C_{\text{MAX}}^{k-1}\}.
\]
Hence, using Lemma~\ref{L2KL},
\begin{align*}
\int_\Omega \left[k\xi^{k-1}(p-q)\right]^2 pdx
&\le
k^2 \max\{C_{\min}^{2k-2},C_{\text{MAX}}^{2k-2}\} C_{\text{MAX}}
\int_\Omega (p-q)^2dx \\
&\le
\frac{2k^2 C_{\text{MAX}}^2 \max\{C_{\min}^{2k-2},C_{\text{MAX}}^{2k-2}\}}
{1-\log 2}
D_{\KL}(p\|q).
\end{align*}
For $I_2$, we distinguish two cases.

If $k>0$, since $q(x)\ge C_{\min}$, we have
\[
I_2 \ge C_{\min}^k \int_\Omega \left\|\nabla \log\frac{p}{q}\right\|^2 pdx .
\]

If $k<0$, since $q(x)\le C_{\text{MAX}}$ and $s\mapsto s^k$ is decreasing on $(0,\infty)$, we have
\[
I_2 \ge C_{\text{MAX}}^k \int_\Omega \left\|\nabla \log\frac{p}{q}\right\|^2 pdx .
\]

The result follows by combining the bounds for $I_1$ and $I_2$.

\end{proof}

\medskip 
\begin{proof}[Proof of Theorem \ref{theo:KL_bound}]
   The proof is inspired by the methodologies presented in \cite[Proposition 1]{boffi2023probability} and \cite[Appendix E]{shen2024entropy}. However, dealing with the non-entropy internal energy term introduces substantial complexity, requiring the introduction of novel techniques for a thorough analysis.

    First, using  the definition of KL divergence, together with the facts  that  $p_t$ satisfies \eqref{eqn:continuity_f},  $\widehat{p}_t$ satisfies  \eqref{eqn:WGF}  and 
    $\int \partial_t p_t \d x  =0$, we derive that
   \begin{equation*}
        \begin{aligned}
            &\ \frac{\d}{\d t}D_{\mathrm{KL}}(p_t\|\widehat{p}_t) = \int_{\Omega} \left( \partial_t p_t \log \frac{p_t}{\widehat{p}_t} + p_t   \partial_t \left( \log \frac{p_t}{\widehat{p}_t}\right) \right) \d x 
            = \int_{\Omega} \left( \partial_t p_t \log \frac{p_t}{\widehat{p}_t} -  \partial_t  {\widehat{p}_t}   \frac{p_t}{\widehat{p}_t}\right) \d x \\
            =&\ \int_{\Omega}\mathbf{f}(x,t)\cdot\nabla\log\left(\frac{p_t}{\widehat{p}_t}\right) p_t\d x
             -\int_{\Omega}\mathcal{V}[\widehat{p}_t]\cdot\nabla\left(\frac{p_t}{\widehat{p}_t}\right)\widehat{p}_t\d x \\
            =&\  \int_{\Omega}\left( \mathbf{f}(x,t) - \mathcal{V}[p_t] +\mathcal{V}[p_t]-\mathcal{V}[\widehat{p}_t] \right)\cdot \nabla\log \left(\frac{p_t}{ \widehat{p}_t }\right)p_t\d x.
        \end{aligned}
    \end{equation*}
Using the expression of $\mathcal{V}[\widehat{p}_t]$ by \eqref{eqn:velocity}, we decompose the above  as:
   {
    \begin{equation*}
        \begin{aligned}        
            \ \frac{\d}{\d t}D_{\mathrm{KL}}(p_t\|\widehat{p}_t)&=\ \underbrace{\int_{\Omega}\left( \mathbf{f}(x,t) - \mathcal{V}[p_t](x,t)\right)\cdot \nabla\log \left(\frac{p_t}{ \widehat{p}_t }\right)p_t\d x}_{\text{Perturbation}}\\
            &\ - \underbrace{\int_{\Omega} \beta^{-1} \left( \nabla U^{\prime}_m(p_t) - \nabla U^{\prime}_m(\widehat{p}_t) \right)\cdot\nabla\log \left(\frac{p_t}{ \widehat{p}_t }\right)p_t\d x }_{\text{Internal}}\\
            &\ - \underbrace{\int_{\Omega} \left( \nabla V - \nabla V \right)\cdot\nabla\log \left(\frac{p_t}{ \widehat{p}_t }\right)p_t\d x }_{\text{Potential}}\\
            &\ - \underbrace{\int_{\Omega}\left( \nabla W \ast (p_t-\widehat{p}_t) \right)\cdot\nabla\log \left(\frac{p_t}{ \widehat{p}_t }\right)p_t\d x }_{\text{Interaction}}.
        \end{aligned}
    \end{equation*}
    }
    The bounds for these terms are studied below in four steps.
    
{\bf (Step 1.)} For the perturbation part, by Young's inequality with any constant $\Lambda > 0$:
    {
    \begin{equation*}
        \begin{aligned}
            &\ \int_{\Omega}\left( \mathbf{f}(x,t) - \mathcal{V}[p_t](x,t)\right)\cdot \nabla\log \left(\frac{p_t}{ \widehat{p}_t }\right)p_t\d x\\
            \leq&\ \frac{1}{2\Lambda}\int_{\Omega}\left\|\nabla\log \left(\frac{p_t}{ \widehat{p}_t }\right)\right\|^2 p_t\d x + \frac {\Lambda}{2}\int_{\Omega}\left\| \mathbf{f}(x,t) - \mathcal{V}[p_t](x,t)\right\|^2p_t\d x.
        \end{aligned}
    \end{equation*}
    }

{\bf (Step 2.)} For the internal energy part, there are two cases about diffusion ($m=1$) and power-law case ($m\ne 1$) for the internal energy $U_m$. We discuss each case respectively.  

{\bf (Entropy case $: U_m(\rho) =U_1(\rho)= \rho \log \rho$)}
    {
    \begin{equation*}
         -\int_{\Omega} \beta^{-1} \left( \nabla U^{\prime}_m(p_t) - \nabla U^{\prime}_m(\widehat{p}_t) \right)\cdot\nabla\log \left(\frac{p_t}{ \widehat{p}_t }\right)p_t\d x 
            = -\frac{1}{\beta}\int_{\Omega}\left\|\nabla\log \left(\frac{p_t}{ \widehat{p}_t }\right)\right\|^2 p_t\d x.
    \end{equation*}
    }

{\bf (Power-law case $: U_m(\rho)=\frac{1}{m-1}\rho^m,\ m\neq 1$)}

Set $k=m-1$. Then $k\neq 0$, and by Lemma~\ref{lem:inequality},
{\small
\begin{equation*}
-\int_\Omega \beta^{-1}\frac{m}{m-1}
\left(\nabla (p_t)^{m-1}-\nabla (\widehat p_t)^{m-1}\right)
\cdot
\nabla\log\left(\frac{p_t}{\widehat p_t}\right)p_tdx
\le
-\frac{m}{\beta}K_1^{\sigma_m}
\int_\Omega
\left\|\nabla\log\frac{p_t}{\widehat p_t}\right\|^2 p_tdx
-\frac{m}{\beta}K_2^{\sigma_m} D_{\KL}(p_t\|\widehat p_t),
\end{equation*}
}
where $\sigma_m=\operatorname{sign}(m-1)\in\{+,-\}$, and
$K_1^{\sigma_m},K_2^{\sigma_m}$ are the constants from
Lemma~\ref{lem:inequality} evaluated at $k=m-1$.
To ensure $K_1^{\sigma_m}>0$, it is sufficient to choose
\[
0<\lambda<\frac{2a_m}{C_s},
\qquad
a_m=
\begin{cases}
C_{\min}^{m-1}, & m>1,\\[1mm]
C_{\text{MAX}}^{m-1}, & 0<m<1.
\end{cases}
\]

{\bf (Step 3.)} The potential term vanishes trivially.

{\bf (Step 4.)} Due to Assumption \ref{subass:bounded} and Csisz\'{a}r–Kullback–Pinsker inequality:
    {
    \begin{equation*}
        \begin{aligned}
            &\ -\int_{\Omega}\left(\nabla W \ast (p_t-\widehat{p}_t)\right)\cdot\nabla\log \left(\frac{p_t}{ \widehat{p}_t }\right)p_t\d x\\
            \leq&\ \frac{1}{2\Lambda}\int_{\Omega}\left\|\nabla\log \left(\frac{p_t}{ \widehat{p}_t }\right)\right\|^2 p_t\d x+ \frac{\Lambda}{2} (C_W)^2 \left(\int_{\Omega}|p_t(y)-\widehat{p}_t(y)|\d y\right)^2\\
            \leq&\ \frac{1}{2\Lambda}\int_{\Omega}\left\|\nabla\log \left(\frac{p_t}{ \widehat{p}_t }\right)\right\|^2 p_t\d x+ \Lambda C_W^2 D_{\mathrm{KL}}(p_t\| \widehat{p}_t).
        \end{aligned}
    \end{equation*}
    }
    
Finally, we collect  these bounds.
    
    {\bf (Entropy Case)}
    Setting $\Lambda = \beta$, we obtain:
    {
    \begin{equation*}
        \frac{\d}{\d t}D_{\mathrm{KL}}(p_t\|\widehat{p}_t) \leq \left( \beta C_W^2 \right)D_{\mathrm{KL}}(p_t\| \widehat{p}_t)+ \frac{\beta}{2}\int_{\Omega}\left\| \mathbf{f}(x,t) - \mathcal{V}[p_t](x,t)\right\|^2p_t\d x.
    \end{equation*}
    }

{\bf (Power-law Case)}
{\small
\begin{equation*}
\frac{\d}{\d t}D_{\mathrm{KL}}(p_t\|\widehat{p}_t)
\leq\;
\left(\frac{1}{\Lambda}-\frac{m}{\beta}K_1^{\sigma_m}\right)
\int_{\Omega}\left\|\nabla\log \left(\frac{p_t}{\widehat{p}_t}\right)\right\|^2 p_t\d x
+
\left(\Lambda C_W^2-\frac{m}{\beta}K_2^{\sigma_m}\right)
D_{\mathrm{KL}}(p_t\|\widehat{p}_t)
+
\frac{\Lambda}{2}
\int_{\Omega}\left\| \mathbf{f}(x,t)-\mathcal{V}[p_t](x,t)\right\|^2 p_t\d x .
\end{equation*}
}
Setting
$
\Lambda=\frac{\beta}{mK_1^{\sigma_m}},
$
the gradient term vanishes, and therefore
\begin{equation*}
\frac{\d}{\d t}D_{\mathrm{KL}}(p_t\|\widehat{p}_t)
\leq
\left(
\frac{\beta}{mK_1^{\sigma_m}}C_W^2
-\frac{m}{\beta}K_2^{\sigma_m}
\right)
D_{\mathrm{KL}}(p_t\|\widehat{p}_t)
+
\frac{\beta}{2mK_1^{\sigma_m}}
\int_{\Omega}\left\| \mathbf{f}(x,t)-\mathcal{V}[p_t](x,t)\right\|^2 p_t\d x.
\end{equation*}

By Gronwall's inequality, we obtain
\[
\sup_{t\in[0,T]} D_{\KL}(p_t\|\widehat p_t)
\le
\exp\{\gamma T\}
\left(
\alpha \int_0^T\int_\Omega
\|\mathbf f(x,t)-\mathcal V[p_t](x,t)\|^2 p_tdxdt
\right),
\]
where $\frac{1}{2}\int_{0}^T\int_{\Omega}\left\| f - \mathcal{V}[p_t]\right\|^2p_t\d x\d t = J[\Phi]$ represents the flow matching loss \eqref{eqn:J-continuous}, and 
    the constants in the above bound  are defined as:
\begin{equation}\label{eqn:Lambda_barC}
\begin{aligned}
\frac{\alpha}{2} &=
\begin{cases}
\dfrac{\beta}{2}, & m=1,\\[2mm]
\dfrac{\beta}{2mK_1^{\sigma_m}}, & m\neq 1,
\end{cases}
\qquad
\gamma =
\begin{cases}
\beta C_W^2, & m=1,\\[2mm]
\dfrac{\beta}{mK_1^{\sigma_m}}C_W^2
-\dfrac{m}{\beta}K_2^{\sigma_m}, & m\neq 1,
\end{cases}
\end{aligned}
\end{equation}
where $\sigma_m=\operatorname{sign}(m-1)$ and
$K_1^{\sigma_m},K_2^{\sigma_m}$ are the constants from
Lemma~\ref{lem:inequality} with $k=m-1$.

\end{proof}

%% file: proof_disc_error_bound.tex
\begin{proof}[Proof of Theorem \ref{theo:disc_error_bound}]
For notational simplicity, we write $X^*_k(z) := X^*(t_k,z)$ and $X^N_k(z) := \Phi_k(z)$ for $k=0,\dots,K$, where $t_k = k\Delta t$.
We define the (expected) trajectory error at the grid points by
\[
    \epsilon_{t_k}
    :=
    \mathbb{E}_{z\sim \rho_0}\big\|X^*_k(z)-X^N_k(z)\big\|.
\]

\paragraph{Step 1: Average-velocity formulation at one time step.}
Fix $k\in\{1,\dots,K\}$ and $z$.
Using the integral form of the exact ODE,
\[
    X^*_k(z) - X^*_{k-1}(z)
    =
    \int_{t_{k-1}}^{t_k} \mathcal{V}[\widehat{p}_s]\big(X^*(s,z)\big)\d s,
\]
we define the \emph{exact average velocity} over the interval $[t_{k-1},t_k]$ by
\[
    \bar v^*_k(z)
    :=
    \frac{X^*_k(z)-X^*_{k-1}(z)}{\Delta t}
    =
    \frac{1}{\Delta t}
    \int_{t_{k-1}}^{t_k} 
        \mathcal{V}[\widehat{p}_s]\big(X^*(s,z)\big)\d s.
\]
By definition of the discrete trajectory $X^N_k(z)=\Phi_k(z)$, we introduce the \emph{discrete kinematic velocity}
\[
    v^N_k(z)
    :=
    \frac{X^N_k(z)-X^N_{k-1}(z)}{\Delta t}.
\]

We also introduce the \emph{ideal Crank-Nicolson} velocity (on the continuous model) and its empirical counterpart:
\[
    v^{\mathrm{CN},*}_k(z)
    :=
    \frac{1}{2}\Big(
       \mathcal{V}[\widehat{p}_{t_{k-1}}]\big(X^*_{k-1}(z)\big)
       +
       \mathcal{V}[\widehat{p}_{t_k}]\big(X^*_k(z)\big)
    \Big),
\]
\[
    v^{\mathrm{CN},N}_k(z)
    :=
    \frac{1}{2}\Big(
       \mathcal{V}_N[\widehat{p}_{t_{k-1}}]\big(X^N_{k-1}(z)\big)
       +
       \mathcal{V}_N[\widehat{p}_{t_k}]\big(X^N_k(z)\big)
    \Big).
\]
By definition,
\[
    \sup_{k,z} \big\|v^N_k(z) - v^{\mathrm{CN},N}_k(z)\big\| \le \varepsilon.
\]

\paragraph{Step 2: Local Crank-Nicolson discretization error.}
Define
\[
    f_k(s,z)
    := \mathcal{V}[\widehat{p}_s]\big(X^*(s,z)\big),\qquad s\in[t_{k-1},t_k].
\]
Then
\[
    \bar v^*_k(z)
    = \frac{1}{\Delta t}\int_{t_{k-1}}^{t_k} f_k(s,z)\d s,
    \qquad
    v^{\mathrm{CN},*}_k(z)
    = \frac{f_k(t_{k-1},z)+f_k(t_k,z)}{2}.
\]
Under the smoothness assumptions in Appendix~\ref{ass:densitytheta}, $f_k(\cdot,z)$ is twice continuously differentiable in $t$, with uniformly bounded second derivative.
Applying the classical error estimate of the trapezoidal rule (which is the time-discretization underlying Crank-Nicolson) yields
\[
    \Big\|
        \bar v^*_k(z) - v^{\mathrm{CN},*}_k(z)
    \Big\|
    \le C\Delta t^2,
\]
for some constant $C>0$ independent of $k$, $\Delta t$, and $N$.
Equivalently,
\begin{equation}\label{eq:loc-CN-disc}
    \bar v^*_k(z)
    =
    v^{\mathrm{CN},*}_k(z)
    + \mathcal{O}(\Delta t^2).
\end{equation}

\paragraph{Step 3: Ideal CN velocity vs. empirical CN velocity.}
We next bound the difference $v^{\mathrm{CN},*}_k(z) - v^{\mathrm{CN},N}_k(z)$.
Let
\[
    a_j(z):=\mathcal{V}[\widehat{p}_{t_j}]\big(X^*_j(z)\big),
    \qquad
    b_j(z):=\mathcal{V}_N[\widehat{p}_{t_j}]\big(X^N_j(z)\big),
    \quad j=k-1,k.
\]
Then
\[
    v^{\mathrm{CN},*}_k(z)-v^{\mathrm{CN},N}_k(z)
    = \frac{1}{2}\Big((a_{k-1}-b_{k-1}) + (a_k-b_k)\Big),
\]
and hence
\begin{equation}\label{eq:vCNstar-vCNN-bound}
    \big\|v^{\mathrm{CN},*}_k(z)-v^{\mathrm{CN},N}_k(z)\big\|
    \le \frac{1}{2}\Big(
        \|a_{k-1}(z)-b_{k-1}(z)\|
        +
        \|a_k(z)-b_k(z)\|
    \Big).
\end{equation}
For a generic index $j\in\{k-1,k\}$, we decompose
\[
\begin{aligned}
    \|a_j(z)-b_j(z)\|
    &=
    \big\|
        \mathcal{V}[\widehat{p}_{t_j}]\big(X^*_j(z)\big)
        -
        \mathcal{V}_N[\widehat{p}_{t_j}]\big(X^N_j(z)\big)
    \big\| \\
    &\le
    \underbrace{
    \big\|
        \mathcal{V}[\widehat{p}_{t_j}]\big(X^*_j(z)\big)
        -
        \mathcal{V}_N[\widehat{p}_{t_j}]\big(X^*_j(z)\big)
    \big\|
    }_{\text{finite-sample approximation of the field}}
    \\
    &\quad+
    \underbrace{
    \big\|
        \mathcal{V}_N[\widehat{p}_{t_j}]\big(X^*_j(z)\big)
        -
        \mathcal{V}_N[\widehat{p}_{t_j}]\big(X^N_j(z)\big)
    \big\|
    }_{\text{Lipschitzinspace $\times$ trajectory error}}.
\end{aligned}
\]

Since $\mathcal V_N[\widehat p_{t_j}](x)$ is an unbiased empirical average of i.i.d.\ random vectors with uniformly bounded variance \ref{subass:bounded}, it follows from the standard Monte Carlo estimate that
\[
\mathbb E\Big\|
\mathcal V[\widehat p_{t_j}](x)-\mathcal V_N[\widehat p_{t_j}](x)
\Big\|
\le C_1 N^{-1/2}.
\]
In particular,
\begin{equation}\label{eq:sampling-V}
    \mathbb{E}_{z\sim\rho_0}\big\|
        \mathcal{V}[\widehat{p}_{t_j}]\big(X^*_j(z)\big)
        -
        \mathcal{V}_N[\widehat{p}_{t_j}]\big(X^*_j(z)\big)
    \big\|
    \le C_1 N^{-1/2}.
\end{equation}

Moreover, by the Lipschitz continuity of $\mathcal{V}_N[\widehat{p}_{t_j}](\cdot)$ in space (Assumption~\ref{ass:densitytheta}), there exists $L_x>0$ such that
\[
    \big\|
        \mathcal{V}_N[\widehat{p}_{t_j}]\big(X^*_j(z)\big)
        -
        \mathcal{V}_N[\widehat{p}_{t_j}]\big(X^N_j(z)\big)
    \big\|
    \le
    L_x \big\|X^*_j(z)-X^N_j(z)\big\|.
\]
Taking expectation in $z\sim\rho_0$ and recalling the definition of $\epsilon_{t_j}$, we obtain
\begin{equation}\label{eq:Lip-VN}
    \mathbb{E}_{z\sim\rho_0}\big\|
        \mathcal{V}_N[\widehat{p}_{t_j}]\big(X^*_j(z)\big)
        -
        \mathcal{V}_N[\widehat{p}_{t_j}]\big(X^N_j(z)\big)
    \big\|
    \le
    L_x\epsilon_{t_j}.
\end{equation}

Combining \eqref{eq:vCNstar-vCNN-bound}, \eqref{eq:sampling-V}, and \eqref{eq:Lip-VN}, we deduce that
\begin{equation}\label{eq:CNstar-CNN-expect}
    \mathbb{E}_{z\sim\rho_0}
    \big\|
        v^{\mathrm{CN},*}_k(z)-v^{\mathrm{CN},N}_k(z)
    \big\|
    \le
    C_1 N^{-1/2}
    + C_2 \max\{\epsilon_{t_{k-1}},\epsilon_{t_k}\},
\end{equation}
for some constant $C_2>0$ independent of $N$ and $\Delta t$.

\paragraph{Step 4: Consistency residual relative to the exact Crank-Nicolson driving force.}
By the definition of $\varepsilon$ in~\eqref{eqn:epsilon_def},
\[
    \sup_{z}\big\|v^N_k(z)-v^{\mathrm{CN},N}_k(z)\big\| \le \varepsilon,
\]
and thus, in particular,
\begin{equation}\label{eq:train-error-vel}
    \mathbb{E}_{z\sim\rho_0}
    \big\|v^N_k(z)-v^{\mathrm{CN},N}_k(z)\big\|
    \le \varepsilon.
\end{equation}

\paragraph{Step 5: One-step error recursion in average velocity form.}
For each $z$,
\[
    X^*_k(z)-X^N_k(z)
    =
    X^*_{k-1}(z)-X^N_{k-1}(z)
    +
    \Delta t\Big(
        \bar v^*_k(z) - v^N_k(z)
    \Big),
\]
hence
\[
    \big\|X^*_k(z)-X^N_k(z)\big\|
    \le
    \big\|X^*_{k-1}(z)-X^N_{k-1}(z)\big\|
    +
    \Delta t\big\|
        \bar v^*_k(z) - v^N_k(z)
    \big\|.
\]
Taking expectation and using the triangle inequality,
\[
\begin{aligned}
    \epsilon_{t_k}
    &\le
    \epsilon_{t_{k-1}}
    +
    \Delta t
    \mathbb{E}_{z\sim\rho_0}
    \Big\|
        \bar v^*_k(z)-v^N_k(z)
    \Big\|.
\end{aligned}
\]
We now split the average-velocity discrepancy into the three components derived above:
\[
\begin{aligned}
    \bar v^*_k(z)-v^N_k(z)
    &=
    \underbrace{
    \bar v^*_k(z)-v^{\mathrm{CN},*}_k(z)
    }_{\text{CN discretization error}}
    +
    \underbrace{
    v^{\mathrm{CN},*}_k(z)-v^{\mathrm{CN},N}_k(z)
    }_{\text{finite-sample + trajectory error}}
    +
    \underbrace{
    v^{\mathrm{CN},N}_k(z)-v^N_k(z)
}_{\text{consistency residual}}.
\end{aligned}
\]
Using \eqref{eq:loc-CN-disc}, \eqref{eq:CNstar-CNN-expect}, and \eqref{eq:train-error-vel}, we obtain
\[
\begin{aligned}
    \mathbb{E}_{z\sim\rho_0}
    \big\|
        \bar v^*_k(z)-v^N_k(z)
    \big\|
    &\le
    \mathcal{O}(\Delta t^2)
    + C_1 N^{-1/2}
    + C_2 \max\{\epsilon_{t_{k-1}},\epsilon_{t_k}\}
    + \varepsilon.
\end{aligned}
\]
Therefore,
\begin{equation}\label{eq:error-recursion-raw}
    \epsilon_{t_k}
    \le
    \epsilon_{t_{k-1}}
    +
    \Delta t\Big(
        \mathcal{O}(\Delta t^2)
        + C_1 N^{-1/2}
        + C_2 \max\{\epsilon_{t_{k-1}},\epsilon_{t_k}\}
        + \varepsilon
    \Big).
\end{equation}

Using $\max\{\epsilon_{t_{k-1}},\epsilon_{t_k}\}\le \epsilon_{t_{k-1}}+\epsilon_{t_k}$ and absorbing the resulting terms, we can rewrite \eqref{eq:error-recursion-raw} as
\[
    \epsilon_{t_k}
    \le
    (1+C\Delta t)\epsilon_{t_{k-1}}
    +
    \Delta t\Big(
        C\Delta t^2
        + C N^{-1/2}
        + C\varepsilon
    \Big),
\]
for some constant $C>0$ independent of $N$, $\Delta t$, and $k$ (provided $\Delta t$ is sufficiently small so that $1-C_2\Delta t>0$).
Applying the discrete Grönwall inequality and using the shared initial condition $X^*_0 = X^N_0$ (i.e.\ $\epsilon_{t_0}=0$), we obtain
\[
    \epsilon_{t_k}
    \le
    C' t_k
    \Big(
        \Delta t^2
        + N^{-1/2}
        + \varepsilon
    \Big),
\]
for all $k$ with $t_k\le T$, where $C'>0$ is a constant depending only on $T$ and the regularity constants of $\mathcal{V}$.
This proves the claimed bound~\eqref{eqn:disc_error} and completes the proof.
\end{proof}

%% file: proof_gMAM.tex
\begin{assumption}\label{ass:gMAM}
\hfill
    \begin{enumerate}
        \item $\mathcal F(\rho_a),\mathcal F(\rho_b)<+\infty$.
        \item  For $\forall T \ge 0$ and  every path $p\in AC_{\rho_a,\rho_b,T}$, the map $t\mapsto \mathcal F(p_t)$ is absolutely continuous on $[0,T]$ and satisfies
\[
\frac{d}{dt}\mathcal F(p_t)
=
\Bigl\langle
\nabla_{\d_{\mathcal W}}\mathcal F(p_t),\partial_t p_t
\Bigr\rangle_{-1,p_t}
\qquad\text{for a.e. }t\in(0,T).
\]
    \end{enumerate}
\end{assumption}
\begin{lemma}\label{lem:finite_action_L2}
Under the assumptions \ref{ass:gMAM}
,if $S_T[p]<+\infty$, then
\[
\|\partial_t p_t\|_{-1,p_t}\in L^2(0,T),
\qquad
\|\nabla_{\d_{\mathcal W}}\mathcal F(p_t)\|_{-1,p_t}\in L^2(0,T),
\]
and
\begin{equation}\label{eq:finite_action_identity}
S_T[p]
=
\frac12\int_0^T \|\partial_t p_t\|_{-1,p_t}^2\d t
+
\frac12\int_0^T \|\nabla_{\d_{\mathcal W}}\mathcal F(p_t)\|_{-1,p_t}^2\d t
+
\mathcal F(\rho_b)-\mathcal F(\rho_a).
\end{equation}
\end{lemma}

\begin{proof}
Set
\[
u_t:=\partial_t p_t,
\qquad
v_t:=\nabla_{\d_{\mathcal W}}\mathcal F(p_t).
\]
Since $S_T[p]<+\infty$, the definition of $S_T$ yields
\[
\|u_t+v_t\|_{-1,p_t}^2\in L^1(0,T).
\]
Moreover, by absolute continuity of $t\mapsto \mathcal F(p_t)$ and the chain rule,
\[
\langle v_t,u_t\rangle_{-1,p_t}\in L^1(0,T),
\qquad
\int_0^T \langle v_t,u_t\rangle_{-1,p_t}\d t
=
\mathcal F(\rho_b)-\mathcal F(\rho_a).
\]
Using the pointwise identity
\[
\|u_t\|_{-1,p_t}^2+\|v_t\|_{-1,p_t}^2
=
\|u_t+v_t\|_{-1,p_t}^2
-
2\langle v_t,u_t\rangle_{-1,p_t},
\]
we conclude that
\[
\|u_t\|_{-1,p_t}^2+\|v_t\|_{-1,p_t}^2\in L^1(0,T).
\]
Since both terms on the left-hand side are nonnegative, each belongs to $L^1(0,T)$ separately, i.e.
\[
\|\partial_t p_t\|_{-1,p_t}\in L^2(0,T),
\qquad
\|\nabla_{\d_{\mathcal W}}\mathcal F(p_t)\|_{-1,p_t}\in L^2(0,T).
\]
Integrating the same identity and using the chain rule again gives
\[
\int_0^T \|u_t+v_t\|_{-1,p_t}^2\d t
=
\int_0^T \|u_t\|_{-1,p_t}^2\d t
+
\int_0^T \|v_t\|_{-1,p_t}^2\d t
+
2\bigl(\mathcal F(\rho_b)-\mathcal F(\rho_a)\bigr),
\]
which is exactly \eqref{eq:finite_action_identity}.
\end{proof}

\begin{lemma}\label{lem:finite_geometric_parts}
Under the assumptions \ref{ass:gMAM}, let
\[
a_p(\tau):=\|\partial_\tau p_\tau\|_{-1,p_\tau},\qquad
b_p(\tau):=\|\nabla_{\d_{\mathcal W}}\mathcal F(p_\tau)\|_{-1,p_\tau},\qquad
c_p(\tau):=
\Bigl\langle
\nabla_{\d_{\mathcal W}}\mathcal F(p_\tau),\partial_\tau p_\tau
\Bigr\rangle_{-1,p_\tau}.
\]
If $\widehat S[p]<+\infty$, then $c_p\in L^1(0,1)$,
\(
\int_0^1 c_p(\tau)\d\tau=\mathcal F(\rho_b)-\mathcal F(\rho_a),
\)
and $a_pb_p\in L^1(0,1)$. In particular,
\[
\widehat S[p]
=
\int_0^1 a_p(\tau)b_p(\tau)\d\tau
+
\int_0^1 c_p(\tau)\d\tau.
\]
\end{lemma}

\begin{proof}
Set
\[
g_p(\tau):=a_p(\tau)b_p(\tau)+c_p(\tau).
\]
By Cauchy-Schwarz,
\[
c_p(\tau)\ge -a_p(\tau)b_p(\tau)
\qquad\text{for a.e. }\tau\in(0,1),
\]
hence $g_p(\tau)\ge0$ a.e. Since $\widehat S[p]<+\infty$, the definition of $\widehat S$ implies that $g_p\in L^1(0,1)$. On the other hand, the absolute continuity of $\tau\mapsto \mathcal F(p_\tau)$ and the chain rule give
\[
c_p\in L^1(0,1),
\qquad
\int_0^1 c_p(\tau)\d\tau
=
\mathcal F(\rho_b)-\mathcal F(\rho_a).
\]
Therefore
\[
a_pb_p=g_p-c_p\in L^1(0,1),
\]
and the claimed representation of $\widehat S[p]$ follows.
\end{proof}


\begin{proof}[Proof of Theorem \ref{thm:gMAM_equivalence}]
We prove the two inequalities separately.

\medskip
\noindent\textbf{Step 1.}
Fix $T>0$ and $p\in AC_{\rho_a,\rho_b,T}$. If $S_T[p]=+\infty$, the claim is immediate. Assume $S_T[p]<+\infty$. By Lemma~\ref{lem:finite_action_L2},
\[
\|\partial_t p_t\|_{-1,p_t},
\|\nabla_{\d_{\mathcal W}}\mathcal F(p_t)\|_{-1,p_t}\in L^2(0,T).
\]
Define $\tilde p_\tau:=p_{T\tau}$ on $[0,1]$. Then $\tilde p\in AC_{\rho_a,\rho_b,1}$ and, with $t=T\tau$,
\[
\partial_t p_t=\frac1T\partial_\tau\tilde p_\tau
\qquad\text{for a.e. }\tau\in(0,1).
\]
Substituting this relation into \eqref{eqn:rate-functional} gives
\begin{align*}
S_T[p]
&=
\frac12\int_0^1
T\Bigl\|
\frac1T\partial_\tau\tilde p_\tau+\nabla_{\d_{\mathcal W}}\mathcal F(\tilde p_\tau)
\Bigr\|_{-1,\tilde p_\tau}^2\d\tau \\
&=
\frac1{2T}\int_0^1 \|\partial_\tau\tilde p_\tau\|_{-1,\tilde p_\tau}^2\d\tau
+
\frac T2\int_0^1 \|\nabla_{\d_{\mathcal W}}\mathcal F(\tilde p_\tau)\|_{-1,\tilde p_\tau}^2\d\tau
+
\int_0^1
\Bigl\langle
\nabla_{\d_{\mathcal W}}\mathcal F(\tilde p_\tau),\partial_\tau\tilde p_\tau
\Bigr\rangle_{-1,\tilde p_\tau}\d\tau\\
& =\frac{A}{2T}+\frac{BT}{2}+C\ge \sqrt{AB}+C,
\end{align*}
where
\[
A:=\int_0^1 \|\partial_\tau\tilde p_\tau\|_{-1,\tilde p_\tau}^2\d\tau,\qquad
B:=\int_0^1 \|\nabla_{\d_{\mathcal W}}\mathcal F(\tilde p_\tau)\|_{-1,\tilde p_\tau}^2\d\tau,
\qquad
C:=\int_0^1
\Bigl\langle
\nabla_{\d_{\mathcal W}}\mathcal F(\tilde p_\tau),\partial_\tau\tilde p_\tau
\Bigr\rangle_{-1,\tilde p_\tau}\d\tau.
\]
By Cauchy-Schwarz,
\[
\sqrt{AB}\ge
\int_0^1
\|\partial_\tau\tilde p_\tau\|_{-1,\tilde p_\tau}
\|\nabla_{\d_{\mathcal W}}\mathcal F(\tilde p_\tau)\|_{-1,\tilde p_\tau}\d\tau.
\]
Therefore
\[
S_T[p]\ge \widehat S[\tilde p].
\]
Since this estimate holds for every $T>0$ and every $p\in AC_{\rho_a,\rho_b,T}$,
\[
\inf_{T>0}\inf_{p\in AC_{\rho_a,\rho_b,T}} S_T[p]
\ge
\inf_{q\in AC_{\rho_a,\rho_b,1}} \widehat S[q].
\]

\medskip
\noindent\textbf{Step 2.}
Fix $p\in AC_{\rho_a,\rho_b,1}$. If $\widehat S[p]=+\infty$, then
\[
\inf_{T>0}\inf_{q\in AC_{\rho_a,\rho_b,T}} S_T[q]\le \widehat S[p]
\]
is immediate. Assume $\widehat S[p]<+\infty$.

Since the integrand in \eqref{eqn:geo-action} is positively homogeneous of degree one in $\partial_\tau p_\tau$, the functional $\widehat S$ is invariant under absolutely continuous monotone reparameterizations. Thus, without loss of generality, we may assume that 
$p$ is parametrized with constant speed
$
a(\tau):=\|\partial_\tau p_\tau\|_{-1,p_\tau}\equiv \ell
$ for a.e. $\tau\in(0,1),
$
for some $\ell\ge0$. If $\ell=0$, then $p$ is constant, hence $\rho_a=\rho_b$, and both sides of \eqref{eqn:gMAM_equivalence} vanish. We therefore restrict to the case $\ell>0$. Define
\[
b(\tau):=\|\nabla_{\d_{\mathcal W}}\mathcal F(p_\tau)\|_{-1,p_\tau},\qquad
c(\tau):=\Bigl\langle
\nabla_{\d_{\mathcal W}}\mathcal F(p_\tau),\partial_\tau p_\tau
\Bigr\rangle_{-1,p_\tau}.
\]
By Lemma~\ref{lem:finite_geometric_parts}, we have $ab\in L^1(0,1)$, $c\in L^1(0,1)$, and
\[
\widehat S[p]=\int_0^1 a(\tau)b(\tau)\d\tau+\int_0^1 c(\tau)\d\tau.
\]

For $\varepsilon>0$, define
\[
\nu_\varepsilon(\tau):=\frac{a(\tau)}{b(\tau)+\varepsilon},\qquad
t_\varepsilon(\tau):=\int_0^\tau \nu_\varepsilon(s)\d s.
\]
Since $a(\tau)=\ell>0$ a.e. and $b(\tau)+\varepsilon\ge\varepsilon$, we have
\[
0<\nu_\varepsilon(\tau)\le \frac{\ell}{\varepsilon}
\qquad\text{for a.e. }\tau\in(0,1),
\]
so $\nu_\varepsilon\in L^1(0,1)$ and $t_\varepsilon$ is absolutely continuous and strictly increasing. Set
\[
T_\varepsilon:=t_\varepsilon(1),
\]
and let $\tau_\varepsilon:[0,T_\varepsilon]\to[0,1]$ denote the inverse map of $t_\varepsilon$. Then $\tau_\varepsilon$ is absolutely continuous and
\[
\dot\tau_\varepsilon(t)=\frac1{\nu_\varepsilon(\tau_\varepsilon(t))}
\qquad\text{for a.e. }t\in(0,T_\varepsilon).
\]
Define
\[
p_t^\varepsilon:=p_{\tau_\varepsilon(t)},\qquad t\in[0,T_\varepsilon].
\]
Then $p^\varepsilon\in AC_{\rho_a,\rho_b,T_\varepsilon}$. Moreover,
\[
\frac{a(\tau)^2}{\nu_\varepsilon(\tau)}=a(\tau)\bigl(b(\tau)+\varepsilon\bigr)\in L^1(0,1),
\qquad
b(\tau)^2\nu_\varepsilon(\tau)=\frac{a(\tau)b(\tau)^2}{b(\tau)+\varepsilon}\le a(\tau)b(\tau)\in L^1(0,1),
\]
and $c\in L^1(0,1)$, so in particular $S_{T_\varepsilon}[p^\varepsilon]<+\infty$.

Using the chain rule and the change of variables $t=t_\varepsilon(\tau)$, we obtain
\begin{align*}
S_{T_\varepsilon}[p^\varepsilon]
&=
\frac12\int_0^{T_\varepsilon}
\Bigl\|
\partial_t p_t^\varepsilon+\nabla_{\d_{\mathcal W}}\mathcal F(p_t^\varepsilon)
\Bigr\|_{-1,p_t^\varepsilon}^2\d t \\
&=
\frac12\int_0^1
\left(
\frac{a(\tau)^2}{\nu_\varepsilon(\tau)}
+
b(\tau)^2\nu_\varepsilon(\tau)
\right)\d\tau
+
\int_0^1 c(\tau)\d\tau.
\end{align*}
Substituting $\nu_\varepsilon(\tau)=a(\tau)/(b(\tau)+\varepsilon)$ gives
\[
\frac12\left(
\frac{a^2}{\nu_\varepsilon}+b^2\nu_\varepsilon
\right)
=
\frac12\left(
a(b+\varepsilon)+\frac{ab^2}{b+\varepsilon}
\right)
=
ab+\frac{\varepsilon^2 a}{2(b+\varepsilon)}.
\]
Therefore
\[
S_{T_\varepsilon}[p^\varepsilon]
=
\widehat S[p]+R_\varepsilon[p],
\qquad
R_\varepsilon[p]:=
\frac12\int_0^1 \frac{\varepsilon^2 a(\tau)}{b(\tau)+\varepsilon}\d\tau.
\]
Since
\[
0\le \frac{\varepsilon^2 a(\tau)}{b(\tau)+\varepsilon}\le \varepsilon a(\tau)
\qquad\text{for a.e. }\tau\in(0,1),
\]
and $a\in L^1(0,1)$, it follows that
\[
0\le R_\varepsilon[p]\le \frac{\varepsilon}{2}\int_0^1 a(\tau)\d\tau
=\frac{\varepsilon\ell}{2}\xrightarrow[\varepsilon\downarrow0]{}0.
\]
Hence
\[
\lim_{\varepsilon\downarrow0} S_{T_\varepsilon}[p^\varepsilon]=\widehat S[p].
\]
Since $p^\varepsilon\in AC_{\rho_a,\rho_b,T_\varepsilon}$ for every $\varepsilon>0$,
\[
\inf_{T>0}\inf_{q\in AC_{\rho_a,\rho_b,T}} S_T[q]
\le S_{T_\varepsilon}[p^\varepsilon]
\qquad\text{for all }\varepsilon>0.
\]
Passing to the limit $\varepsilon\downarrow0$ yields
\[
\inf_{T>0}\inf_{q\in AC_{\rho_a,\rho_b,T}} S_T[q]\le \widehat S[p].
\]
Taking the infimum over all $p\in AC_{\rho_a,\rho_b,1}$ gives
\[
\inf_{T>0}\inf_{q\in AC_{\rho_a,\rho_b,T}} S_T[q]
\le
\inf_{p\in AC_{\rho_a,\rho_b,1}} \widehat S[p].
\]
Combining Steps 1 and 2 proves \eqref{eqn:gMAM_equivalence}.
\end{proof}

%% file: proof_EL_compare.tex
\begin{proof}
  Let \(\hat{J}^{K,*}\) denote the minimum value of the Lagrangian problem \eqref{eqn:prob_lagrange}, and \(\hat{S}^{K,*}\) denote the minimum value of the Eulerian problem \eqref{eqn:prob_euler}.
  First, we compare the minimum values of the two problems. For any admissible sequence of maps \(\Phi\), let \(p\) be the induced density sequence defined by \(p_k=(\Phi_k)_\#\rho_0\). By definition of the Wasserstein-2 distance, for each \(k\), the coupling \((\Phi_{k-1},\Phi_k)_\#\rho_0\) between \(p_{k-1}\) and \(p_k\) yields
  \[
  W_2^2(p_{k-1},p_k)\le \mathbb{E}\|\Phi_k-\Phi_{k-1}\|^2.
  \]
  Therefore,
  \begin{equation}\label{eq:basic_inequality}
    \hat{S}^K[p]\le \hat{J}^K[\Phi].
  \end{equation}
  Taking the infimum over all admissible \(\Phi\), we obtain
  \[
    \hat{S}^{K,*}\le \hat{J}^{K,*}.
  \]

  Conversely, let \(p^*\) be a minimizer of \eqref{eqn:prob_euler}, so that \(\hat{S}^K[p^*]=\hat{S}^{K,*}\). Let \(\Psi_k\) be the optimal transport map from \(p^*_{k-1}\) to \(p^*_k\), and define the Lagrangian map sequence by
  \[
    \Phi_0^*=\mathrm{id},\qquad
    \Phi_k^*=\Psi_k\circ\Phi_{k-1}^*
    =\Psi_k\circ\cdots\circ\Psi_1.
  \]
  Then \((\Phi_k^*)_\#\rho_0=p_k^*\) for every \(k\), so \(\Phi^*\) is admissible for \eqref{eqn:prob_lagrange}. By construction, each transport cost in \(\hat{J}^K[\Phi^*]\) agrees exactly with the corresponding Wasserstein distance in \(\hat{S}^K[p^*]\). Hence
  \[
    \hat{J}^K[\Phi^*]=\hat{S}^K[p^*]=\hat{S}^{K,*}.
  \]
  It follows that
  \[
    \hat{J}^{K,*}\le \hat{S}^{K,*}.
  \]
  Combining the two inequalities, we conclude that
  \[
    \hat{J}^{K,*}=\hat{S}^{K,*}.
  \]

  \textbf{Proof of (a):}
  Let \(\Phi^*\) be a minimizer of \eqref{eqn:prob_lagrange}, and let \(p^*\) be the induced density sequence defined by \(p_k^*=(\Phi_k^*)_\#\rho_0\). By \eqref{eq:basic_inequality},
  \[
    \hat{S}^K[p^*]\le \hat{J}^K[\Phi^*]=\hat{J}^{K,*}=\hat{S}^{K,*}.
  \]
  Since \(\hat{S}^{K,*}\) is the minimum value of the Eulerian problem, we also have \(\hat{S}^{K,*}\le \hat{S}^K[p^*]\). Therefore,
  \[
    \hat{S}^K[p^*]=\hat{S}^{K,*},
  \]
  and \(p^*\) is a minimizer of \eqref{eqn:prob_euler}.

  \textbf{Proof of (b):}
  Suppose \(p^*\) is a minimizer of \eqref{eqn:prob_euler}, so \(\hat{S}^K[p^*]=\hat{S}^{K,*}\). Let \(\Psi_k\) be the optimal transport map from \(p^*_{k-1}\) to \(p^*_k\), and define \(\Phi_0^*=\mathrm{id}\), \(\Phi_k^*=\Psi_k\circ\Phi_{k-1}^*\). Then, as shown above,
  \[
    \hat{J}^K[\Phi^*]=\hat{S}^K[p^*]=\hat{S}^{K,*}=\hat{J}^{K,*}.
  \]
  This proves that the composite map \(\Phi^*\) is a minimizer of the Lagrangian problem \eqref{eqn:prob_lagrange}.
\end{proof}

%% file: proof_disc_conv.tex
\begin{proof}
We analyze the discretization error $|\hat{J}(\Phi)-\hat{J}_N^K(\Phi)|$, which arises from the temporal discretization of the geometric action and the Monte Carlo approximation of the $L^2(\rho_0)$-norms.

The continuous action functional is
\begin{equation}\label{eqn:conti-loss}
    \hat{J}(\Phi)
    =
    \int_0^1
    \left\| \mathcal{V}[p_\tau](\Phi(\tau,\cdot)) \right\|_{L^2(\rho_0)}
    \left\| \partial_\tau \Phi(\tau,\cdot) \right\|_{L^2(\rho_0)}
    \d\tau.
\end{equation}
Let $L(\tau):=A(\tau)B(\tau)$, where
\[
    A(\tau):=\left\| \mathcal{V}[p_\tau](\Phi(\tau,\cdot)) \right\|_{L^2(\rho_0)},
    \qquad
    B(\tau):=\left\| \partial_\tau \Phi(\tau,\cdot) \right\|_{L^2(\rho_0)}.
\]

Let $\tau_k:=k\Delta\tau$, with $\Delta\tau=1/K$, and $\Phi_k:=\Phi(\tau_k,\cdot)$. In accordance with \eqref{eqn:dk_vk_def}, define
\[
    \hat{A}_j^N:=v_j(\Phi),
    \qquad
    \hat{B}_k^N:=\frac{d_k(\Phi)}{\Delta\tau}
    =
    \left(
    \frac1N\sum_{i=1}^N
    \left\|
    \frac{\Phi_k(z_i)-\Phi_{k-1}(z_i)}{\Delta\tau}
    \right\|^2
    \right)^{1/2},
\]
where $\{z_i\}_{i=1}^N$ are i.i.d.\ samples drawn from $\rho_0$. Then
\begin{equation}\label{eqn:disc-loss}
    \hat{J}_N^K(\Phi)
    =
    \sum_{k=1}^K
    \underbrace{\frac{\hat{A}_k^N+\hat{A}_{k-1}^N}{2}}_{\approx A(\tau_{k-1/2})}
    \cdot
    \underbrace{\hat{B}_k^N}_{\approx B(\tau_{k-1/2})}
    \Delta\tau,
\end{equation}
with $\tau_{k-1/2}:=(\tau_{k-1}+\tau_k)/2$.

We assume directly that $A,B\in C^2([0,1])$ with uniformly bounded derivatives. We also assume that $\Phi\in C^3([0,1];L^2(\rho_0))$. Furthermore, for
\[
    \widetilde{B}_k
    :=
    \left\|
    \frac{\Phi(\tau_k,\cdot)-\Phi(\tau_{k-1},\cdot)}{\Delta\tau}
    \right\|_{L^2(\rho_0)},
\]
we assume the Monte Carlo estimators satisfy the second-moment bounds
\[
    \mathbb{E}\bigl[|\hat{A}_j^N-A(\tau_j)|^2\bigr]\le C_A^2 N^{-1},
    \qquad
    \mathbb{E}\bigl[|\hat{B}_k^N-\widetilde{B}_k|^2\bigr]\le C_B^2 N^{-1},
\]
together with
\[
    \sup_j \mathbb{E}\bigl[|\hat{A}_j^N|^2\bigr]
    +
    \sup_k \mathbb{E}\bigl[|\hat{B}_k^N|^2\bigr]
    \le C,
\]
for some constant $C>0$ independent of $K$ and $N$.

We first estimate the deterministic quadrature error. Since $L\in C^2([0,1])$, the midpoint rule yields
\begin{equation}\label{eq:midpoint-expansion}
    \int_{\tau_{k-1}}^{\tau_k} L(\tau)\d\tau
    =
    A(\tau_{k-1/2})B(\tau_{k-1/2})\Delta\tau
    +
    \mathcal{O}((\Delta\tau)^3),
\end{equation}
where the constant depends only on a uniform bound on $L''$.

We next estimate the two discrete factors.

\medskip
\noindent\textbf{Term $\hat{A}^N$.}
By Taylor expansion around $\tau_{k-1/2}$,
\[
    A(\tau_j)
    =
    A(\tau_{k-1/2})
    \pm
    \frac{\Delta\tau}{2}A'(\tau_{k-1/2})
    +
    \mathcal{O}((\Delta\tau)^2),
    \qquad j\in\{k-1,k\},
\]
and therefore
\[
    \frac{A(\tau_k)+A(\tau_{k-1})}{2}
    =
    A(\tau_{k-1/2})+\mathcal{O}((\Delta\tau)^2).
\]
It follows that
\begin{equation}\label{eq:A-error}
    \left(
    \mathbb{E}\left|
    \frac{\hat{A}_k^N+\hat{A}_{k-1}^N}{2}
    -
    A(\tau_{k-1/2})
    \right|^2
    \right)^{1/2}
    =
    \mathcal{O}((\Delta\tau)^2)+\mathcal{O}(N^{-1/2}).
\end{equation}

\medskip
\noindent\textbf{Term $\hat{B}_k^N$.}
Since $\Phi\in C^3([0,1];L^2(\rho_0))$, the centered difference at the midpoint satisfies
\[
    \frac{\Phi(\tau_k,\cdot)-\Phi(\tau_{k-1},\cdot)}{\Delta\tau}
    =
    \partial_\tau \Phi(\tau_{k-1/2},\cdot)
    +
    \mathcal{O}((\Delta\tau)^2)
    \qquad\text{in }L^2(\rho_0).
\]
Because the $L^2(\rho_0)$-norm is $1$-Lipschitz, we obtain
\[
    \widetilde{B}_k
    =
    B(\tau_{k-1/2})+\mathcal{O}((\Delta\tau)^2).
\]
Hence
\begin{equation}\label{eq:B-error}
    \left(
    \mathbb{E}\bigl|
    \hat{B}_k^N-B(\tau_{k-1/2})
    \bigr|^2
    \right)^{1/2}
    \le
    \left(
    \mathbb{E}\bigl|
    \hat{B}_k^N-\widetilde{B}_k
    \bigr|^2
    \right)^{1/2}
    +
    \bigl|
    \widetilde{B}_k-B(\tau_{k-1/2})
    \bigr|
    =
    \mathcal{O}((\Delta\tau)^2)+\mathcal{O}(N^{-1/2}).
\end{equation}

Define
\[
    L_k^{\mathrm{disc}}
    :=
    \frac{\hat{A}_k^N+\hat{A}_{k-1}^N}{2}\hat{B}_k^N.
\]
Set
\[
    \Delta A_k
    :=
    \frac{\hat{A}_k^N+\hat{A}_{k-1}^N}{2}-A(\tau_{k-1/2}),
    \qquad
    A_m:=A(\tau_{k-1/2}),
    \qquad
    B_m:=B(\tau_{k-1/2}).
\]
Then
\[
    L_k^{\mathrm{disc}}-A_mB_m
    =
    \Delta A_k\hat{B}_k^N
    +
    A_m\bigl(\hat{B}_k^N-B_m\bigr).
\]
Therefore,
\[
    \mathbb{E}\bigl|L_k^{\mathrm{disc}}-A_mB_m\bigr|
    \le
    \mathbb{E}\bigl|\Delta A_k\hat{B}_k^N\bigr|
    +
    |A_m|\mathbb{E}\bigl|\hat{B}_k^N-B_m\bigr|.
\]
By Cauchy-Schwarz and \eqref{eq:A-error},
\[
    \mathbb{E}\bigl|\Delta A_k\hat{B}_k^N\bigr|
    \le
    \bigl(\mathbb{E}|\Delta A_k|^2\bigr)^{1/2}
    \bigl(\mathbb{E}|\hat{B}_k^N|^2\bigr)^{1/2}
    =
    \mathcal{O}((\Delta\tau)^2)+\mathcal{O}(N^{-1/2}),
\]
while \eqref{eq:B-error} gives
\[
    |A_m|\mathbb{E}\bigl|\hat{B}_k^N-B_m\bigr|
    \le
    |A_m|
    \bigl(\mathbb{E}|\hat{B}_k^N-B_m|^2\bigr)^{1/2}
    =
    \mathcal{O}((\Delta\tau)^2)+\mathcal{O}(N^{-1/2}).
\]
Hence
\begin{equation}\label{eq:Ldisc-error}
    \mathbb{E}\bigl|
    L_k^{\mathrm{disc}}-A(\tau_{k-1/2})B(\tau_{k-1/2})
    \bigr|
    =
    \mathcal{O}((\Delta\tau)^2)+\mathcal{O}(N^{-1/2}).
\end{equation}

Let
\[
    \mathcal{E}_k
    :=
    L_k^{\mathrm{disc}}\Delta\tau
    -
    \int_{\tau_{k-1}}^{\tau_k}L(\tau)\d\tau
\]
denote the local error on $[\tau_{k-1},\tau_k]$. Then
\[
    \mathcal{E}_k
    =
    \left(
    L_k^{\mathrm{disc}}-A(\tau_{k-1/2})B(\tau_{k-1/2})
    \right)\Delta\tau
    +
    \left(
    A(\tau_{k-1/2})B(\tau_{k-1/2})\Delta\tau
    -
    \int_{\tau_{k-1}}^{\tau_k}L(\tau)\d\tau
    \right).
\]
Taking expectations and using \eqref{eq:midpoint-expansion} and \eqref{eq:Ldisc-error}, we obtain
\begin{align*}
    \mathbb{E}|\mathcal{E}_k|
    &\le
    \mathbb{E}\bigl|
    L_k^{\mathrm{disc}}-A(\tau_{k-1/2})B(\tau_{k-1/2})
    \bigr|\Delta\tau
    +
    \mathcal{O}((\Delta\tau)^3) \\
    &=
    \left(
    \mathcal{O}((\Delta\tau)^2)+\mathcal{O}(N^{-1/2})
    \right)\Delta\tau
    +
    \mathcal{O}((\Delta\tau)^3) \\
    &=
    \mathcal{O}((\Delta\tau)^3)+\mathcal{O}(\Delta\tau N^{-1/2}).
\end{align*}

Summing over all $K=1/\Delta\tau$ intervals gives
\begin{align*}
    \mathbb{E}\bigl|
    \hat{J}_N^K(\Phi)-\hat{J}(\Phi)
    \bigr|
    &\le
    \sum_{k=1}^K \mathbb{E}|\mathcal{E}_k| =
    K\mathcal{O}((\Delta\tau)^3)
    +
    K\mathcal{O}(\Delta\tau N^{-1/2}) =
    \mathcal{O}((\Delta\tau)^2)+\mathcal{O}(N^{-1/2}) =
    \mathcal{O}(K^{-2})+\mathcal{O}(N^{-1/2}).
\end{align*}
This proves the claimed estimate.
\end{proof}

%% file: GenerativeWGF.bib
@book{boltzmann1872weitere,
  title={Weitere studien {\"u}ber das w{\"a}rmegleichgewicht unter gasmolek{\"u}len},
  author={Boltzmann, Ludwig},
  volume={66},
  year={1872},
  publisher={Aus der kk Hot-und Staatsdruckerei}
}

@article{lafferty1988density,
  title={The density manifold and configuration space quantization},
  author={Lafferty, John D},
  journal={Transactions of the American Mathematical Society},
  volume={305},
  number={2},
  pages={699--741},
  year={1988}
}

@article{jordan1998variational,
  title={{The variational formulation of the Fokker--Planck equation}},
  author={Jordan, Richard and Kinderlehrer, David and Otto, Felix},
  journal={SIAM journal on mathematical analysis},
  volume={29},
  number={1},
  pages={1--17},
  year={1998},
  publisher={SIAM}
}

@article{otto2001geometry,
  title={The geometry of dissipative evolution equations: the porous medium equation},
  author={Otto, F},
  journal={Comm. Partial Differential Equations},
  volume={26},
  pages={101--174},
  year={2001}
}

@article{carrillo2003kinetic,
  title={Kinetic equilibration rates for granular media and related equations: entropy dissipation and mass transportation estimates},
  author={Carrillo, Jos{\'e} A and McCann, Robert J and Villani, C{\'e}dric},
  journal={Revista Matematica Iberoamericana},
  volume={19},
  number={3},
  pages={971--1018},
  year={2003}
}

@book{vazquez2007porous,
  title={The porous medium equation: mathematical theory},
  author={V{\'a}zquez, Juan Luis},
  year={2007},
  publisher={Oxford university press}
}

@book{ambrosio2008gradient,
  title={Gradient flows: in metric spaces and in the space of probability measures},
  author={Ambrosio, Luigi and Gigli, Nicola and Savar{\'e}, Giuseppe},
  year={2008},
  publisher={Springer Science \& Business Media}
}

@book{rousset2010free,
  title={Free energy computations: a mathematical perspective},
  author={Rousset, Mathias and Stoltz, Gabriel and Lelievre, Tony},
  year={2010},
  publisher={World Scientific}
}

@book{villani2021topics,
  title={Topics in optimal transportation},
  author={Villani, C{\'e}dric},
  volume={58},
  year={2021},
  publisher={American Mathematical Soc.}
}

@article{byun2024planar,
  title={Planar equilibrium measure problem in the quadratic fields with a point charge},
  author={Byun, Sung-Soo},
  journal={Computational Methods and Function Theory},
  volume={24},
  number={2},
  pages={303--332},
  year={2024},
  publisher={Springer}
}

@article{carrillo2015finite,
  title={A finite-volume method for nonlinear nonlocal equations with a gradient flow structure},
  author={Carrillo, Jos{\'e} A and Chertock, Alina and Huang, Yanghong},
  journal={Communications in Computational Physics},
  volume={17},
  number={1},
  pages={233--258},
  year={2015},
  publisher={Cambridge University Press}
}

@article{carrillo2022primal,
  title={Primal dual methods for Wasserstein gradient flows},
  author={Carrillo, Jos{\'e} A and Craig, Katy and Wang, Li and Wei, Chaozhen},
  journal={Foundations of Computational Mathematics},
  pages={1--55},
  year={2022},
  publisher={Springer}
}

@article{yu2018deep,
  title={{The deep Ritz method: a deep learning-based numerical algorithm for solving variational problems}},
  author={Yu, Bing and others},
  journal={Communications in Mathematics and Statistics},
  volume={6},
  number={1},
  pages={1--12},
  year={2018},
  publisher={Springer}
}

@article{raissi2019physics,
  title={Physics-informed neural networks: A deep learning framework for solving forward and inverse problems involving nonlinear partial differential equations},
  author={Raissi, Maziar and Perdikaris, Paris and Karniadakis, George E},
  journal={Journal of Computational physics},
  volume={378},
  pages={686--707},
  year={2019},
  publisher={Elsevier}
}

@article{tang2022adaptive,
  title={{Adaptive deep density approximation for Fokker-Planck equations}},
  author={Tang, Kejun and Wan, Xiaoliang and Liao, Qifeng},
  journal={Journal of Computational Physics},
  volume={457},
  pages={111080},
  year={2022},
  publisher={Elsevier}
}

@article{xie2023deep,
  title={Deep variational free energy approach to dense hydrogen},
  author={Xie, Hao and Li, Zi-Hang and Wang, Han and Zhang, Linfeng and Wang, Lei},
  journal={Physical Review Letters},
  volume={131},
  number={12},
  pages={126501},
  year={2023},
  publisher={APS}
}

@article{cai2024weak,
  title={Weak Generative Sampler to Efficiently Sample Invariant Distribution of Stochastic Differential Equation},
  author={Cai, Zhiqiang and Cao, Yu and Huang, Yuanfei and Zhou, Xiang},
  journal={SIAM Journal on Scientific Computing (to appear. arXiv2405.19256 )},
  year={2026},
  publisher={SIAM}
}

@article{cai2025weak,
  title={Weak Generative Sampler for Stationary Distributions of McKean-Vlasov System},
  author={Cai, Zhiqiang and Liu, Chengyu and Zhou, Xiang},
  journal={arXiv preprint arXiv:2509.12841},
  year={2025}
}

@article{liu2022neural,
  title={Neural Parametric Fokker--Planck Equation},
  author={Liu, Shu and Li, Wuchen and Zha, Hongyuan and Zhou, Haomin},
  journal={SIAM Journal on Numerical Analysis},
  volume={60},
  number={3},
  pages={1385--1449},
  year={2022},
  publisher={SIAM}
}

@article{boffi2023probability,
  title={Probability flow solution of the {F}okker--{P}lanck equation},
  author={Boffi, Nicholas M and Vanden-Eijnden, Eric},
  journal={Machine Learning: Science and Technology},
  volume={4},
  number={3},
  pages={035012},
  year={2023},
  publisher={IOP Publishing}
}

@article{nurbekyan2023efficient,
  title={{Efficient natural gradient descent methods for large-scale PDE-based optimization problems}},
  author={Nurbekyan, Levon and Lei, Wanzhou and Yang, Yunan},
  journal={SIAM Journal on Scientific Computing},
  volume={45},
  number={4},
  pages={A1621--A1655},
  year={2023},
  publisher={SIAM}
}

@article{huang2023bridging,
  title={Bridging mean-field games and normalizing flows with trajectory regularization},
  author={Huang, Han and Yu, Jiajia and Chen, Jie and Lai, Rongjie},
  journal={Journal of Computational Physics},
  volume={487},
  pages={112155},
  year={2023},
  publisher={Elsevier}
}

@inproceedings{li2023selfconsistent,
title={Self-Consistent Velocity Matching of Probability Flows},
author={Lingxiao Li and Samuel Hurault and Justin Solomon},
booktitle={Thirty-seventh Conference on Neural Information Processing Systems},
year={2023},
url={https://openreview.net/forum?id=C6fvJ2RfsL}
}

@article{huang2024vy,
  title={{Levy Score Function and Score-Based Particle Algorithm for Nonlinear Levy--Fokker--Planck Equations}},
  author={Huang, Yuanfei and Liu, Chengyu and Zhou, Xiang},
journal={SIAM Journal on Numerical Analysis (to appear), arXiv 2412.19520 },
year={2026},
  publisher={SIAM}
}

@article{shen2024entropy,
  title={Entropy-dissipation informed neural network for {M}ckean-{V}lasov type {PDE}s},
  author={Shen, Zebang and Wang, Zhenfu},
  journal={Advances in Neural Information Processing Systems},
  volume={36},
  year={2024}
}

@article{hu2024energetic,
  title={Energetic Variational Neural Network Discretizations of Gradient Flows},
  author={Hu, Ziqing and Liu, Chun and Wang, Yiwei and Xu, Zhiliang},
  journal={SIAM Journal on Scientific Computing},
  volume={46},
  number={4},
  pages={A2528--A2556},
  year={2024},
  publisher={SIAM}
}

@article{lee2024deep,
  title={{Deep JKO: time-implicit particle methods for general nonlinear gradient flows}},
  author={Lee, Wonjun and Wang, Li and Li, Wuchen},
  journal={Journal of Computational Physics},
  pages={113187},
  year={2024},
  publisher={Elsevier}
}

@article{lu2024score,
  title={Score-based transport modeling for mean-field {F}okker-{P}lanck equations},
  author={Lu, Jianfeng and Wu, Yue and Xiang, Yang},
  journal={Journal of Computational Physics},
  volume={503},
  pages={112859},
  year={2024},
  publisher={Elsevier}
}

@article{peletier2014variational,
  title={Variational modelling: Energies, gradient flows, and large deviations},
  author={Peletier, Mark A},
  journal={arXiv preprint arXiv:1402.1990},
  year={2014}
}

@article{onsager1931reciprocal1,
  title={Reciprocal relations in irreversible processes. I.},
  author={Onsager, Lars},
  journal={Physical review},
  volume={37},
  number={4},
  pages={405},
  year={1931},
  publisher={APS}
}

@article{heymann2008geometric,
  title={The geometric minimum action method: A least action principle on the space of curves},
  author={Heymann, Matthias and Vanden-Eijnden, Eric},
  journal={Communications on Pure and Applied Mathematics: A Journal Issued by the Courant Institute of Mathematical Sciences},
  volume={61},
  number={8},
  pages={1052--1117},
  year={2008},
  publisher={Wiley Online Library}
}

@inproceedings{rezende2015variational,
  title={Variational inference with normalizing flows},
  author={Rezende, Danilo and Mohamed, Shakir},
  booktitle={International conference on machine learning},
  pages={1530--1538},
  year={2015},
  organization={PMLR}
}

@inproceedings{dinh2016density,
  title={{Density estimation using Real NVP}},
  author={Dinh, Laurent and Sohl-Dickstein, Jascha and Bengio, Samy},
  booktitle={International Conference on Learning Representations},
  year={2016}
}

@article{chen2018neural,
  title={Neural ordinary differential equations},
  author={Chen, Ricky TQ and Rubanova, Yulia and Bettencourt, Jesse and Duvenaud, David K},
  journal={Advances in neural information processing systems},
  volume={31},
  year={2018}
}

@article{durkan2019neural,
  title={Neural spline flows},
  author={Durkan, Conor and Bekasov, Artur and Murray, Iain and Papamakarios, George},
  journal={Advances in neural information processing systems},
  volume={32},
  year={2019}
}

@article{kobyzev2020normalizing,
  title={Normalizing flows: An introduction and review of current methods},
  author={Kobyzev, Ivan and Prince, Simon JD and Brubaker, Marcus A},
  journal={IEEE transactions on pattern analysis and machine intelligence},
  volume={43},
  number={11},
  pages={3964--3979},
  year={2020},
  publisher={IEEE}
}

@article{papamakarios2021normalizing,
  title={Normalizing flows for probabilistic modeling and inference},
  author={Papamakarios, George and Nalisnick, Eric and Rezende, Danilo Jimenez and Mohamed, Shakir and Lakshminarayanan, Balaji},
  journal={Journal of Machine Learning Research},
  volume={22},
  number={57},
  pages={1--64},
  year={2021}
}

@inproceedings{hong2023neural,
  title={{Neural diffeomorphic non-uniform B-spline flows}},
  author={Hong, Seongmin and Chun, Se Young},
  booktitle={Proceedings of the AAAI Conference on Artificial Intelligence},
  volume={37},
  number={10},
  pages={12225--12233},
  year={2023}
}

@article{hyvarinen2005estimation,
  title={Estimation of non-normalized statistical models by score matching.},
  author={Hyv{\"a}rinen, Aapo and Dayan, Peter},
  journal={Journal of Machine Learning Research},
  volume={6},
  number={4},
  year={2005}
}

@inproceedings{
    song2021scorebased,
    title={Score-Based Generative Modeling through Stochastic Differential Equations},
    author={Yang Song and Jascha Sohl-Dickstein and Diederik P Kingma and Abhishek Kumar and Stefano Ermon and Ben Poole},
    booktitle={International Conference on Learning Representations},
    year={2021},
}

@inproceedings{xu2023normalizing,
title={Normalizing flow neural networks by {JKO} scheme},
author={Chen Xu and Xiuyuan Cheng and Yao Xie},
booktitle={Thirty-seventh Conference on Neural Information Processing Systems},
year={2023},
url={https://openreview.net/forum?id=ZQMlfNijY5}
}


%% file: gad.bib
@string{cpam = {Comm. Pure Appl. Math.}}

@string{jcp = {J. Chem. Phys.}}

@article{weinan-MAM2004,
	Author = {W. ~E and W. ~Ren and E. ~Vanden-Eijnden},
	Issue = {5},
	Journal = CPAM,
	Pages = {637-656},
	Title = {Minimum Action Method for the Study of Rare Events},
	Volume = {57},
	Year = {2004}}

@article{String2002,
	Author = {W.~E and W. ~Ren and E. ~Vanden-Eijnden},
	Journal = {Phys. Rev. B},
	Pages = {052301},
	Title = {String method for the study of rare events},
	Volume = {66},
	Year = {2002}}

@article{Heymann2006,
	Author = {Matthias.~Heymann and Eric ~Vanden-Eijnden},
	Journal = CPAM,
	Pages = {1052-1117},
	Title = {The geometric minimum action method: a least action principle on the space of curves},
	Volume = {61},
	Year = {2008}}

@article{Heyman2008,
	Author = {E. Vanden-Eijnden and M. Heymann},
	Journal = JCP,
	Pages = {061103},
	Title = {The Geometric Minimum Action Method for Computing Minimum Energy Paths},
	Volume = {128},
	Year = {2008}}

@article{SIMONNET2023112349,
	author = {Eric Simonnet},
	doi = {https://doi.org/10.1016/j.jcp.2023.112349},
	issn = {0021-9991},
	journal = {Journal of Computational Physics},
	pages = {112349},
	title = {Computing non-equilibrium trajectories by a deep learning approach},
	url = {https://www.sciencedirect.com/science/article/pii/S0021999123004448},
	volume = {491},
	year = {2023}}

@article{aMAM2008,
	Author = {X.~Zhou and W.~Ren and W.~E},
	Doi = {10.1063/1.2830717},
	Eid = {104111},
	Journal = JCP,
	Number = {10},
	Numpages = {11},
	Pages = {104111},
	Publisher = {AIP},
	Title = {Adaptive minimum action method for the study of rare events},
	Url = {http://link.aip.org/link/?JCP/128/104111/1},
	Volume = {128},
	Year = {2008}}

@Article{CiCP2018-SUNZHOU,
author = {Yiqun Sun and Xiang Zhou},
title = {An Improved Adaptive Minimum Action Method for the Calculation of Transition Path in Non-Gradient Systems},
journal = {Communications in Computational Physics},
year = {2018},
volume = {24},
number = {1},
pages = {44--68},
 issn = {1991-7120},
doi = {https://doi.org/10.4208/cicp.OA-2016-0230},
url = {http://global-sci.org/intro/article_detail/cicp/10927.html}
}

@article{WantMAM2015,
  title={A minimum action method with optimal linear time scaling},
  author={Wan, Xiaoliang},
  journal={Communications in Computational Physics},
  volume={18},
  number={5},
  pages={1352--1379},
  year={2015},
  publisher={Cambridge University Press}
}

@book{FW2012,
	Address = {New York},
	Author = {M. I.~Freidlin and A. D.~Wentzell},
	Edition = {3},
	Publisher = {Springer-Verlag},
	Series = {Grundlehren der mathematischen Wissenschaften},
	Title = {Random Perturbations of Dynamical Systems},
	Year = {2012}}

@book{Feng2006,
	Address = {Prividence, RI},
	Author = {Jin Feng and Thomas G. Kurtz},
	Publisher = {American Mathematical Society},
	Series = {Mathematical Surveys and Monographs},
	Title = {Large Deviations for Stochastic Processes},
	Volume = {131},
	Year = {2006}}

@article{Dawson1983,
  	Author = {Dawson, Donald A.},
  	Doi = {10.1007/BF01010922},
   	Journal = {Journal of Statistical Physics},
  	Language = {English},
  	Number = {1},
  	Pages = {29-85},
  	Publisher = {Kluwer Academic Publishers-Plenum Publishers},
  	Title = {Critical dynamics and fluctuations for a mean-field model of cooperative behavior},
  	Volume = {31},
  	Year = {1983},
  	 }

@article{DG1987,
	Author = {Dawson, Donald A. and G\"{a}rtner, J\"{u}rgen},
	Doi = {10.1080/17442508708833446},
	Eprint = {http://dx.doi.org/10.1080/17442508708833446},
	Journal = {Stochastics},
	Number = {4},
	Pages = {247-308},
	Title = {Large deviations from the {M}c{K}ean-{V}lasov limit for weakly interacting diffusions},
	Url = {http://dx.doi.org/10.1080/17442508708833446},
	Volume = {20},
	Year = {1987}}

@inbook{DW1989,
  	Author = {Dawson, Donald A. and G\"{a}rtner, J\"{u}rgen},
  	Doi = {http://dx.doi.org/10.1090/memo/0398},
  	Number = {398},
  	Publisher = {Memoirs of the American Mathematical Society},
  	Title = {Large deviations, free energy functional and quasi-potential for a mean field model of interacting diffusions},
  	Volume = {78},
  	Year = {1989}}

@article{Zimmer2013,
  author = {Adams, Stefan  and Dirr, Nicolas  and Peletier, Mark  and Zimmer, Johannes },
  title = {Large deviations and gradient flows},
  journal = {Philosophical Transactions of the Royal Society A: Mathematical, Physical and Engineering Sciences},
  volume = {371},
  number = {2005},
  pages = {20120341},
  year = {2013},
  doi = {10.1098/rsta.2012.0341},

  URL = {https://royalsocietypublishing.org/doi/abs/10.1098/rsta.2012.0341},
  eprint = {https://royalsocietypublishing.org/doi/pdf/10.1098/rsta.2012.0341}
  }

@book{Villani2009,
	Author = {Villani, Cedric},
	Publisher = {Springer-Verlag, Berlin},
	Title = {Optimal transport: old and new},
	Year = {2009},
  Series = {Grundlehren der Mathematischen Wissenschaften [Fundamental Principles of Mathematical Sciences]},
	Volume ={338},
  }

@article{StringNET2026,
  title={{ StringNET: Neural Network based Variational Method for Transition Pathways}},
  author={Jiayue Han and Zhiyou Wu  and   Shuting Gu and Xiang Zhou},
journal = {Communications in Computational Physics},
  year={2026}
}
